%% file: main.tex
\documentclass[twoside]{article}

\usepackage[accepted]{aistats2026}

\usepackage{url}
\usepackage{xspace}
\usepackage{paralist}

\usepackage[utf8]{inputenc} % allow utf-8 input
\usepackage[T1]{fontenc}    % use 8-bit T1 fonts
\usepackage{hyperref}       % hyperlinks
\usepackage{url}            % simple URL typesetting
\usepackage{booktabs}       % professional-quality tables
\usepackage{amsfonts}       % blackboard math symbols
\usepackage{nicefrac}       % compact symbols for 1/2, etc.
\usepackage{microtype}      % microtypography
\usepackage{xcolor}         % colors
\usepackage{amsmath}
\usepackage{mathtools}
\usepackage{natbib}
\usepackage{cite}
\usepackage{xspace}
\usepackage{empheq}
\usepackage{amsmath}
\usepackage{amssymb}
\usepackage{mathtools}
\usepackage{amsthm}
\usepackage{float}
\usepackage{url,enumerate}
\usepackage{color,xcolor}
\usepackage{makeidx}  % allows for indexgeneration
\usepackage{amsmath,amssymb}
\usepackage{mathtools}
\usepackage[small, compact]{titlesec}
\usepackage{xspace}
\usepackage{epstopdf}
\usepackage{mathrsfs}
\usepackage{times}
\usepackage{enumerate}
\usepackage{color}
\usepackage{graphicx,epsfig}
\usepackage{amsmath,amssymb,xspace}
\usepackage{url}
\usepackage{bm}
\usepackage{bbm}
\usepackage{upgreek}
\usepackage{multirow}
\usepackage{ulem}
\usepackage{caption}
\usepackage{subcaption}
\usepackage{adjustbox}

\theoremstyle{plain}
\newtheorem{theorem}{Theorem}[section]
\newtheorem{proposition}[theorem]{Proposition}
\newtheorem{lemma}[theorem]{Lemma}

\theoremstyle{definition}
\newtheorem{definition}[theorem]{Definition}
\newtheorem{assumption}[theorem]{Assumption}
\theoremstyle{remark}

\newcommand{\cmt}[1]{}

\newcommand{\ours}{{TGPS}\xspace}

\newcommand{\sks}{{SKS}\xspace}
\newcommand{\yifan}{{DAKS}\xspace}

\begin{document}
	\input{./emacscomm.tex}

% \documentclass[accepted]{aistats2026}

\runningtitle{Tensor Gaussian Processes: Efficient Solvers for Nonlinear PDEs}

\runningauthor{Yuan, Xu, Chen, Xu, Owhadi, Zhe}

\twocolumn[

\aistatstitle{Tensor Gaussian Processes: Efficient Solvers for Nonlinear PDEs}

\aistatsauthor{ Qiwei Yuan$^{1}$ \And Zhitong Xu$^{1}$ \And  Yinghao Chen$^{1}$ \And Yiming Xu$^{2}$ \And Houman Owhadi$^{3}$ \And Shandian Zhe$^{1}$}

\aistatsaddress{ $^{1}$ Kahlert School of Computing, University of Utah\\
                $^{2}$ Department of Mathematics, University of Kentucky\\
                $^{3}$ Department of Computing and Mathematical Sciences, California Institute of Technology
                }]

\input{abstract}
\input{intro}
\input{method}

\input{related}

\input{expr}
\input{conclusion}

\bibliography{reference.bib}
\bibliographystyle{apalike}

%\clearpage
%\newpage
\input{checklist}
%\clearpage
%\newpage
%\appendix
%\input{appendix}

\clearpage
\appendix

%\onecolumn
\input{appendix}
%%%%%%%%%%%%%%%%%%%%%%%%%%%%%%%%%%%%%%%%%%%%%%%%%%%%%%%%%%%%

\end{document}

%% file: emacscomm.tex
% Math commands by Thomas Minka
\newcommand{\var}{{\rm var}}
\newcommand{\vtrans}[2]{{#1}^{(#2)}}
\newcommand{\kron}{\otimes}
\newcommand{\gp}{\mathcal{GP}}
\newcommand{\schur}[2]{({#1} | {#2})}
\newcommand{\schurdet}[2]{\left| ({#1} | {#2}) \right|}
\newcommand{\had}{\circ}
\newcommand{\diag}{{\rm diag}}
\newcommand{\invdiag}{\diag^{-1}}
\newcommand{\rank}{{\rm rank}}
\newcommand{\expt}[1]{\langle #1 \rangle}
\newcommand{\whalpha}{\widehat{\alpha}}
% careful: ``null'' is already a latex command
\newcommand{\nullsp}{{\rm null}}
\newcommand{\tr}{{\rm tr}}
\renewcommand{\vec}{{\rm vec}}
\newcommand{\vech}{{\rm vech}}
\renewcommand{\det}[1]{\left| #1 \right|}
\newcommand{\pdet}[1]{\left| #1 \right|_{+}}
\newcommand{\pinv}[1]{#1^{+}}
\newcommand{\erf}{{\rm erf}}
\newcommand{\hypergeom}[2]{{}_{#1}F_{#2}}
\newcommand{\mcal}[1]{\mathcal{#1}}

% boldface characters
\renewcommand{\a}{{\bf a}}
\renewcommand{\b}{{\bf b}}
\renewcommand{\c}{{\bf c}}
\renewcommand{\d}{{\rm d}}  % for derivatives
\newcommand{\e}{{\bf e}}
\newcommand{\f}{{\bf f}}
\newcommand{\g}{{\bf g}}
\newcommand{\h}{{\bf h}}
\newcommand{\bi}{{\bf i}}
\newcommand{\bj}{{\bf j}}
\newcommand{\bK}{{\bf K}}
%\newcommand{\k}{{\bf k}}
% in Latex2e this must be renewcommand
\renewcommand{\k}{{\bf k}}
\newcommand{\m}{{\bf m}}
\newcommand{\mhat}{{\overline{m}}}
\newcommand{\tm}{{\tilde{m}}}
\newcommand{\n}{{\bf n}}
\renewcommand{\o}{{\bf o}}
\newcommand{\p}{{\bf p}}
\newcommand{\q}{{\bf q}}
\renewcommand{\r}{{\bf r}}
\newcommand{\s}{{\bf s}}
\renewcommand{\t}{{\bf t}}
\renewcommand{\u}{{\bf u}}
\renewcommand{\v}{{\bf v}}
\newcommand{\w}{{\bf w}}
\newcommand{\x}{{\bf x}}
\newcommand{\y}{{\bf y}}
\newcommand{\z}{{\bf z}}
\newcommand{\bl}{{\bf l}}
\newcommand{\A}{{\bf A}}
\newcommand{\B}{{\bf B}}
\newcommand{\C}{{\bf C}}
\newcommand{\D}{{\bf D}}
\newcommand{\Dcal}{\mathcal{D}}
\newcommand{\Ocal}{\mathcal{O}}
\newcommand{\F}{{\bf F}}
\newcommand{\G}{{\bf G}}
\newcommand{\Gcal}{{\mathcal{G}}}
\newcommand{\Rcal}{{\mathcal{R}}}

\renewcommand{\H}{{\bf H}}
\newcommand{\I}{{\bf I}}
\newcommand{\J}{{\bf J}}
\newcommand{\K}{{\bf K}}
\renewcommand{\L}{{\bf L}}
\newcommand{\Lcal}{{\mathcal{L}}}
\newcommand{\M}{{\bf M}}
\newcommand{\Mcal}{{\mathcal{M}}}
\newcommand{\Ecal}{{\mathcal{E}}}
\newcommand{\N}{\mathcal{N}}  % for normal density
\newcommand{\TN}{\mathcal{TN}}  % for normal density
\newcommand{\MN}{\mathcal{MN}}
\newcommand{\bupeta}{\boldsymbol{\upeta}}
\newcommand{\kl}{{\text{KL}}}
\renewcommand{\O}{{\bf O}}
\renewcommand{\P}{{\bf P}}
\newcommand{\Q}{{\bf Q}}
\renewcommand{\S}{{\bf S}}
\newcommand{\Scal}{{\mathcal{S}}}
\newcommand{\Bcal}{{\mathcal{B}}}
\newcommand{\Pcal}{{\mathcal{P}}}
\newcommand{\T}{{\bf T}}
\newcommand{\Tcal}{{\mathcal{T}}}
\newcommand{\U}{{\bf U}}
\newcommand{\Ucal}{{\mathcal{U}}}
\newcommand{\tUcal}{{\tilde{\Ucal}}}
\newcommand{\V}{{\bf V}}
\newcommand{\W}{{\bf W}}
\newcommand{\Wcal}{{\mathcal{W}}}
\newcommand{\Vcal}{{\mathcal{V}}}
\newcommand{\X}{{\bf X}}
\newcommand{\Xcal}{{\mathcal{X}}}
\newcommand{\Acal}{{\mathcal{A}}}
\newcommand{\Y}{{\bf Y}}
\newcommand{\Ycal}{{\mathcal{Y}}}
\newcommand{\Z}{{\bf Z}}
\newcommand{\Zcal}{{\mathcal{Z}}}
\newcommand{\Hcal}{{\mathcal{H}}}
\newcommand{\Fcal}{{\mathcal{F}}}
\newcommand{\whL}{{\widehat{\Lcal}}}
\newcommand{\whJ}{{\widehat{J}}}

% this is for latex 2.09
% unfortunately, the result is slanted - use Latex2e instead
%\newcommand{\bfLambda}{\mbox{\boldmath$\Lambda$}}
% this is for Latex2e
\newcommand{\bfLambda}{\boldsymbol{\Lambda}}

% Yuan Qi's boldsymbol
\newcommand{\bsigma}{\boldsymbol{\sigma}}
\newcommand{\balpha}{\boldsymbol{\alpha}}
\newcommand{\bpsi}{\boldsymbol{\psi}}
\newcommand{\bphi}{\boldsymbol{\phi}}
\newcommand{\bPhi}{\boldsymbol{\Phi}}
\newcommand{\cov}{{\text{cov}}}
\newcommand{\uhat}{{\widehat{u}}}
\newcommand{\fhat}{{\hat{f}}}

\newcommand{\bbeta}{\boldsymbol{\beta}}
\newcommand{\bepsi}{\boldsymbol{\epsilon}}
\newcommand{\boldeta}{\boldsymbol{\eta}}
\newcommand{\btau}{\boldsymbol{\tau}}
\newcommand{\bvarphi}{\boldsymbol{\varphi}}
\newcommand{\bzeta}{\boldsymbol{\zeta}}

\newcommand{\blambda}{\boldsymbol{\lambda}}
\newcommand{\bLambda}{\mathbf{\Lambda}}

\newcommand{\btheta}{{\boldsymbol{\theta}}}
\newcommand{\bTheta}{\boldsymbol{\Theta}}
\newcommand{\bpi}{\boldsymbol{\pi}}
\newcommand{\bxi}{\boldsymbol{\xi}}
\newcommand{\bSigma}{\boldsymbol{\Sigma}}
\newcommand{\bPi}{\boldsymbol{\Pi}}
\newcommand{\bOmega}{\boldsymbol{\Omega}}
\newcommand{\brho}{\boldsymbol{\rho}}

\newcommand{\bgamma}{\boldsymbol{\gamma}}
\newcommand{\bGamma}{\boldsymbol{\Gamma}}
\newcommand{\bUpsilon}{\boldsymbol{\Upsilon}}
\newcommand{\barZ}{\bar{Z}}
\newcommand{\barz}{\bar{z}}
\newcommand{\whatR}{\widehat{R}}

\newcommand{\bmu}{\boldsymbol{\mu}}
\newcommand{\1}{{\bf 1}}
\newcommand{\0}{{\bf 0}}

\newcommand{\bs}{\backslash}
\newcommand{\ben}{\begin{enumerate}}
\newcommand{\een}{\end{enumerate}}

 \newcommand{\notS}{{\backslash S}}
 \newcommand{\nots}{{\backslash s}}
 \newcommand{\noti}{{\backslash i}}
 \newcommand{\notj}{{\backslash j}}
 \newcommand{\nott}{\backslash t}
 \newcommand{\notone}{{\backslash 1}}
 \newcommand{\nottp}{\backslash t+1}

\newcommand{\notk}{{^{\backslash k}}}
\newcommand{\notij}{{^{\backslash i,j}}}
\newcommand{\notg}{{^{\backslash g}}}
\newcommand{\wnoti}{{_{\w}^{\backslash i}}}
\newcommand{\wnotg}{{_{\w}^{\backslash g}}}
\newcommand{\vnotij}{{_{\v}^{\backslash i,j}}}
\newcommand{\vnotg}{{_{\v}^{\backslash g}}}
\newcommand{\half}{\frac{1}{2}}
\newcommand{\msgb}{m_{t \leftarrow t+1}}
\newcommand{\msgf}{m_{t \rightarrow t+1}}
\newcommand{\msgfp}{m_{t-1 \rightarrow t}}

\newcommand{\proj}[1]{{\rm proj}\negmedspace\left[#1\right]}
 \newcommand{\argmin}{\operatornamewithlimits{argmin}}
 \newcommand{\argmax}{\operatornamewithlimits{argmax}}

\newcommand{\dif}{\mathrm{d}}
\newcommand{\abs}[1]{\lvert#1\rvert}
\newcommand{\norm}[1]{\lVert#1\rVert}

\newcommand{\ie}{{\textit{i.e.,}}\xspace}
\newcommand{\etc}{{\textit{etc}.}\xspace}
\newcommand{\eg}{{{\textit{e.g.},}}\xspace}
\newcommand{\EE}{\mathbb{E}}
\newcommand{\HH}{\mathbb{H}}
\newcommand{\sbr}[1]{\left[#1\right]}
\newcommand{\rbr}[1]{\left(#1\right)}
\newcommand{\zhe}[1]{{\textcolor{blue}{#1}}}
\newcommand{\Vtr}{\mathrm{Vec}}
\newcommand{\tlam}{{\tilde{\lambda}}}
\newcommand{\tp}{{\widetilde{p}}}
\newcommand{\tmu}{{\widetilde{\mu}}}
\newcommand{\tv}{{\widetilde{v}}}
\newcommand{\talpha}{{\widetilde{\alpha}}}
\newcommand{\tomega}{{\widetilde{\omega}}}
\newcommand{\bkh}{{\backslash}}
\newcommand{\whmu}{\widehat{\bmu}}
\newcommand{\whV}{\widehat{\V}}

\newcommand{\YM}[1]{\textcolor{blue}{\small {\sf YM: #1}}}

%% file: abstract.tex
\begin{abstract}
Machine learning solvers for partial differential equations (PDEs) have attracted growing interest. However, most existing approaches, such as neural network solvers, rely on stochastic training, which is inefficient and typically requires a great many training epochs.  Gaussian process (GP)/kernel-based solvers, while mathematical principled, suffer from scalability issues when handling large numbers of collocation points often needed for challenging or higher-dimensional  PDEs.
 To overcome these limitations, we propose \ours, a tensor-GP-based solver  that introduces factor functions along each input dimension using one-dimensional GPs and combines them via tensor decomposition to approximate the full solution. This design reduces the task to learning a collection of one-dimensional GPs, substantially lowering computational complexity, and enabling scalability to massive collocation sets.
 For efficient nonlinear PDE solving, we use a partial freezing strategy and Newton's method to linerize the nonlinear terms. We then develop an alternating least squares (ALS) approach that admits closed-form  updates, thereby substantially enhancing the training efficiency. We establish theoretical guarantees on the expressivity of our model, together with convergence proof and error analysis  under standard regularity assumptions.  Experiments on several benchmark PDEs demonstrate that our method achieves superior accuracy and efficiency compared to existing approaches. The code  is released at \url{https://github.com/BayesianAIGroup/TGPSolve-NonLinear-PDEs}

\end{abstract}

%% file: intro.tex
%solvers --> problem of current GP solvers --> our method --> experiment
\section{Introduction}
Machine learning (ML) solvers for partial differential equations (PDEs) have been receiving increasing attention due to their ease of implementation and competitive accuracy. These approaches approximate the solution with a machine learning model, such as deep neural networks~\citep{raissi2019physics}, trained by minimizing a composite objective function that combines boundary and residual losses evaluated at a set of collocation points, thereby enforcing  the boundary conditions and  equation. Unlike traditional numerical solvers, ML approaches avoid complex, problem-specific discretization schemes and numerical routines, making them simpler and more convenient to implement and verify.

Despite these advantages, most existing ML solvers --- including physics-informed neural networks~\citep{raissi2019deep} and recent Gaussian process (GP) and kernel-based methods~\citep{fang2023solving,xu2024efficientkernelbasedsolversnonlinear} --- rely on stochastic optimization to effectively learn the model parameters, which often requires tens of thousands to even millions of iterations, making  solving procedure quite inefficient. GP and kernel-based solvers, though mathematically principled, also face scalability challenges as the number of collocation points grows. For example,~\citet{CHEN2021110668,long2022autoip} place a GP prior over the solution and its derivatives, yielding block-structured covariance matrices whose time and memory costs exceed the standard $\Ocal(M^3)$ and $\Ocal(M^2)$ scaling, where $M$ is the number of collocation points. To mitigate this,~\citet{fang2023solving,xutoward2025} proposed using product kernels on Cartesian grids, exploiting Kronecker algebra for efficient matrix operations. However, this approach requires estimating the solution values at \textit{all} grid points, leading to exponential growth in parameters with dimension, and its reliance on structured grids limits applicability to irregular domains.

To overcome these limitations, we propose \ours, a tensor Gaussian process solver for nonlinear PDEs. Our main contributions are as follows. 

%\begin{compactitem}
    %\item 
    \noindent\textbf{Model:} We introduce a set of one-dimensional factor functions, each modeled as a GP, along every input dimension. These factors are combined via multilinear tensor decompositions --- such as CANDECOMP/PARAFAC (CP)~\citep{harshman1970foundations} or tensor-ring decomposition~\citep{zhao2016tensor} --- to approximate the full solution. For each dimension, we place inducing points and represent the factor function by its GP conditional mean, with the inducing values serving as trainable parameters. Collocation points are freely sampled to form the training objective. This design allows  our model to scale linearly with both the PDE dimension and the number of collocation points --- not only in covariance matrix computation and storage but also in the number of trainable parameters.
    % We introduce a set of one-dimensional factor functions along each input dimension, each modeled as a GP. These factors are combined via a multilinear tensor decomposition --- such as CANDECOMP/PARAFAC (CP)~\citep{harshman1970foundations} or tensor-ring decomposition~\citep{zhao2016tensor} --- to approximate the full solution function. We place a number of inducing points in each dimension and use GP conditional mean to represent each factor function. The function values at the inducing points (inducing values) are treated as model parameters. Collocation points are freely sampled to construct the training objective.    In this way, our model scales linearly with both the PDE dimension and the number of collocation points --- not only in the computation and storage of covariance matrices but also in the number of trainable parameters.
    %\item 
    
   \noindent \textbf{Algorithm:} To efficiently solve nonlinear PDEs, we linearize the nonlinear terms using two complementary strategies. The first is a partial freezing strategy, which fixes part of the nonlinear terms using results from the previous iteration, leaving only a linear component. The second is Newton’s method, which approximates nonlinear terms via their first-order Taylor expansion. We then exploit two key properties of our model: (i) the solution approximation is multilinear in the inducing values, and (ii) derivatives of the solution preserve the same multilinear structure. Based on these insights, we design an alternating least squares (ALS) scheme that cyclically updates the inducing values along each dimension in closed form. This eliminates the need for stochastic optimization, yielding far greater efficiency and achieving accurate approximations with only a small number of iterations.
    
    %To enable efficient nonlinear PDE solving, we first linearize the nonlinear terms using  two strategies. One is  a partial freezing strategy, which  freezes part of the nonlinear terms with results in the previous iteration, only leaving a linear part.  The second one is Newton's method that approximates the nonlinear terms with their first-order Taylor approximation. Next, we leverage two important insights of our solution model: the solution approximation is multi-linear in the model parameters (\ie inducing values), and any derivative in the PDE maintains the same multi-linear structure. Based on the linearized PDE and these insights, we develop an alternating least squares (ALS) scheme, that cyclically updating inducing values along each dimension with a closed-form. In this way, our training avoid stochastic updates, is much more efficient, and can find a good solution approximation using much fewer number of iterations.  
    %\item 
    \noindent\textbf{Theorem:} We present a rigorous theoretical analysis of our framework. We show that, despite relying on a multilinear functional decomposition, our model can approximate the true solution arbitrarily well when provided with a sufficient number of factor functions (\ie rank). Under CP decomposition, we further prove that not only do such approximations exist within our modeling space, but also that, as the number of collocation points increases, the training optimum  converges to these approximations. These results theoretically guarantee the effectiveness of our method in recovering high-quality solutions.

    %We provide a rigorous analysis our framework. We show that our solution model, while using a multi-linear functional decomposition, can be arbitrarily close to the true solution with an enough number of factor functions (\ie rank). Under the CP decomposition, we further show that not only do such solution approximations exist in our modeling space,   by increasing the number of collocation points, solving the optimization problem in our training can indeed converge to such solution approximations. The results theoretically affirm the efficacy
%of our method in finding high-quality solution approximations. 
    %\item 
    \noindent\textbf{Experiments:} 
    We evaluate our method on a range of benchmark PDEs.\cmt{, including Burgers’, nonlinear elliptic, Eikonal, Allen–Cahn, and nonlinear Darcy flow equations.} In less challenging settings, where only a modest number of collocation points (\eg 1,000) suffices, our method consistently achieves lower or comparable errors than existing approaches. In more challenging cases --- such as Burgers’ equation with viscosity $0.001$ or a 6D Allen-Cahn equation --- our method seamlessly scales to tens of thousands of collocation points, achieving errors on the order of $10^{-3}$ to $10^{-6}$. Across all benchmarks, it runs orders of magnitude faster than PINNs and recent GP solvers, while delivering comparable or superior accuracy with drastically reduced runtime.

%% file: method.tex
\section{Background}
Consider solving a PDE of the general form:
\begin{align}
	\Pcal(u) = a(\x) \;\; (\x \in \Omega), \quad	\Bcal(u) = b(\x) \;\; (\x \in \partial \Omega), \label{eq:pde}
\end{align}
where $\x = (x_1, \ldots, x_d)^\top$, $\Omega$ and $\partial \Omega$ denote  the interior and boundary domains respectively, $\Pcal$ and $\Bcal$ are (possibly nonlinear) differential operators applied to $u$.% and its derivatives. 

\noindent\textbf{Physics-Informed Neural Networks (PINNs).} To solve~\eqref{eq:pde}, PINNs approximate the PDE solution using a (deep) neural network. A set of  collocation points $\Mcal = \{\x_1, \ldots, \x_{M_\Omega} \in \Omega, \x_{M_\Omega+1}, \ldots, \x_{M} \in \partial \Omega\}$ is sampled, and the network is trained by minimizing the loss, 
\begin{align}
    \Theta^* = \argmin_{\Theta} \lambda_b \Fcal_b(\Theta) + \Fcal_r(\Theta)\label{eq:pinn-loss}
\end{align}
where $\Fcal_r = \frac{1}{M_\Omega}\sum_{m=1}^{M_\Omega}\left(\Pcal(\text{NN}_\Theta)(\x_m) - a(\x_m)\right)^2$ is the PDE  residual loss, and $\Fcal_b = \frac{1}{M - M_\Omega} \sum_{j=1}^{M-M_\Omega} \left(\Bcal\left(\text{NN}_\Theta\right)(\x_{M_\Omega + j}) - b(\x_{M_\Omega + j})\right)^2$ is the boundary loss. Here NN denotes the neural network, $\Theta$ its parameters, and $\lambda_b>0$ a weighting coefficient. Training typically relies on stochastic optimization (\eg ADAM~\citep{kingma2014adam}), as in standard neural network applications like image classification.
Achieving good accuracy generally requires tens of thousands of iterations, and a second-order optimizer (\eg L-BFGS) is often used for refinement and stabilization~\citep{shin2020convergence,li2023meta,penwarden2023unified}. As a result, while this framework is straightforward and convenient to implement, the training (solving) process is often lengthy and inefficient.

\noindent\textbf{Gaussian Process and Kernel-Based Solvers.} %An alternative class of solvers is based on GP/kernel methods, which rest on strong mathematical foundations. In \citep{CHEN2021110668, long2022autoip} , a GP prior is placed over the solution $u$ and all its derivatives (or more generally, linear operators) $D_j(u)$ appearing in the PDE. The goal is to estimate the values of $u$ and all $D_j(u)$  evaluated at the collocation points, denoted collectively by $\z$. This leads to a multi-variate Gaussian prior distribution over $\z$ with a block covariance matrix, $\C = \{\C_{ij}\}$, where each block $\C_{ij}= \cov(\z_i, \z_j) = D_i \circ D_j(\kappa)(\z_i, \z_j)$ is  associated with a pair of linear operators (\eg derivatives and $u$ itself), $\kappa$ is the covariance function for $u$, $\z_i$ and $\z_j\subset \z$ represent the values of $D_{i}(u)$ and $D_{j}(u)$ evaluated at the corresponding collocation points, respectively.  Since each collocation point can contribute to multiple entries in $\z$ (via different $D_j(u)$), the covariance matrix $\C$ is typically larger than $M \times M$, where $M$ is the number of collocation points. As $M$ increases, the time and space complexity exceed $\mathcal{O}(M^3)$ and $\mathcal{O}(M^2)$, respectively, making computation prohibitively expensive or even infeasible for large $M$.
An alternative class of solvers is based on GP/kernel methods, which rest on strong mathematical foundations. In \citep{CHEN2021110668, long2022autoip} , a GP prior is placed over the solution $u$ and all its derivatives (or more generally, linear operators) $D_j(u)$ appearing in the PDE. The goal is to estimate the values of $u$ and all $D_j(u)$  evaluated at the collocation points,  which leads to a multi-variate Gaussian prior distribution  with a block covariance matrix, $\C = \{\C_{ij}\}$, where each block $\C_{ij}$ is  associated with a pair of linear operators (\eg derivatives and $u$ itself).  Since each collocation point can contribute to multiple values (\eg different $D_j(u)$), the covariance matrix $\C$ is typically larger than $M \times M$, where $M$ is the number of collocation points. As $M$ increases, the time and space complexity exceed $\mathcal{O}(M^3)$ and $\mathcal{O}(M^2)$, respectively, making computation prohibitively expensive or even infeasible for large $M$.
  
To mitigate this issue, recent work by \citet{xutoward2025} approximates the solution using standard GP/kernel interpolation: $u(\x; \u_{\Mcal}) = \kappa(\x, \Mcal)\K_{MM}^{-1} \u_{\Mcal}$, where $\K_{MM} = \kappa(\Mcal, \Mcal)$ is the $M \times M$ covariance matrix at the collocation points, and $\u_{\Mcal}$ denotes the solution values at these points. Differential operators $D_j$ are applied directly to this interpolation to approximate $D_j(u)$, enabling the use of a training loss similar to~\eqref{eq:pinn-loss}. To reduce computation cost, collocation points are placed on a Cartesian grid: $\Mcal = \s^1 \times \ldots \times \s^d$ where each $\s^j$ is a set of input locations along dimension $j$.  By adopting a product kernel of the form $\kappa(\mathbf{x}, \mathbf{x}') = \prod_j \kappa_j(x_j, x_j')$, the  covariance matrix admits a Kronecker product structure: $\kappa(\Mcal, \Mcal) = \K_1 \kron \ldots \kron \K_d$, where each $\K_j = \kappa_j(\s^j, \s^j)$ is the kernel matrix along input dimension $j$. Kronecker algebra~\citep{kolda2006multilinear} allows one to avoid computing the full $\K_{MM}$; instead, only the smaller matrices $\K_j$ are computed and inverted, greatly reducing the cost.  However, this method still requires estimating $\Ucal_{\Mcal}$ --- the solution values over the entire grid $\Mcal$. As the input dimension increases, the size of $\Ucal_{\mathcal{M}}$ grows exponentially, making storage and estimation infeasible.
Furthermore, effective training still relies on stochastic optimization, often requiring up to one million iterations~\citep{xutoward2025}, rendering the solving procedure inefficient in practice.

\section{Our Method}
%\vspace{-0.05in}
\subsection{Model}\label{sect:model}
%\vspace{-0.05in}
To leverage the principled mathematical framework of  GPs/kernel methods while overcoming their scalability and efficiency bottlenecks, we propose \ours, a tensor-GP based PDE solver.  Specifically, we introduce a set of one-dimensional factor (component) functions for each input dimension $i$, 
\begin{align}
    %\zeta^i = \{f^i_1(\cdot), \ldots, f^i_{R_i}(\cdot)\}, \quad\quad 
    f^i_r: \Omega_0 \subset \mathbb{R} \rightarrow \mathbb{R} \in \Gcal^i \;(1 \le r \le R_i),
\end{align}
where $1 \le i \le d$, and $\Gcal^i$ is a Reproducing Kernel Hilbert Space (RKHS) induced by a Mercer kernel function $\kappa_i(\cdot, \cdot)$. In other words, we model each factor function as a GP, $f^i_r\sim \mathcal{GP}(0, \kappa_i(\cdot, \cdot))$.  For simplicity of discussion, we assume the PDE domain $\Omega \cup \partial \Omega \subseteq \bigtimes_{i=1}^d\Omega_0$. We then combine  these factor functions via mutilinear tensor decomposition to construct the solution approximation. We consider two representative tensor decomposition models.  

\noindent \textbf{CANDECOMP/PARAFAC(CP) Decomposition}~\citep{harshman1970foundations}. \cmt{The first one is the CP decomposition.} We set $R_1 = \ldots = R_d = R$, and model the solution function as 
\begin{align}
	&u(x_1, \ldots, x_d) = \sum\nolimits_{r=1}^R \prod\nolimits_{i=1}^d f^i_r(x_i) \notag \\
	&= \left(\f^1(x_1)\circ \ldots \circ \f^d(x_d)\right)^\top \1, \label{eq:cp}
\end{align}
where each $\f^i(x_i) = \left(f^i_1(x_i), \ldots, f^i_R(x_i)\right)^\top$, and $\circ$ is Hadamard (element-wise) product.  

\noindent\textbf{Tensor-Ring (TR) Decomposition}~\citep{zhao2016tensor}. %The second one is the TR decomposition. 
In each dimension $i$, we view the factor functions together as a single function with a matrix output,  $\F^i = \{f^i_r(\cdot)\}: \mathbb{R} \rightarrow \mathbb{R}^{R_{i-1} \times R_{i}}$ where $R_0 = R_d$. The solution is modeled as
\begin{align}
	u(x_1, \ldots, x_d) = \text{Trace}(\F^1(x_1) \F^2(x_2) \cdots   \F^d(x_d)).\label{eq:tr}
\end{align}
When $d=2$, TR decomposition reduces to CP decomposition: $\text{Trace}\left(\F^1(x_1)\F^2(x_2)\right) = {\f^1(x_1)}^\top {\f^2(x_2)} =\left(\f^1(x_1) \circ \f^2(x_2)\right)^\top \1$ where $\f^1(x_1) = \vec\left(\left(\F^1(x_1)\right)^\top\right)$, $\f^2(x_2) = \vec(\F^2(x_2))$, and $\vec$ denotes vectorization.%,  which is consistent with~\eqref{eq:cp}.
%Note that the popular tensor-train format~\citep{oseledets2011tensor} is a special case where $R_0 = R_d = 1$. 

Our formulation can be interpreted as a multilinear decomposition in functional space. 
To learn these factor functions, we introduce a set of inducing locations $\bgamma_i = (\gamma_1, \ldots, \gamma_{N_i})^\top$ in each dimension $i$, and represent each factor function as kernel interpolation (GP conditional mean),  
\begin{align}
f^i_r(x_i) = \kappa_i(x_i, \bgamma_i) \K_i^{-1} \boldeta^i_r, \label{eq:factor-func}
\end{align}
where $\K_i = \kappa_i(\bgamma_i, \bgamma_i)$, and $\boldeta^i_r = \left(f^i_r(\gamma_1), \ldots, f^i_r(\gamma_{N_i})\right)^\top$ denotes the values of the factor function at the inducing locations, \ie {inducing values}. The learning is conducted by solving the following constrained optimization problem, 
\begin{align}
    \begin{cases}
        \underset{\{f^i_r\in\Gcal^i\}}{\text{minimize}}\;\; \sum_{i=1}^d\sum_{r=1}^R\|f^i_r\|^2_{\Gcal^i}  \\
        \text{s.t.  }
        \frac{1}{M_\Omega} \sum_{m=1}^{M_\Omega} \left(\Pcal(u)(\x_m) - a(\x_m)\right)^2   \\
        +\frac{1}{M-M_\Omega}\sum_{m=M_\Omega+1}^M \left(\Bcal(u)(\x_m) - b(\x_m)\right)^2 \le \delta^2, \\
        f^i_r \text{ takes kernel interpolation form}~\eqref{eq:factor-func},\\
        u \text{ takes the tensor decomposition form }~\eqref{eq:cp} \text{ or }~\eqref{eq:tr},
        \raisetag{0.5in} \label{eq:optimization-formulation}
    \end{cases}
\end{align}
where $\|\cdot \|_{\Gcal^i}$ is the RKHS norm under $\Gcal^i$, and $\delta$ is a relaxation parameter. Since we approximate the full solution function with a multilinear (low-rank) decomposition, we introduce  $\delta^2>0$ to guarantee the feasibility of  optimization, and to establish convergence. See Section~\ref{sect:theorem} for our theoretical analysis. Directly solving~\eqref{eq:optimization-formulation} can be unwieldy in practice. We may choose to solve an unconstrained optimization with soft regularization instead, 
\begin{align}
    &\underset{\{\boldeta^i_r\}}{\text{minimize}} \;\; \mathcal \Fcal(u(\x; \{\boldeta^i_r\}); \alpha_1, \alpha_2):= \sum_{i=1}^d\sum_{r=1}^R\|f^i_r\|_{\Gcal^i}^2 \notag \\
    &+ \alpha_1 \left[\frac{1}{M_{\Omega}}\sum\nolimits_{m=1}^{M_\Omega}( \Pcal(u)(\x_m) - a(\x_m) )^2-\delta^2/2\right] \label{eq:our-loss}\\
    & + \alpha_2 \left[\frac{1}{M-M_{\Omega}}\sum\nolimits_{m={M_\Omega+1}}^M ( \Bcal(u)(\x_m) - b(\x_m))^2 - \delta^2/2\right], \notag 
\end{align}
where $\alpha_1, \alpha_2>0$ represent regularization strength, and $\delta$ can be simply set to zero. 

Minimizing~\eqref{eq:our-loss}  is equivalent to maximizing the log joint probability of our tensor GP model, \ie performing probabilistic training. Specifically, from the interpolation form~\eqref{eq:factor-func}, each squared RKHS norm is  $\|f^i_r\|_{\Gcal^i}^2 = \left(\boldeta^i_r\right)^\top \K_i^{-1} \boldeta^i_r$, which corresponds to the negative log prior probability of $\boldeta^i_r$ under the GP prior over $f^i_r$, namely, $p(\boldeta^i_r) = \N(\boldeta^i_r|\0, \K_i)$. Meanwhile, the residual and boundary loss terms at each collocation point correspond to negative log Gaussian likelihoods, given by $N\left(a(\x_m)| \Pcal(u)(\x_m), \alpha_1^{-1}\right)$ and $\N\left(b(\x_m)|\Bcal(u)(\x_m), \alpha_2^{-1}\right)$.

%talk about the problem of SGD, key insight, as an example, and talk about ALS.
\subsection{Algorithm}
%\vspace{-0.1in}
While applying stochastic optimization to~\eqref{eq:our-loss} is straightforward, it typically requires many iterations and is therefore inefficient.\cmt{ for estimating the inducing values $\{\boldeta_r^i\}_{i,r}$.} To improve learning efficiency and achieve higher accuracy, we propose an alternating least squares (ALS) approach that performs closed-form updates across input dimensions. For clarity, we focus on the CP decomposition in~\eqref{eq:cp}, noting that the TR decomposition extends naturally. To illustrate the idea, we present a nonlinear 2D Allen-Cahn equation as a concrete example, 
\begin{align}
u_{x_1x_1} + u_{x_2x_2} + u(u^2-1) = a(x_1, x_2), \label{eq:pde-example}
\end{align}
where $a(x_1, x_2)$ is the source function. 

First, we observe two key properties of our tensor-GP model: (i)  the solution approximation is mutilinear in the inducing values in each dimension $i$, denoted as $\H_i = [\boldeta^i_1, \ldots, \boldeta^i_R] \in \mathbb{R}^{N_i \times R}$. Specifically, according to~\eqref{eq:cp}, we have $\f^i(x_i) = \left(\kappa_i(x_i, \bgamma_i)\K_i^{-1} \H_i\right)^\top=\H_i^\top \w_i(x_i)$ where $\w_i(x_i) =\K_i^{-1}\kappa_i(\bgamma_i, x_i)$, and
\begin{align}
    u(x_1, \ldots, x_d) = \left\langle \circ_{i=1}^d \H_i^\top \w_i(x_i), \1 \right\rangle, \label{eq:cp-form2}
\end{align}
where $\langle \cdot, \cdot \rangle$ is the dot product. Hence, we have $u$ is linear in each $\H_i$ when fixing inducing values in other dimensions. \cmt{: 
$u(x_1, \ldots, x_d) = \w_i(x_i)^\top \H_i \bbeta_{i} = \text{trace}(\bbeta_i \w_i(x_i)^\top \H_i)$ where $\bbeta_i = \circ_{j\neq i} \H_j^\top \w_j(x_j)$.} For example, the solution of~\eqref{eq:pde-example} is modeled as $u(x_1, x_2) = \w_1(x_1)^\top \H_1 \H_2^\top \w_2(x_2)$.
%\begin{align}
%    u(x_1, x_2) = \w_1(x_1)^\top \H_1 \H_2^\top \w_2(x_2).
%\end{align}
(ii) Due to the mutlilinear combination of 1D functions, any partial derivative over our solution approximation maintains the same mutilinear structure in~\eqref{eq:cp-form2}: $\partial_{x_{j_1}\ldots x_{j_K} } u =\left\langle \circ_{i=1}^d \H_i^\top \widehat{\w}_i(x_i), \1 \right\rangle  $ where $\widehat{\w}_i(x_i)$ is the derivative of $\w_i(x_i)$ if  $x_i \in \{x_{j_1}, \ldots, x_{j_K}\}$ and otherwise $\widehat{\w}_i(x_i) = \w_i(x_i)$. Therefore, the partial derivatives are still multilinear in $\H_i$'s. For instance, for~\eqref{eq:pde-example}, we have
\begin{align}
    %u_{x_1x_1} &= \left(\partial_{x_1x_1} \w_1(x_1)\right)^\top \H_1 \H_2^\top \w_2(x_2), \notag \\
    %u_{x_2x_2} &= \w_1(x_1)^\top \H_1 \H_2^\top \partial_{x_2x_2} \w_2(x_2). \label{eq:u-der-ex}
       u_{x_1x_1} &= \left( \w_1''(x_1)\right)^\top \H_1 \H_2^\top \w_2(x_2), \notag \\
    u_{x_2x_2} &= \w_1(x_1)^\top \H_1 \H_2^\top  \w_2''(x_2),  \label{eq:u-der-ex}
\end{align}
where $\w_1''(x_1) = \partial^2 \w_1/\partial x_1^2$ and $\w_2''(x_2) = \partial^2 \w_2/\partial x_2^2$.

\cmt{
Furthermore, the squared RKHS norm in~\eqref{eq:optimization-formulation} and ~\eqref{eq:our-loss} is quadratic to each $\H_i$,
\begin{align}
 \sum_{r=1}^R \|f^i_r\|^2_{\Gcal^i} = \sum_{r=1}^R \left(\boldeta^i_r\right)^\top \K_i^{-1} \boldeta^i_r =  \text{Trace}(\K_i^{-1}\H_i\H_i^\top)   \label{eq:rkhs}. 
\end{align}
}
%\begin{align}
%    &\sum_{i=1}^d \sum_{r=1}^R \|f^i_r\|^2_{\Gcal^i} = \sum_{i=1}^d \sum_{r=1}^R \left(\boldeta^i_r\right)^\top \K_i^{-1} \boldeta^i_r = \sum_{i=1}^d \text{Trace}(\K_i^{-1}\H_i\H_i^\top) \notag \\
%    &= \sum_{i=1}^d \vec(\H_i)^\top \left(\K_i^{-1} \kron \I\right)\vec(\H_i). \label{eq:rkhs}
%\end{align}

Therefore, if the operators $\Pcal$ and $\Bcal$ in~\eqref{eq:pde} are linear in $u$ and its derivatives,  the squared residual and boundary loss at each collocation point is quadratic to each $\H_i$ (see~\eqref{eq:optimization-formulation} and~\eqref{eq:our-loss}). For example, suppose $d=2$ and $\Pcal[u] = u_{x_1x_1} + u_{x_2x_2}$. At collocation point $\x_m=(x_{m1}, x_{m2})$, the residual loss w.r.t $\H_1$ takes the form,
\begin{align}
	\left(\Pcal(u)(\x_m) - a(x_m)\right)^2 = \left(\tr(\B_m^\top\H_1) - a(\x_m)\right)^2 \label{eq:sq-loss}
    %&\left(\Pcal(u)(\x_m) - a(x_m)\right)^2 = \left(\tr(\B(\x_m)^\top\H_1) - a(\x_m)\right)^2 \label{eq:sq-loss}%\notag \\
    %&= \left(\vec\left(\B(\x_m)\right)^\top \vec(\H_1) - a(\x_m) \right)^2, \label{eq:sq-loss}
\end{align}
where $\B_m = \left(\w_1''(x_{m1})\w_2^\top(x_{m2}) + \w_1(x_{m1})\w_2''(x_{m2})^\top\right)\cdot\H_2$,
 %$\B_m= \left(\partial_{x_{m1}x_{m1}}\w_1(x_{m1}) \w_2^\top(x_2) + \w_1(x_1)\partial_{x_2x_2}\w_2^\top(x_2)\right)\cdot\H_2$
 according to~\eqref{eq:u-der-ex}. The residual loss w.r.t $\H_2$ takes a similar form.  
 
 Furthermore, we observe that the sum of squared RKHS norms in~\eqref{eq:optimization-formulation} and ~\eqref{eq:our-loss} is also quadratic to each $\H_i$,
 \begin{align}
 	\sum_{r=1}^R \|f^i_r\|^2_{\Gcal^i} = \sum\nolimits_{r=1}^R \left(\boldeta^i_r\right)^\top \K_i^{-1} \boldeta^i_r =  \text{Trace}(\K_i^{-1}\H_i\H_i^\top).  \notag 
 \end{align}
 Combining this with~\eqref{eq:sq-loss}, we see that \textit{optimizing any $\H_i$ while holding the other inducing values fixed reduces to a least-squares problem with a closed-form solution}. This naturally suggests an alternating least squares (ALS) scheme, where we cyclically update each $\H^i$ while keeping the others fixed. Unlike stochastic gradient descent, which is noisy, inaccurate, and requires carefully adjusted small stepsizes to prevent divergence, ALS updates are more direct and aggressive, leading to substantially higher efficiency.

However, a critical bottleneck arises when  $\Pcal$ and $\Bcal$ are nonlinear. The nonlinear terms\cmt{ in the PDE and boundary conditions} disrupt the mutilinear structure, making ALS infeasible. To address this issue, we use two complementary strategies to linearize the nonlinear terms. 

\noindent \textbf{Partial Freezing.} The first strategy is to freeze part of the nonlinear terms from the previous iteration, leaving only a linear component. For example, consider the nonlinear term $u(u^2 - 1)$ in~\eqref{eq:pde-example} as an example. We freeze $u^2 -1$ and approximate $u(u^2-1) \approx u \cdot \left((u^{\text{prev}})^2-1\right)$,  where $u^{\text{prev}}$ is computed from the factor functions estimated in the previous iteration. At the collocation points, the values of $\left(u^{\text{prev}}\right)^2-1$ are treated as constants, making the entire term mutilinear in the $\H_i$'s. As the updates proceed, the discrepency between $u^\text{prev}$ and $u$ gradually diminishes, and vanishes upon convergence. 
    
\noindent \textbf{Newton's method (Tangent Linearization).} Consider solving the PDE as a root finding problem: $\mathcal{R}(u) = 0$, where $\Rcal$ is the PDE operator. We apply Newton's method by linearizing $\Rcal(u)$ around the previous iteration $u^{\text{prev}}$ via a first-order Taylor expansion: $\Rcal(u) \approx \Rcal(u^{\text{prev}}) + J(u^{\text{prev}})(u - u^{\text{prev}}) = 0$,  
    where $J(u^{\text{prev}})$ is the Fr\'echet derivative, $J(u^{\text{prev}}) = \frac{\d \mathcal{R}(u^{\text{prev}}+ \epsilon v)}{\d \epsilon}\bigg|_{\epsilon = 0}$, and $v$ is an arbitrarily small perturbation. By construction, this first-order approximation is always linear in $u$. In practice, the procedure amounts to replacing the nonlinear terms in the PDE and boundary conditions with their first-order Taylor expansions. For example, in~\eqref{eq:pde-example}, the cubic term $u^3$ is replaced by $(u^{\text{prev}})^3 + 3 (u^{\text{prev}})^2(u - u^{\text{prev}})$. 
%\end{itemize}

\textbf{Computational Complexity.} At each iteration, we linearize the nonlinear terms and perform ALS updates on each $\H_i$ while holding the others fixed. The overall time complexity is thereby $\Ocal(M\sum_{i=1}^d (N_iR)^2 + N_i^3R^3)$, where $M$ is the number of collocation points, $N_i$ the number of inducing points in dimension $i$, and $R$ the number of factor functions per dimension. Thus, the time complexity scales linearly with both the number of collocation points and the input dimension. The space complexity is $\Ocal(\sum_{i=1}^d N_iR + N_i^2)$, which accounts for storing the inducing values $\H_i$ and the kernel matrices $\K_i$ in each dimension. 
%Unlike previous GP/kernel-based solvers~\citep{chen2021physics,xutoward2025} that require estimating the solution values (and their derivatives) at every collocation point, our method decouples inducing points from collocation points. As a result, the number of inducing points per dimension --- typically only a few hundred --- can be kept much smaller than the number of collocation points, substantially reducing both kernel matrix size and the number of inducing values (\ie model parameters). This design significantly lowers computational and storage costs, enabling our method to easily scale to very large collocation sets. 

\section{Theoretical Analysis}\label{sect:theorem}
We first show that, while our model adopts a multilinear functional decomposition as in~\eqref{eq:cp} and~\eqref{eq:tr}, this decomposition remains sufficiently expressive to accurately approximate the true solution $u^*$ under Sobolev regularity.

\begin{lemma}\label{thm:1}
	Suppose that $u^* \in H^k(\Omega)$ for some $k\in\mathbb N$, where  $H^k(\Omega)$ denotes the Sobolev space consisting of functions whose weak derivatives up to the order of $k$ have finite $L^2$ norms. Given an arbitrary error $0<\varepsilon<1$, if we model $u$ as the CP format in \eqref{eq:cp} and let $\overline{R}=(\frac{\sqrt{d}}{\varepsilon})^{\frac{d-1}{k}}$, then  
	\begin{align}
		\min_{\substack{R\leq \overline{R}, \;
				u^i_{r}\in L^2(\Omega_0)}}\left\|u^* - u\right\|_{L^2(\Omega)}\lesssim\varepsilon.\label{lem:key1}
	\end{align}
Moreover, if we further assume $u^*\in H_{\v}^k(\Omega)$ for some $\v \in \mathbb{R}_+^D$ satisfying $v_j\lesssim j^{-(1+\delta')/k}$ for some $\delta'>\delta + k$ where $\delta>0$, then \eqref{lem:key1} holds with $\overline{R} \lesssim(\frac{1}{\varepsilon})^{\frac{1}{k}}$, where the implicit constants depend on $\delta$.  Here  $H^k_\v(\Omega)$ is a weighted Sobolev spaces defined in Appendix Definition~\ref{def:sobolev}.
\end{lemma}
%\begin{proof}
%	The proof leverages existing approximation results for the Tucker decomposition in \citep{griebel2023analysis}. Since the Tucker format in \eqref{tucker} can be viewed as a special instance of CP in \eqref{cp} with $N = r_1\cdots r_m$, Theorem~\ref{thm:1} is implied by \citep[Theorems 2-3]{griebel2023analysis}. 
%\end{proof}

%The approximation results for TT are known \citep[Theorems 4-5 and Remark 2]{griebel2023analysis}, which we state below for completeness. 

\begin{lemma}\label{thm:2}
	Suppose that $u^*\in H^{k+1}(\Omega)$ for some $k\in\mathbb N$. Given an arbitrary error $0<\varepsilon<1$,  if we model $u$ as the TR format in \eqref{eq:tr} and let $\overline{R}=(\frac{\sqrt{d}}{\varepsilon})^{\frac{d}{k}}$, then  
	{\small\begin{align}
			\min_{\substack{\sum_{i=1}^dR_{i-1}R_{i}\leq \overline{R}, \;
					u^i_{r}\in L^2(\Omega_0)}}\left\|u^* - u\right\|_{L^2(\Omega)}\lesssim\varepsilon.\label{lem:key2}
	\end{align}}
If we further assume $u^*\in H_{\v}^{k+1}(\Omega)$ for some $\v \in \mathbb{R}_+^d$ satisfying $v_j\lesssim j^{-(1+\delta')/k}$ for some $\delta'>\delta + k$ where $\delta>0$, then \eqref{lem:key2} holds with $\overline{R} \lesssim \exp\{2(\frac{1}{\varepsilon})^{\frac{1}{k}}+ \log(\frac{1}{\varepsilon})\}$, where the implicit constants depend on $\delta$.  
\end{lemma}

Next, we establish the convergence result under the CP format. Our analysis follows the roadmap of \citet{batlle2023error,xutoward2025}, adopting standard assumptions on PDE stability and on the regularity of the domain and boundary conditions~\citep[Assumption 4.1]{xutoward2025}. %These assumptions, along with minor variants, are standard and mild in convergence analysis. Examples of nonlinear PDEs satisfying them can be found in~\citet{batlle2023error}.

\begin{assumption}\label{assump1}
The following conditions hold: 
\begin{compactitem}
    \item \textit{(C1) (Regularity of the domain and its boundary)} $\Omega \subset \mathbb{R}^d$ with $d>1$ is a compact set and $\partial \Omega$ is a smooth connected Riemannian manifold of dimension $d-1$ endowed with a geodesic distance $\rho_{\partial \omega}$. 
    \item \textit{(C2) (Stability of the PDE)} $\exists k, t \in \mathbb{N}$ with $k>d/2$ and $t>(d-1)/2$, 
    % \zhec{$\exists q_0, q_1>d$ such that $k-\frac{d}{2}\ge 1 - \frac{d}{q_0}$, $t-\frac{d-1}{2}\ge 1 - \frac{d-1}{q_1}$}, 
    and $\exists s, l \in \mathbb{R}$ such that for any $r>0$, it holds that $\forall u_1, u_2 \in B_r(H^l(\Omega))$,
    \begin{align}
        &\|u_1 - u_2\|_{H^l(\Omega)} \le C \big(\|\Pcal(u_1) - \Pcal(u_2)\|_{H^{{0}}(\Omega)} \notag \\
        &+ \|\Bcal(u_1) - \Bcal(u_2)\|_{H^{{0}}(\partial \Omega)}\big), \label{eq:stability-left}
    \end{align}
    and $\forall u_1, u_2 \in B_r(H^s(\Omega))$, 
    \begin{align}
        &\|\Pcal(u_1) - \Pcal(u_2)\|_{H^{k}(\Omega)} + \|\Bcal(u_1)- \Bcal(u_2)\|_{H^{t}(\partial \Omega)} \notag \\
        &\le C \|u_1 - u_2\|_{H^s(\Omega)},\label{eq:stability-right}
    \end{align}
    where $C=C(r)>0$ is a constant independent of $u_1$ and $u_2$, $B(r)$ is an open ball with radius $r$, $H^j = W^{j,2}$ is a Sobolev space.%where each element and its weak derivatives up to the order of $j$ have a finite $L^2$ norm. 
    \item \textit{(C3)} The RKHS $\Ucal$ is continuously embedded in $H^{s+\tau}(\Omega)$ where $\tau>0$.
\end{compactitem}\label{a1}
\end{assumption}

\begin{lemma}\label{th:convergence}
Let $u^* \in \Ucal$ denote the unique strong solution of~\eqref{eq:pde}, and suppose Assumption~\ref{a1} holds. Let $\Mcal = \Mcal_\Omega \cup \Mcal_{\partial \Omega}$ be a set of collocation points, with $\Mcal_\Omega \subset \Omega$ and $\Mcal_{\partial \Omega} \subset \partial \Omega$. Assume the Voronoi diagram induced by $\Mcal$ has a uniformly bounded aspect ratio across all cells.\cmt{\footnote{For each cell $\Tcal_m$ containing $\x_m$, the aspect ratio is defined as the ratio between the smallest radius of a ball centered at $\x_m$ containing ${\Tcal_m}$ and the largest radius of a ball centered at $\x_m$ contained in $\Tcal_m$. This ratio is assumed uniformly bounded across $m$ and independent of $|\Mcal|$.}}
% Let $u^* \in \Ucal$ denote the unique strong solution of \eqref{eq:pde}.
%     Suppose Assumption \ref{a1} holds, and a collection of collocation points $\Mcal =\Mcal_\Omega \cup \Mcal_{\partial \Omega}$ is given, where $\Mcal_{\Omega} \subset \Mcal$ denotes the collocation points in the interior $\Omega$ and $\Mcal_{\partial \Omega} \subset \Mcal$ the collocation points on the boundary $\partial \Omega$. Assume the Voronoi diagram based on the collocation points has a uniformly bounded aspect ratio across all the cells~\footnote{Specifically, for each cell $\Tcal_m$ that includes $\x_m$, the aspect ratio is defined as the ratio between the smallest radius of a ball centered at $\x_m$ containing  the closure of $\Tcal_m$ and the maximum radius of the ball centered at $\x_m$ contained in $\Tcal_m$. The aspect ratio is assumed to be uniformly bounded across $m$ and the total number of collocation points.}. 
    Define the fill-distances $h_{\Omega}:= \underset{\x\in \Omega}\sup\;\underset{\x' \in \Mcal_\Omega}\inf |\x -\x'|$ and $h_{\partial \Omega} :=\underset{\x \in \partial \Omega}\sup\;\;\underset{\x' \in \Mcal_{\partial \Omega}}\inf\rho_{\partial \Omega}(\x, \x')$, 
    %\begin{align}
    %    &h_{\Omega}:= \underset{\x\in \Omega}\sup\;\underset{\x' \in \Mcal_\Omega}\inf |\x -\x'|, \quad
    %    &h_{\partial \Omega} :=\underset{\x \in \partial \Omega}\sup\;\;\underset{\x' \in \Mcal_{\partial \Omega}}\inf\rho_{\partial \Omega}(\x, \x'),
    %\end{align}
    where $|\cdot|$ is the Euclidean distance, and $\rho_{\partial \Omega}$ is a geodesic distance defined on $\partial \Omega$. Set $h = \text{max}(h_\Omega, h_{\partial \Omega})$.
Suppose each RKHS $\Gcal^i$, to which the factor functions in dimension $i$ belong, is associated with a universal kernel, and that $\Ucal = \Gcal^1 \kron \cdots \kron \Gcal^d$\footnote{Here $\Ucal$ is a completed tensor product space, and $\kron$ is  the topological tensor product, \ie the completion of the algebraic tensor product with respect to the product RKHS norm rather than the algebraic tensor product.}.  Then, for any arbitrarily small $\varepsilon>0$, with a  sufficiently large $R$ and an appropriate $\delta$, the optimization (learning) problem~\eqref{eq:optimization-formulation} under the CP format~\eqref{eq:cp}, always has a minimizer $u^\dagger$, and this minimizer satisfies $\|u^\dagger - u^*\|_{H^l(\Omega)} \lesssim \varepsilon$ as $h \rightarrow 0$.
\end{lemma}
\begin{proposition}\label{th:prop}
Given the same set of collocation points $\Mcal$ and $\delta$, there exist constants $\alpha_{1\Mcal}, \alpha_{2\Mcal} > 0$ such that the minimizer of~\eqref{eq:our-loss} with $\alpha_1 = \alpha_{1\Mcal}$ and $\alpha_{2} = \alpha_{2\Mcal}$ coincides with the minimizer of~\eqref{eq:optimization-formulation}. In other words, with appropriately chosen regularization strengths, the minimizer of~\eqref{eq:our-loss} inherits the same convergence guarantee.
  %  Given the same set of collocation points $\Mcal$ and $\epsilon$, there exists $\tau_{1M}, \tau_{2M}>0$ such that the minimizer of \eqref{eq:our-loss} with $\tau_1=\alpha_M$ and $\tau_2=\beta_M$ is also the minimizer of \eqref{eq:optimization-formulation}. That means, with proper choices of the regularization strengths, the minimizer of \eqref{eq:our-loss} enjoys the same convergence result.
\end{proposition}
This result highlights that, as long as the model space is sufficiently expressive to contain a good approximation of the true solution (up to an arbitrarily small error level $\varepsilon$), our training formulation is able to recover such an approximation. In particular, as the number of collocation points increases, the optimization (learning) is guided toward identifying this accurate solution candidate. %Intuitively, the richer sampling of the domain and boundary conditions reduces the discrepancy between the discrete and continuous formulations, ensuring that the learned solution better reflects the true PDE solution. 
The proofs of these theorems are provided in Appendix Section~\ref{sect:proof}, ~\ref{sect:proof-convergence} and~\ref{sect:prop}.

%% file: related.tex
%\vspace{-0.1in}
\section{Related Work}
%\vspace{-0.05in}
\label{gen_inst}
While PINNs have achieved many success stories, \eg~\citep{raissi2020hidden, jin2021nsfnets, sahli2020physics,li2023meta}, the training of PINN is often  lengthy and challenging, partly because applying differential operators over NNs can complicate the loss landscape~\citep{krishnapriyan2021characterizing}. 
%As the most popular class of ML solvers, PINNs have achieved many success stories, such as~\citep{raissi2020hidden, chen2020physics, jin2021nsfnets, sirignano2018dgm, zhu2019physics, geneva2020modeling, sahli2020physics}. However, the training of PINN is often  lengthy and challenging, partly because applying differential operators over the NN can complicate the loss landscape~\citep{krishnapriyan2021characterizing}. 
%Recent works have analyzed common failure modes of PINNs which include modeling problems exhibiting high-frequency, multi-scale, chaotic, or turbulent behaviors~\citep{wang2020and, wang2020eigenvector, wang2020understanding, wang2022respecting}, or when the governing PDEs are stiff~\citep{krishnapriyan2021characterizing, mojgani2022lagrangian}. 
%One class of approaches to mitigate the training challenge is to set different weights for the boundary and residual loss terms. For example, \citet{wight2020solving} suggested to set a large weight for the boundary loss to prevent the dominance of the residual loss.
Alternatively, early works such as \citep{graepel2003solving,raissi2017physics} developed GP models to solve linear PDEs from noisy observations of the source terms. \cmt{More recently, a theoretical justification and analysis for using GPs as priors over PDE solutions is provided by \citet{wang2021bayesian}.} \citet{chen2021solving,long2022autoip} further extended this direction by developing a kernel method capable of addressing both linear and nonlinear PDEs.\cmt{Their solution approximation combines kernel functions and their derivatives --- more generally, the application of linear differential operators directly to the kernels. This approach requires embedding the operators into the kernel to construct the Gram/covariance matrix, whose size is typically greater than the number of collocation points due to the inclusion of various operator combinations.
The work of~\citet{long2022autoip} builds on the same core idea but casts it within a probabilistic framework, introducing a variational inference method to estimate the posterior distribution.}
 \citet{batlle2023error} developed a rigorous convergence framework, establishing both convergence guarantees and rates of~\citep{chen2021solving}. To mitigate the scalability issue for massive collocation sets, \citet{chen2023sparse} proposed a sparse approximation technique based on the sparse inverse Cholesky factorization \citep{schafer2021sparse}. % \citet{meng2023sparse} applied an adjusted Nystr\"om method \citep{jin2013improved} to construct a sparse approximation of the Gram matrix. %However, both approaches require careful handling of different sub-blocks within the Gram matrix --- each corresponding to a pair of linear differential operators applied to the kernels --- bringing in complexity and implementation inconvenience.

\citet{xutoward2025}  proposed using a standard kernel interpolation framework to approximate the PDE solution, and induced a Kronecker product structure to simply the computation. 
%where the covariance matrix is simply the kernel matrix evaluated at the collocation points. This greatly reduces the matrix size and avoids embedding differential operators into the kernel. When using product kernels, the resulting covariance matrix admits a Kronecker product structure, enabling large computational savings by avoiding full matrix construction and inversion.
The efficiency of Kronecker-structured GP models has been studied in prior works \citep{saatcci2012scalable, wilson2015kernel, wilson2015thoughts,izmailov2018scalable}. %n particular, \citet{wilson2015thoughts} noted that for evenly spaced grids, the resulting kernel matrices can exhibit Toeplitz structure, enabling $\mathcal{O}(n \log n)$ computational complexity.
 \citet{fang2023solving} leveraged a similar idea for solving high-frequency and multiscale PDEs, using a spectral mixture kernel along each input dimension to  capture dominant frequencies in the kernel function. Broader discussions on Bayesian approaches to PDEs are given in~\citep{owhadi2015bayesian}.

The idea of using separable function approximations for solving PDEs was introduced by \citet{beylkin2002numerical} and has seen rapid development in recent years. Notable works include \citet{richter2021solving}, \citet{oster2022approximating}, and \citet{fackeldey2022approximative}, which apply tensor decomposition techniques --- often with random sampling --- to solve PDEs and/or stochastic differential equations efficiently. A comprehensive review of deterministic methods in this context can be found in \citet{bachmayr2023low}.

%% file: expr.tex
\section{Numerical Experiments}
We evaluated our method on five commonly-used benchmark PDE families from the literature on machine learning PDE solvers~\citep{raissi2019physics,chen2021physics,xu2024efficientkernelbasedsolversnonlinear}: {viscous Burger's equations, nonlinear elliptic PDEs, Eikonal PDEs, Allen-Cahn equations, and nonlinear Darcy Flow equations}. The details are provided in Appendix Section~\ref{sect:pde-benchmark}. 
We denote our method using partial freezing as \ours-PF and Newton's method as \ours-NT. We compared against two state-of-the-art GP/kernel-based solvers: (1) DAKS (Derivative-Augmented Kernel Solver)~\citet{chen2021physics}, which augments the GP covariance matrix with derivative information to estimate solution values and their derivatives at the collocation points. (2) SKS (Simple Kernel-based Solver)~\citep{xutoward2025}, which employs standard GP interpolation. We also compared with (3) PINN~\citep{raissi2019physics}. SKS, DAKS and \ours were implemented with JAX~\citep{frostig2018compiling}, while PINN with PyTorch~\citep{paszke2019pytorch}. Hyperparameters and settings are detailed in Appendix Section~\ref{sect:method-details}. 
\subsection{Solution Accuracy}
%\vspace{-0.2in}
\noindent\textbf{Simpler Cases.} We first evaluated all methods on relatively simple benchmarks, where only a modest number of collocation points is required. Specifically, we considered Burgers' equation~\eqref{eq:burgers} with viscosity $\nu = 0.02$, the nonlinear elliptic PDE~\eqref{eq:Nonlinear Elliptic}, and the Eikonal PDE~\eqref{eq:ekonal} --- the same test cases as in~\citep{chen2021solving}. Following the setups in~\citep{chen2021solving,xutoward2025}, the number of collocation points was varied as
600, 1200, 2400, 4800 for Burgers' equation and
300, 600, 1200, 2400 for the nonlinear elliptic and Eikonal PDEs. \yifan employed randomly sampled collocation points, whereas SKS used a regularly spaced square grid with a comparable number of points. Notably, \yifan performed worse under gridded collocation, as also observed in~\citep{xutoward2025}. We ran PINN and our method on both randomly sampled points and regularly spaced grids, and report the best outcomes. For methods using random sampling, each experiment was repeated five times and the average  relative $L^2$ error was recorded. 
As shown in Table~\ref{tb:easy-small}, our method --- both \ours-PF and \ours-NT --- consistently achieves the \textit{highest solution accuracy}. The only exception is when solving Burgers' equation ($\nu=0.02$) with 600 collocation points, where \ours-NT performs slightly worse than PINN but still ranks third.\cmt{, with an error magnitude comparable to that of PINN.}  Note that for all these PDEs the input dimension is $d=2$, under which the TR and CP decomposition forms are equivalent (see Section~\ref{sect:model}). Accordingly, we do not introduce additional notation to distinguish between them.
\begin{table}
	\small
	\centering
    \caption {\small Relative $L^2$ error of solving \textit{simpler} PDEs, with a small number of collocation points. Inside the parenthesis of each top row indicates the grid used by \sks, which takes approximately the same number of collocation points used by other methods. The top two results are shown in bold face. }\label{tb:easy-small}
 \begin{subtable}{0.5\textwidth}
 	\centering
 	\caption{\small Burgers' equation \eqref{eq:burgers} with viscosity $\nu=0.02$ } \label{tb:burgers-simple}
 	\resizebox{\linewidth}{!}{
 		\begin{tabular}[c]{ccccc}
 			\toprule
 			\multirow{2}{*}{\textit{Method}} & 600 \cmt{($25 \times 25$)} & 1200 \cmt{($35 \times 35$)} & 2400 \cmt{($49 \times 49$)} & 4800 \cmt{($70 \times 70$)}\\
 			& $(25 \times 25)$ & $(35 \times 35)$ & $(49 \times 49)$ & $(70 \times 70)$\\
 			\hline
 			\yifan & 3.05E-02 & 1.38E-02 & 1.51E-03 & 1.70E-04 \\
 			%\yifan-grid & 4.24E-01& 3.93E-01& 2.96E-02& 5.67E-02\\
 			PINN & \textbf{4.67E-03} & 1.17E-03 & {6.27E-04} &  6.50E-04 \\
 			SKS & 2.51E-02 & 9.41E-03 & 1.36E-03 & 5.59E-04 \\
 			\ours-PF & \textbf{3.56E-03} & \textbf{5.77E-04} & \textbf{1.35E-04} & \textbf{8.83E-05}\\
 			\ours-NT & {8.42E-03} & \textbf{8.50E-04} & \textbf{1.46E-04} & \textbf{5.71E-05}\\
 			\bottomrule
 		\end{tabular}
 	}
 \end{subtable}
 \begin{subtable}{0.5\textwidth}
 	\centering
 	\caption{\small Nonlinear elliptic PDE \eqref{eq:Nonlinear Elliptic}  } \label{tb:elliptic-simple}
 	\resizebox{\linewidth}{!}{
 		\begin{tabular}[c]{ccccc}
 			\toprule
 			\multirow{2}{*}{\textit{Method}} & 300 \cmt{($18 \times 18$)} & 600 \cmt{($25 \times 25$)} & 1200 \cmt{($35 \times 35$)} & 2400 \cmt{($49 \times 49$)}\\
 			%\cline{2-5}
 			& ($18 \times 18$) & ($25 \times 25$) & ($35 \times 35$) & ($49 \times 49$)\\
 			\hline
 			\yifan & 1.03E-01 & 1.03E-04 & 7.78E-04 & 1.51E-07\\
 			%\yifan-grid & \textbf{2.24E-05}& \textbf{7.28E-07} & \textbf{1.36E-07} & \textbf{6.63E-08}\\
 			PINN & 3.05E-01 & 1.73E-02 & 1.15E-03 & 2.88E-04\\
 			SKS & 1.13E-02 & 6.23E-05 & 6.11E-06 & 1.65E-06\\
 			\ours-PF & \textbf{1.97E-06} & \textbf{2.82E-07} & \textbf{1.28E-07} & \textbf{4.04E-08}\\
 			\ours-NT & \textbf{1.78E-06} & \textbf{3.52E-07} & \textbf{1.74E-07} & \textbf{4.06E-08}\\
 			\bottomrule
 		\end{tabular}
 	}
 \end{subtable}
 \begin{subtable}{0.5\textwidth}
 	\caption{\small Eikonal PDE~\eqref{eq:ekonal}  } \label{tb:eikonal-simple}
 	\centering
 	\resizebox{\linewidth}{!}{
 		\begin{tabular}[c]{ccccc}
 			\toprule
 			\textit{Method} & 300 \cmt{($18 \times 18$)} & 600 \cmt{($25 \times 25$)} & 1200 \cmt{($35 \times 35$)} & 2400 \cmt{($49 \times 49$)}\\
 			%\cline{2-5}
 			\hline
 			\yifan & 5.65E-01 & 9.18E-02 & 1.27E-03 & 4.35E-04 \\
 			%\yifan-grid & \textbf{2.99E-04} & 3.31E-04& 2.86E-04& 2.86E-04\\
 			PINN & 1.65E-01 & 7.05E-02 & 2.54E-02 & 1.96E-02 \\
 			SKS & 3.49E-03 & 1.50E-03 & 1.07E-03 & 1.40E-04\\
 			\ours-PF & \textbf{2.20E-04} & \textbf{1.56E-04} & \textbf{6.99E-05} & \textbf{9.98E-05}\\
 			\ours-NT & \textbf{3.60E-04} & \textbf{9.06E-05} & \textbf{5.95E-05} & \textbf{7.23E-05}\\
 			\bottomrule
 		\end{tabular}
 	}
 \end{subtable}
 %\vspace{-0.1in}
\end{table}
\begin{figure*}[h]
	\centering
	\setlength{\tabcolsep}{0pt}
	\begin{tabular}{ccc}
		\begin{subfigure}[b]{0.3\linewidth}
			\includegraphics[width=\linewidth]{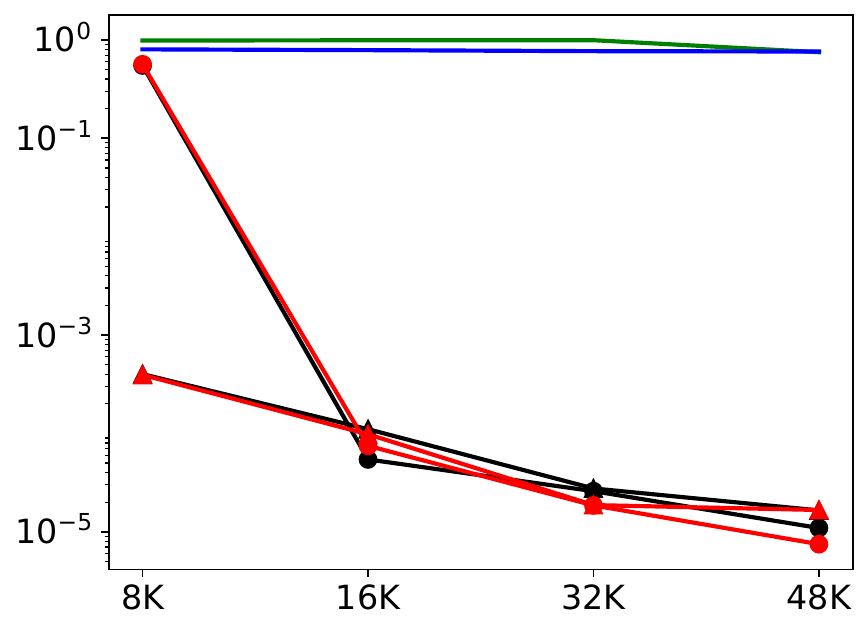}
			\caption{\small 4D Allen-Cahn ($a=15$)}
			\label{fig:high-dim-4d}
		\end{subfigure} &
		\begin{subfigure}[b]{0.3\linewidth}
			\includegraphics[width=\linewidth]{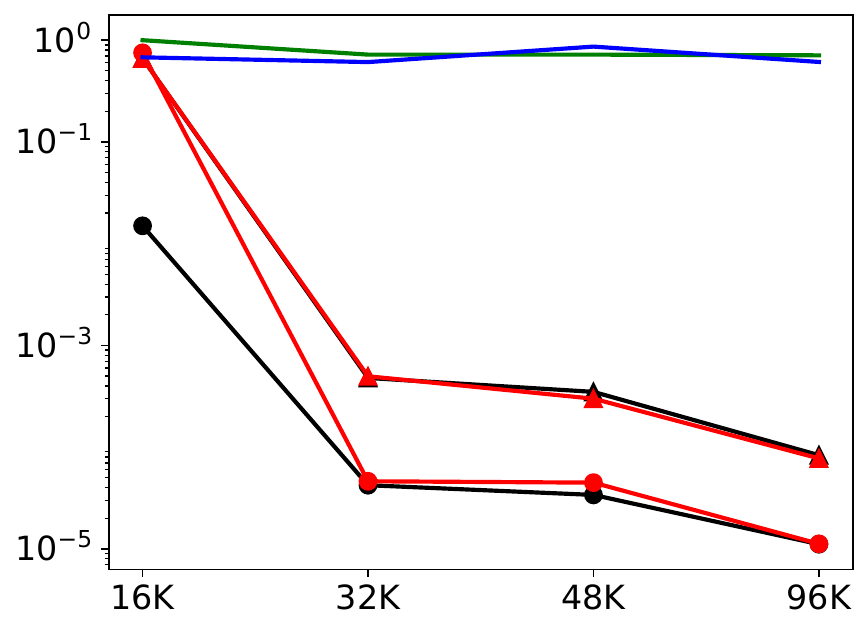}
			\caption{\small 6D Allen-Cahn ($a=15$)}
			\label{fig:high-dim-6dAllen-Cahn}
		\end{subfigure} &
		\begin{subfigure}[b]{0.3\linewidth}
			\includegraphics[width=\linewidth]{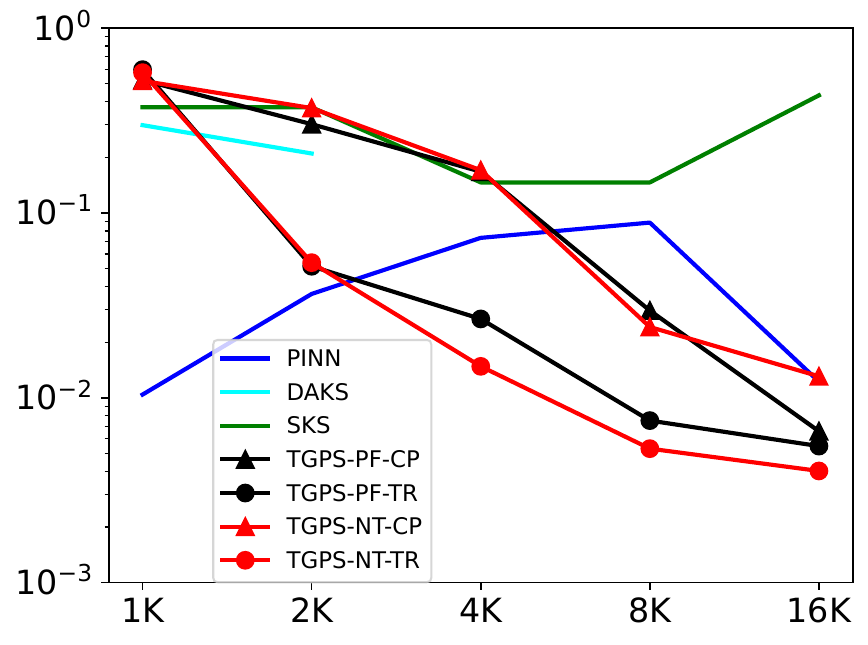}
			\caption{\small 6D Nonlinear Darcy Flow}
			\label{fig:high-dim-6ddarcy}
		\end{subfigure} 
	\end{tabular}
%\vspace{-0.1in}
	\caption{\small Relative $L^2$ error \textit{vs.} the number of collocation points in solving higher-dimensional PDEs. }\label{fig:high-dim-res}
	%\vspace{-0.1in}
\end{figure*}

\noindent\textbf{Difficult Cases.} Next, we evaluated the methods on more challenging problems: Burgers' equation with $\nu = 0.001$, the 2D Allen-Cahn equation with $a = 15$ and $a = 20$, higher-dimensional Allen-Cahn equations with $a = 15$, $d = 4$ and $d = 6$, and the 6D nonlinear Darcy flow. 

To assess the necessity of massive collocation, we first used the same scale as in the simpler PDEs and tested on these PDEs with input dimension $d=2$. As shown in Appendix Table~\ref{tb:hard-small}, the performance of all the methods deteriorates noticeably while \ours  still consistently outperforms the competing methods. Notably, with $2400$ and $4800$ collocation points, \ours achieves relative $L^2$ errors of $10^{-5}$ and $10^{-6}$, respectively, on the 2D Allen-Cahn equation with $a=15$, and $10^{-3}$ and $10^{-5}$ with $a=20$. %These results confirm that substantially larger collocation sets are indeed necessary to attain satisfactory accuracy.

 %Next, we evaluated the methods on more challenging problems that require substantially larger numbers of collocation points. These included Burgers' equation with $\nu = 0.001$, the 2D Allen-Cahn equation with $a = 15$ and $a = 20$, higher-dimensional Allen-Cahn equations with $a = 15$, $d = 4$ and $d = 6$, and the 6D nonlinear Darcy flow.

%To assess the necessity of massive collocation, we first used the same scale as in the simpler PDEs (hundreds to thousands of points) and tested on these PDEs with input dimension $d=2$. As shown in Appendix Table~\ref{tb:hard-small}, the performance of all the methods deteriorates noticeably. Both  \yifan and PINN consistently produce large errors on the order of $10^{-1}$, indicating failure. \sks typically yields errors of the same order, but improves by one to two orders of magnitudes at the biggest number of  collocation points. \ours exhibits a similar trend while still consistently outperforming the competing methods. Notably, with $2400$ and $4800$ collocation points, \ours achieves relative $L^2$ errors of $10^{-5}$ and $10^{-6}$, respectively, on the 2D Allen-Cahn equation with $a=15$, and $10^{-3}$ and $10^{-5}$ with $a=20$. These results confirm that substantially larger collocation sets are indeed necessary to attain satisfactory accuracy.

We then increased the number of collocation points to the range of 6.4K -- 40K. In this regime, \yifan became prohibitively expensive, so no results are reported. As shown in Table~\ref{tb:hard-large}, \ours substantially reduce errors, achieving $10^{-4}$ for Burgers’ equation with $28$K–$32$K collocation points, and $10^{-6}$ for the 2D Allen-Cahn equation (both $a=15$ and $a=20$) across all the cases. While \sks also improves significantly, its accuracy is consistently worse than \ours (except for Burgers’ equation with $8$K collocation points, where \sks slightly outperforms \ours-PF), often by one order of magnitude. PINN, by contrast, only reaches a relative $L^2$ error of $10^{-2}$ on Burgers’ equation and remains highly inaccurate on Allen-Cahn (relative $L^2$ error exceeding one), with little benefit from additional collocation points. This poor performance is likely due to the relatively high-frequency components in the Allen-Cahn solution (see \eqref{eq:allen-cahn}), which neural networks struggle to capture because of their known spectral bias~\citep{rahaman2019spectral}.

We next evaluated our method on higher-dimensional PDEs. As shown in Figure~\ref{fig:high-dim-res}, with only a few thousand collocation points, all methods exhibit large relative $L^2$ errors, indicating that these collocation points are insufficient. As the collocation number increases, the performance of \ours  improves substantially (\eg achieving $7.5\times10^{-6}$ and $1.1\times 10^{-5}$ errors in the 4D and 6D Allen-Cahn equations). \cmt{For example, in the 4D and 6D Allen-Cahn equations, \ours-NT with TR decomposition achieves relative $L^2$ errors of $7.46\times10^{-6}$ and $1.12\times10^{-5}$,  when using 48K and 96K collocation points, respectively.} By contrast, \sks and PINN show little improvement.\cmt{, while \sks yields errors so large that we excluded it from the figure.} This is likely because they require far more collocation points to realize significant gains. For instance, when applying \sks to the 6D Allen-Cahn equation, even 96K collocation points correspond to a $7\times7\times7\times7\times7\times7$ grid --- far too coarse to achieve meaningful accuracy. However, increasing the grid to a reasonably dense level --- say, 100 points per dimension --- causes the number of model parameters in \sks to explode. PINNs may still suffer from spectral bias in solving higher-dimensional Allen-Cahn equations.  Moreover, \yifan can only accommodate a few thousand collocation points (see Figure~\ref{fig:high-dim-6ddarcy}), which is inadequate for higher-dimensional problems. These results not only show the superiority of \ours in solution accuracy,  but also highlight its ability to leverage collocation points more efficiently than competing methods.

%We next evaluated our method on higher-dimensional PDEs, varying the number of collocation points from 1K to 96K. As shown in Figure~\ref{fig:high-dim-res}, with only a few thousand collocation points, all methods exhibit large relative $L^2$ errors, indicating that these collocation points are insufficient. As the number of collocation points increases, the performance of \ours --- whether with CP or TR decomposition --- improves substantially. For example, in the 2D and 4D Allen–Cahn equations, \ours-NT with TR decomposition achieves relative $L^2$ errors of $7.46\times10^{-6}$ and $1.12\times10^{-5}$,  when using 48K and 96K collocation points, respectively. By contrast, \sks and PINN show little improvement. This is likely because they require far larger numbers of collocation points to realize significant gains. For instance, when applying \sks to the 6D Allen-Cahn equation, even 96K collocation points correspond to only a $7\times7\times7\times7\times7\times7$ grid --- far too coarse to achieve meaningful accuracy. Moreover, \yifan can only accommodate a few thousand collocation points (see Figure~\ref{fig:high-dim-6ddarcy}), which is inadequate for higher-dimensional problems. These results not only demonstrate the superiority of \ours in solution accuracy,  but also highlight its ability to leverage collocation points more efficiently than competing methods.

%\zhec{Together it shows that}
\begin{table}[t]
	\caption {\small Relative $L^2$ error of solving \textit{more difficult} PDEs.}\label{tb:hard-large}
	%\vspace{-0.25in}
	\small
	\centering
	\begin{subtable}{0.5\textwidth}
		\caption{\small The Burgers' equation \eqref{eq:burgers} with viscosity $\nu=0.001$. } \label{tb:burgers-nu1}
		\centering
		\resizebox{\linewidth}{!}{
			\begin{tabular}[c]{ccccc}
				\toprule
				\multirow{2}{*}{\textit{Method}} & 8000 \cmt{($200 \times 40$)} & 16000 \cmt{($400 \times 40$)} & 28000 \cmt{($700 \times 40$)} & 32000 \cmt{($800 \times 40$)}\\
				& (200$\times$40) & (400$\times$40) & (700$\times$40) & (800$\times$40)\\
				\hline
				PINN & 5.58E-01 & 2.65E-01 & 2.30E-02 & 1.07E-02\\
				SKS & \textbf{2.65E-02} & {9.41E-03} & 4.20E-03 & 3.88E-03\\
				\ours-PF & {3.41E-02} & \textbf{5.32E-03} & \textbf{4.75E-04} & \textbf{5.05E-04}\\
				\ours-NT & \textbf{3.30E-02} & \textbf{4.85E-03} & \textbf{5.44E-04} & \textbf{6.52E-04}\\
				\bottomrule
			\end{tabular}
		}
	\end{subtable}
	
	\begin{subtable}{0.5\textwidth}
		\caption{\small The 2D Allen-Cahn equation \eqref{eq:allen-cahn} with $a=15$ } \label{tb:allen-cahn-15}
		\centering
		\resizebox{\linewidth}{!}{
			\begin{tabular}[c]{ccccc}
				\toprule
				\multirow{2}{*}{\textit{Method}} & 6400 \cmt{($80 \times 80$)} & 8100 \cmt{($90 \times 90$)} & 22500 \cmt{($150 \times 150$)} & 40000 \cmt{($200 \times 200$)}\\
				%\cline{2-5}
				& (80$\times$80) & (90$\times$90) & (150$\times$150) & (200$\times$200)\\
				\hline
				PINN & 7.11E0 & 7.50E0 & 5.95E0 & 8.29E0\\
				SKS & 1.17E-04 & 4.82E-05 & {6.14E-06} & {6.28E-06} \\
				\ours-PF & \textbf{6.00E-06} & \textbf{1.21E-06} & \textbf{4.87E-06} & \textbf{1.43E-06}\\
				\ours-NT & \textbf{3.99E-06} & \textbf{1.28E-06} & \textbf{1.70E-06} & \textbf{1.76E-06}\\
				\bottomrule
			\end{tabular}
		}
	\end{subtable}
	\begin{subtable}{0.5\textwidth}
		\caption{\small The 2D Allen-Cahn equation \eqref{eq:allen-cahn} with $a=20$ } \label{tb:allen-cahn-20}
		\centering
		\resizebox{\linewidth}{!}{
			\begin{tabular}[c]{ccccc}
				\toprule
				\textit{Method} & 6400 \cmt{($80 \times 80$)} & 8100 \cmt{($90 \times 90$)} & 22500 \cmt{($150 \times 150$)} & 40000 \cmt{($200 \times 200$)}\\
				%\cline{2-5}
				\hline
				PINN & 5.91E0 & 6.29E0 & 8.29E0 & 8.39E0\\
				SKS & 5.63E-04 & 2.57E-04 & 5.66E-05 & 4.21E-05 \\
				\ours-PF & \textbf{8.50E-06} & \textbf{8.47E-06} & \textbf{5.90E-06} & \textbf{5.14E-06}\\
				\ours-NT & \textbf{9.03E-06} & \textbf{7.42E-06} & \textbf{5.79E-06} & \textbf{4.80E-06}\\
				\bottomrule
			\end{tabular}
		}
	\end{subtable}
%\vspace{-0.05in}
\end{table}

\begin{figure*}
	\centering
	\setlength{\tabcolsep}{0pt}
	\begin{tabular}{cccc}
		\begin{subfigure}[b]{0.25\linewidth}
			\includegraphics[width=\linewidth]{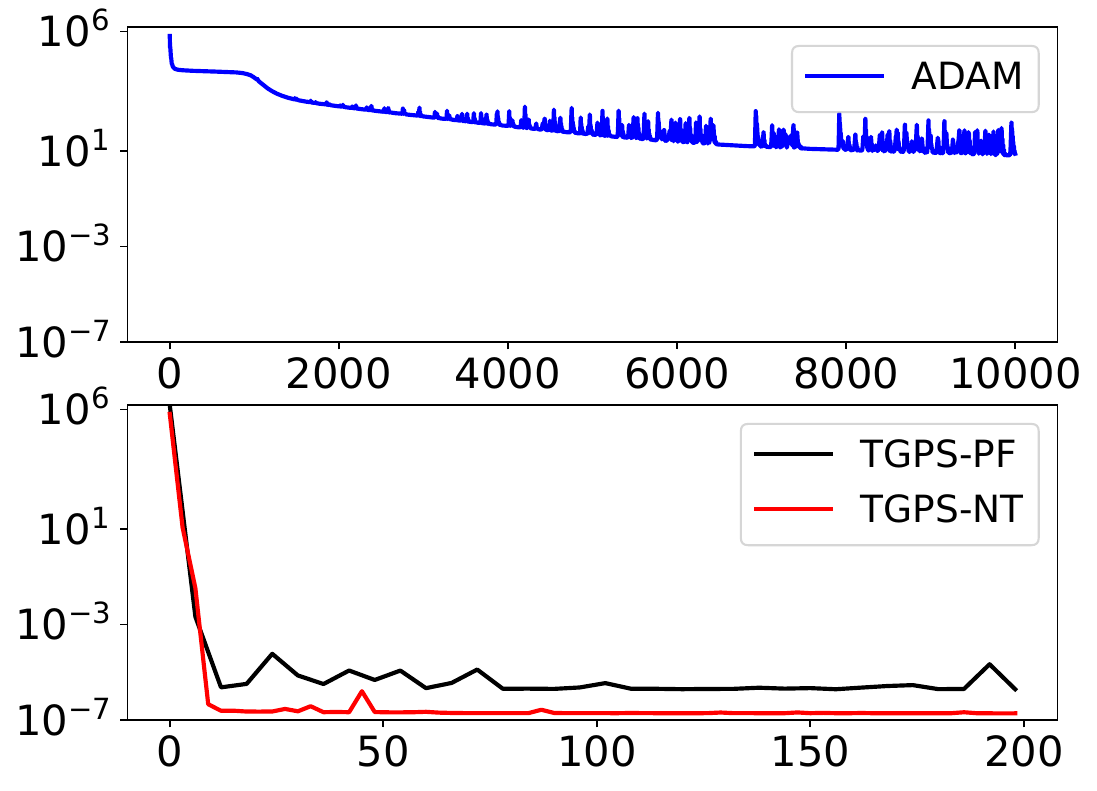}
			\caption{\small Nonlinear Elliptic}
			\label{fig:tr-curve-nonlinear-elliptic}
		\end{subfigure} &
		\begin{subfigure}[b]{0.25\linewidth}
			\includegraphics[width=\linewidth]{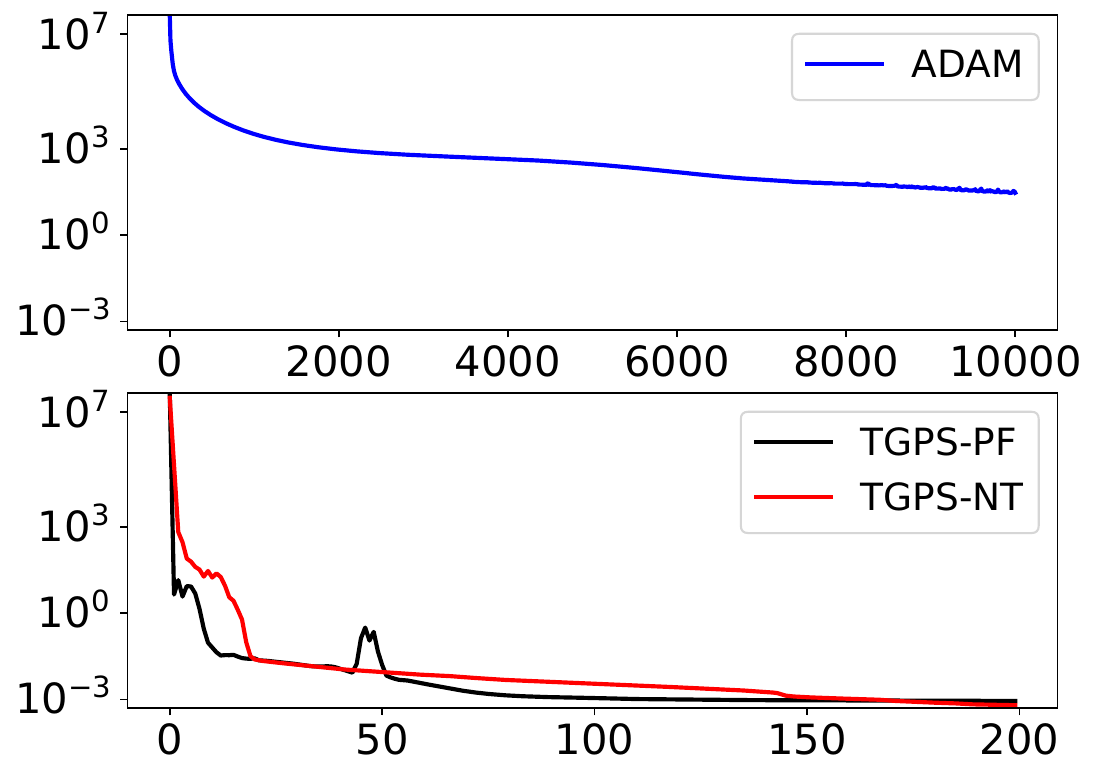}
			\caption{\small Burgers ($\nu=0.001$)}
			\label{fig:tr-curve-burgers}
		\end{subfigure} &
		\begin{subfigure}[b]{0.25\linewidth}
			\includegraphics[width=\linewidth]{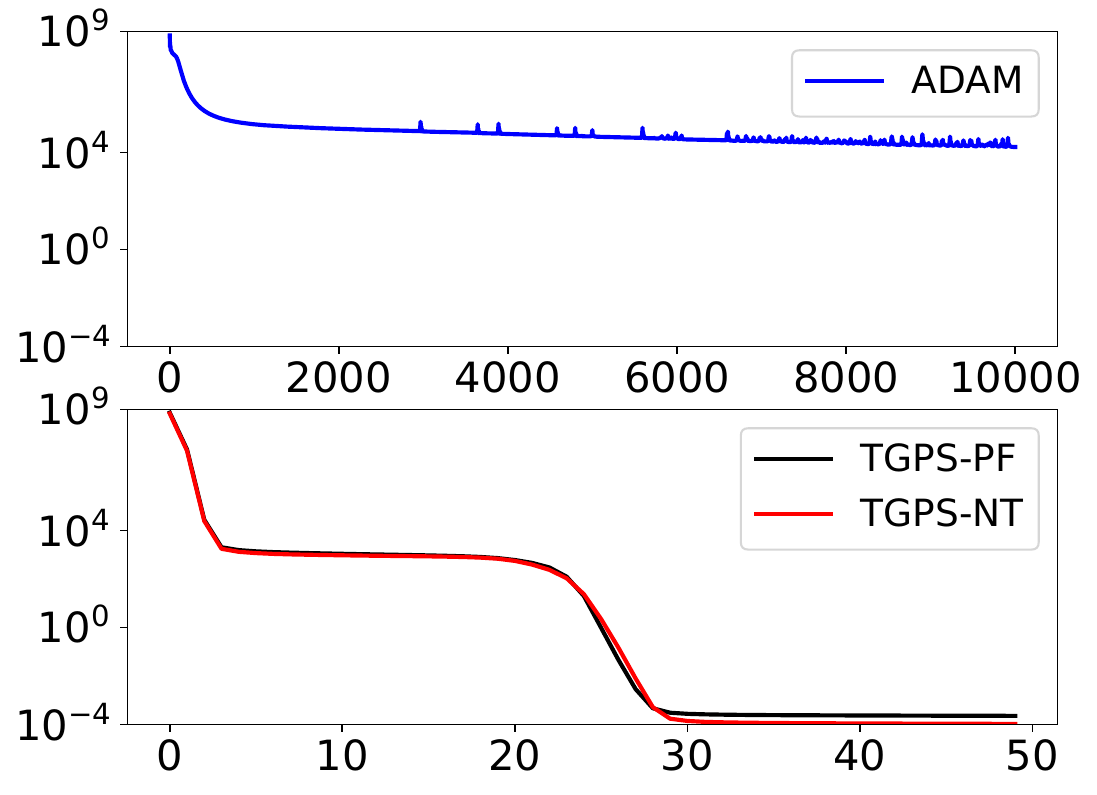}
			\caption{\small 2D Allen-Cahn ($a=15$)}
			\label{fig:tr-curve-allen-cahn-2d}
		\end{subfigure} &
		\begin{subfigure}[b]{0.25\linewidth}
			\includegraphics[width=\linewidth]{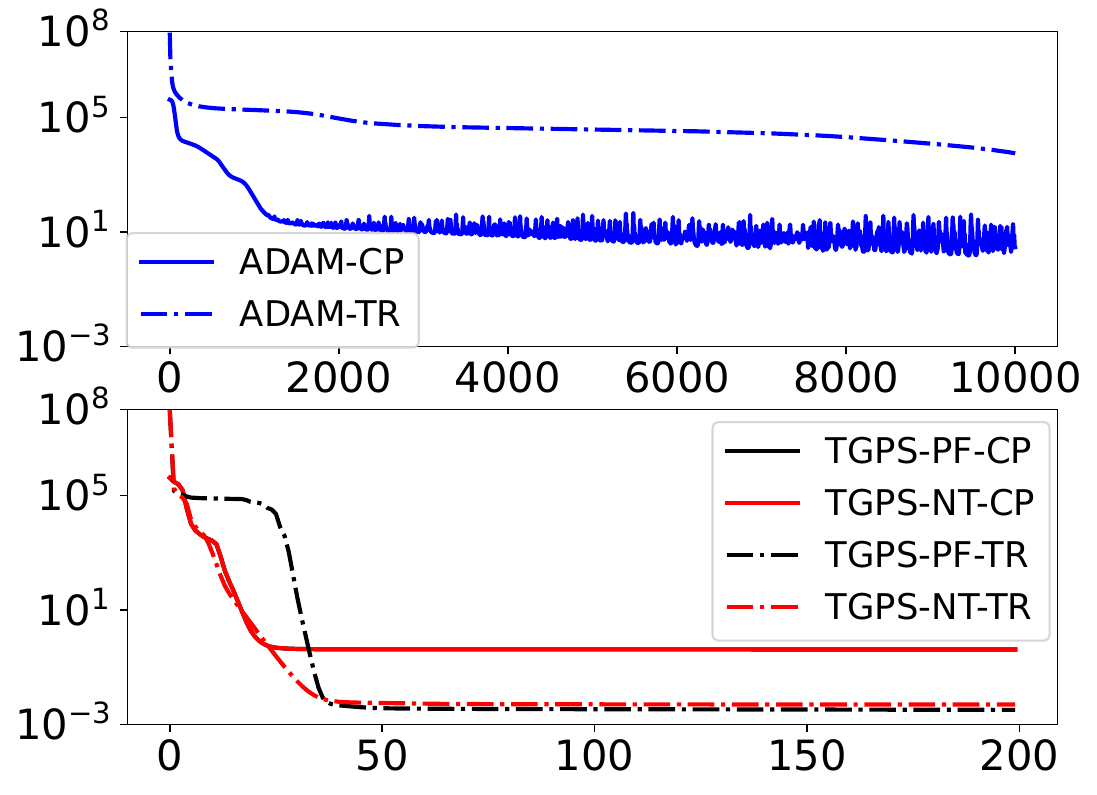}
			\caption{\small 6D Allen-Cahn ($a=15$)}
			\label{fig:tr-curve-allen-cahn-6d}
		\end{subfigure} 	
	\end{tabular}
%\vspace{-0.1in}
	\caption{\small Training curves: Training loss \textit{vs.} Number of iterations. }\label{fig:train-curve}
%	\vspace{-0.1in}
\end{figure*}
\begin{figure*}
	\centering
	\setlength{\tabcolsep}{0pt}
	\begin{tabular}{cccc}
		\begin{subfigure}[b]{0.25\linewidth}
			\includegraphics[width=\linewidth]{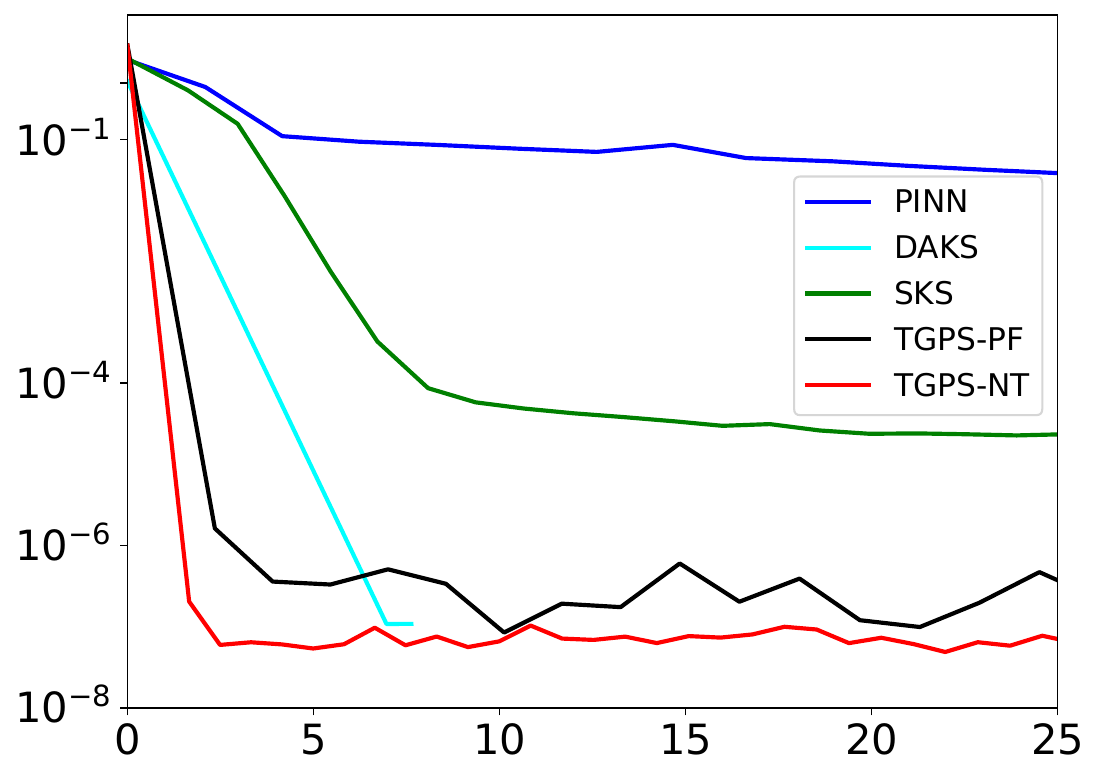}
			\caption{\small Nonlinear Elliptic}
			\label{fig:tr-curve-nonlinear-elliptic}
		\end{subfigure} &
		\begin{subfigure}[b]{0.25\linewidth}
			\includegraphics[width=\linewidth]{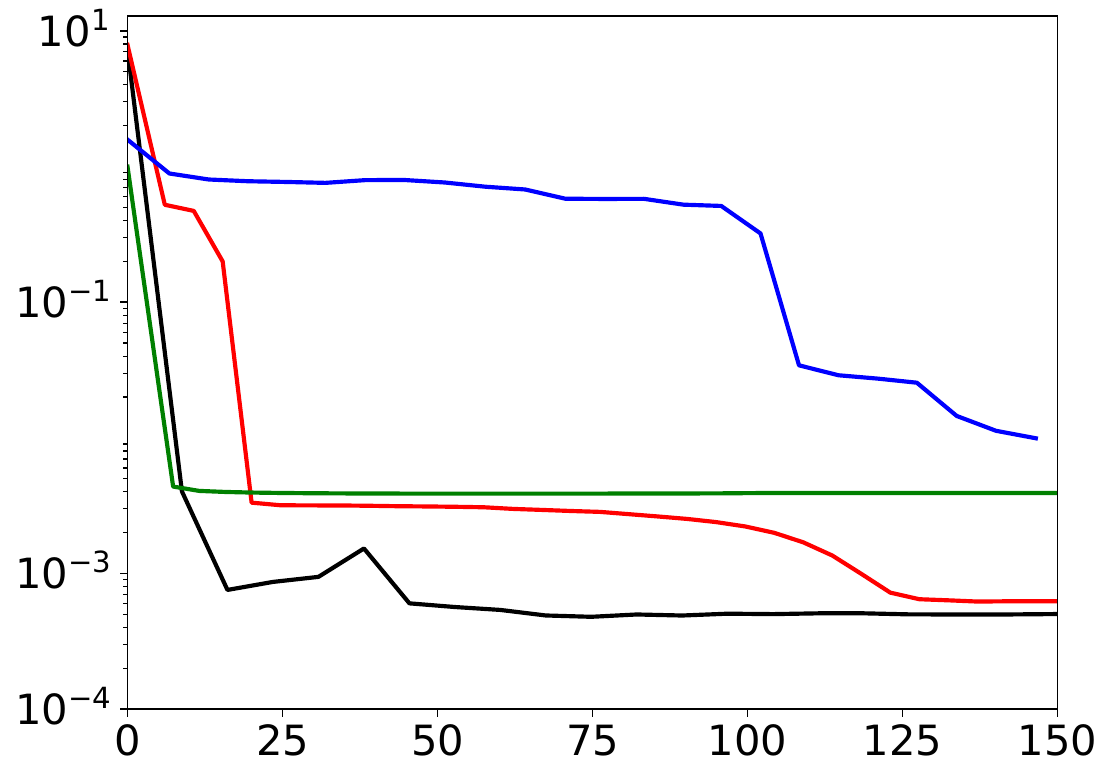}
			\caption{\small Burgers ($\nu=0.001$)}
			\label{fig:tr-curve-burgers}
		\end{subfigure} &
		\begin{subfigure}[b]{0.25\linewidth}
			\includegraphics[width=\linewidth]{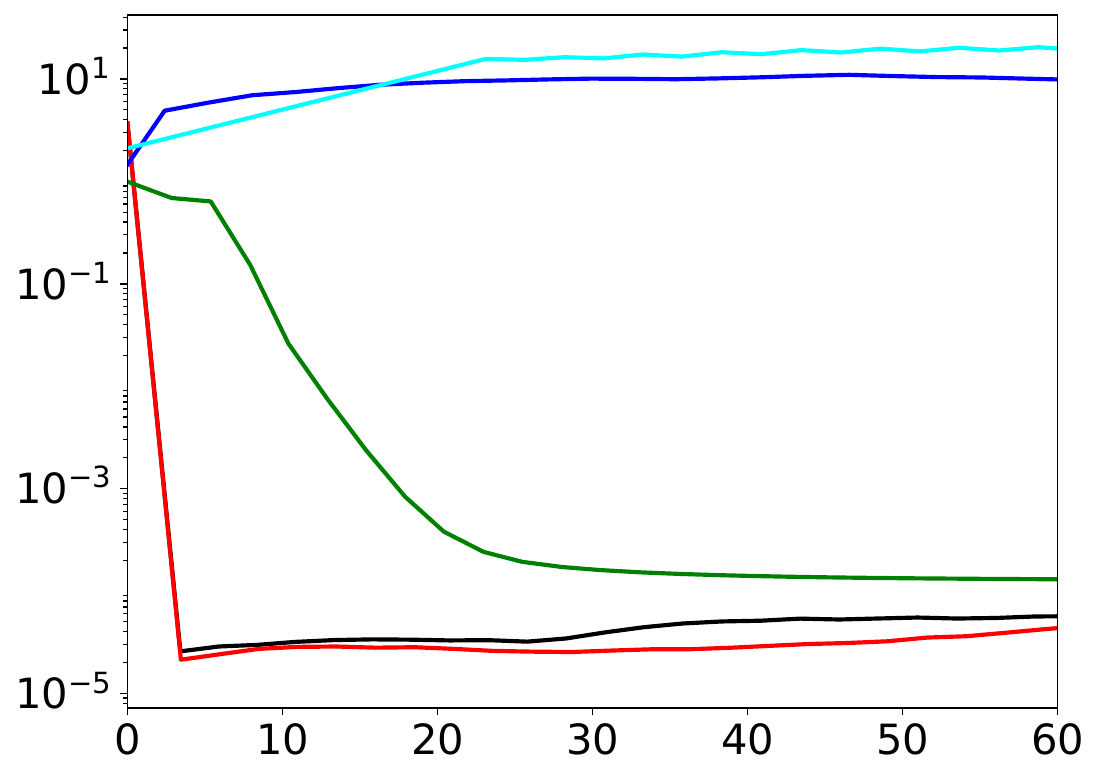}
			\caption{\small 2D Allen-Cahn ($a=15$)}
			\label{fig:tr-curve-allen-cahn-2d}
		\end{subfigure} &
		\begin{subfigure}[b]{0.25\linewidth}
			\includegraphics[width=\linewidth]{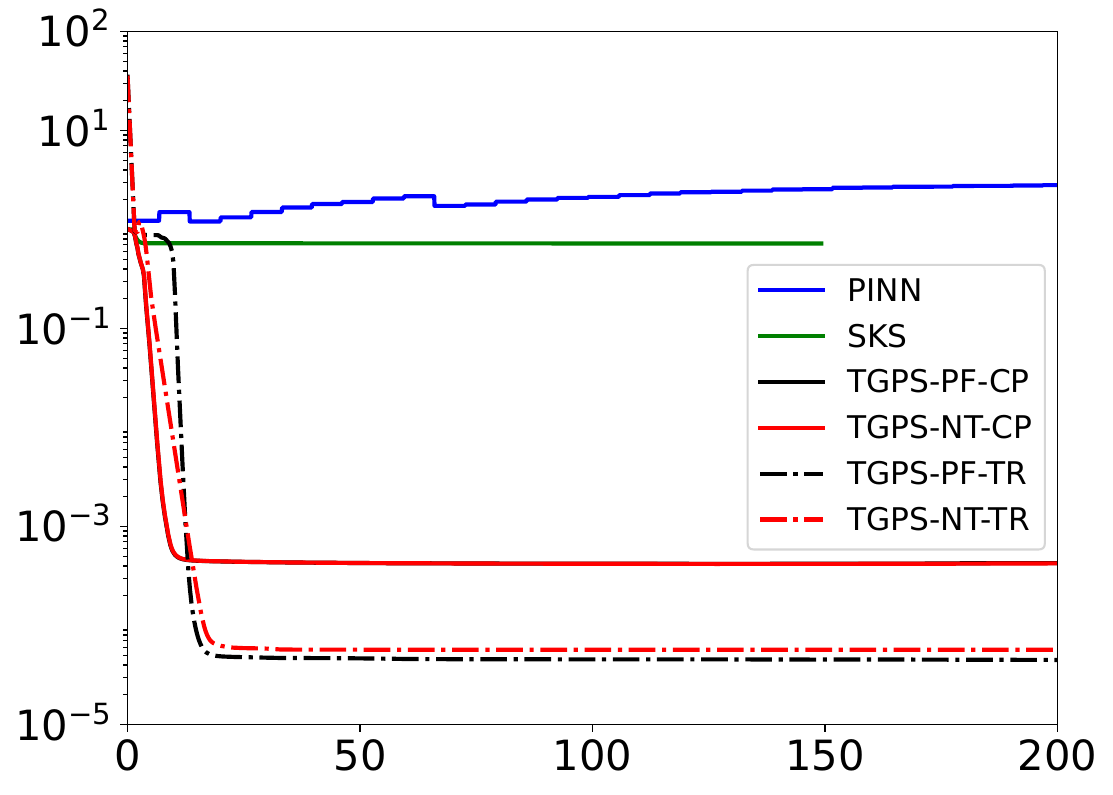}
			\caption{\small 6D Allen-Cahn ($a=15$)}
			\label{fig:tr-curve-allen-cahn-6d}
		\end{subfigure} 	
	\end{tabular}
%\vspace{-0.05in}
	\caption{\small Running time (in seconds) \textit{vs.} Relative $L^2$ error. }\label{fig:test-curve}
%	\vspace{-0.1in}
\end{figure*}

\noindent\textbf{Comparison with Conventional Numerical Methods.} We compared against two established numerical methods. The first is the P2 Galerkin Finite Element Method (FEM)~\citep{barrett1993finite,brenner2008mathematical}, and the second is a robust finite difference (FD) scheme. Details are provided in Appendix Section~\ref{sect:numerical-sover}.
%In addition to ML-based solvers, we compared \ours against two powerful numerical methods. The first is the P2 Galerkin Finite Element Method (FEM)~\citep{barrett1993finite,brenner2008mathematical}, and the second is the Chebyshev-Gauss-Lobatto (CGL) collocation method~\citep{boyd2001chebyshev} --- a global spectral approach that constructs the approximation space as a tensor product of 1D Chebyshev polynomial bases. We also included a robust finite difference (FD) scheme based on centered second-order finite differences combined with a Newton-Krylov solver for nonlinear systems. Further details of these methods are provided in Appendix Section~\ref{sect:numerical-sover}.
We conducted experiments on the nonlinear elliptic PDE~\eqref{eq:Nonlinear Elliptic} and the 2D Allen–Cahn equation~\eqref{eq:allen-cahn} with $a=15$ and $a=20$, where ground-truth solutions are available for fair comparison. For FEM, we  set the mesh spacing to match the collocation grids used in \sks. The FD scheme used the same grids as \sks. As reported in Appendix Table~\ref{tb:comparison-finite-diff},\cmt{Note that the relative $L^2$ errors of \yifan, PINN, and SKS have already been presented in Tables~\ref{tb:elliptic-simple}, \ref{tb:allen-cahn-15}, and \ref{tb:allen-cahn-20}.} our method (\ours) consistently outperforms the conventional numerical solvers, often by several orders of magnitude in error. %Notably, as the grid is refined, the memory cost of CGL grows rapidly, rendering it impractical for solving 2D Allen–Cahn equations on a $200 \times 200$ grid. In higher-dimensional PDEs, these numerical methods are restricted to extremely coarse grids (or large mesh spacing), leading to large relative $L^2$ errors and, effectively, failure. 

\noindent\textbf{Point-Wise Error.} For a more fine-grained comparison, we present the pointwise error of each method when solving Burgers' equation ($\nu=0.001$) and 2D Allen-Cahn equations. Details are provided in Appendix Section~\ref{sect:point-wise}.

\noindent\textbf{Irregular Domains.} We further tested our approach by solving the nonlinear elliptic PDE and 2D Allen-Cahn equation on two irregular domains: one is circular and the other triangular.  In all the cases, \ours attained errors at the same level as those on regular domains, thereby confirming the advantage of \ours as a mesh-free method. Detailed results and discussion are provided in Appendix Section~\ref{sect:irregular}.

%We conducted additional tests, solving the nonlinear elliptic PDE~\eqref{eq:elliptic} on a circular domain and the Allen-Cahn equation~\eqref{eq:allen-cahn} on a triangular domain. In both cases, our method achieved the accuracy on the same level as regular domains, confirming the advantage of \ours as a mesh-free method. Detailed results and discussions are provided in Appendix~\ref{sect:irregular}.
%While our efficient computation is performed on grids, our method can be  applied to irregular-shaped domains by introducing a {(minimal)} virtual grid that encompasses such domains. This allows \ours to be used without any modifications. To validate the effectiveness of this strategy, we conducted additional tests, solving the nonlinear elliptic PDE~\eqref{eq:elliptic} on a circular domain and the Allen-Cahn equation~\eqref{eq:allen-cahn} on a triangular domain. In both cases, our method achieved reasonably good accuracy. Detailed results and discussions are provided in Appendix Section~\ref{sect:irregular}.

\subsection{Running Efficiency}
We next evaluated the computational efficiency of our method. Specifically, we tested four PDEs of increasing difficulty --- nonlinear elliptic, Burgers' equation with $\nu=0.001$, 2D Allen-Cahn ($a=15$), and 6D Allen-Cahn ($a=15$) --- using 2.5K, 28K, 4.8K, and 48K collocation points, respectively. All experiments were conducted on a Linux workstation equipped with an NVIDIA A100 GPU.

Our first objective was to assess whether our ALS training is more efficient than the widely used stochastic gradient descent (SGD) methods. To this end, we trained our model with ADAM using exactly the same training loss and  initial learning rate $10^{-2}$. For a fair comparison, both the ADAM and ALS  started from identical model initialization and hyperparameters. Moreover, no mini-batch sampling was applied: all collocation points were used to compute the full gradient of the training loss at each step, with ADAM adjusting the gradient through momentum and per-element step sizes. We then compared the learning curves of ADAM and ALS. As shown in Figure~\ref{fig:train-curve}, ALS training converges hundreds to thousands of times faster than ADAM in all the cases. With our ALS, the training loss saturates after 25, 150, 30, and 50 iterations for nonlinear elliptic, Burgers', 2D Allen-Cahn, and 6D Allen-Cahn, respectively. In contrast, after 10K iterations, ADAM still yields training losses several orders of magnitude larger. These results confirm that our ALS training, with its closed-form updates at each iteration, dramatically improves efficiency compared to standard SGD-based training.

Our second objective was to evaluate efficiency relative to competing approaches. For this purpose, we examined how the relative $L^2$ error of each method evolves with training time. The results, shown in Figure~\ref{fig:test-curve}, indicate that within the same runtime, \ours almost always achieves the smallest solution error, underscoring its efficiency advantage. Note that \yifan applies only to nonlinear elliptic and 2D Allen-Cahn due to  the use of smaller numbers of collocation points. When solving the nonlinear elliptic equation, although \yifan employs a Gauss–Newton approach that converges quickly, each iteration is significantly more expensive (due to computing inverse Hessian approximations). Consequently, \ours still requires roughly half the runtime to achieve comparable or better accuracy. For Burgers’ equation, \sks converges faster --- likely due to its efficient Kronecker product algebra --- but ultimately saturates at a relative $L^2$ error much larger than both \ours-PF and \ours-NT. Across all cases, PINN exhibits both slower convergence and substantially larger errors.

Overall, the results demonstrate that our method not only achieves superior solution accuracy compared to competing ML solvers, but also requires much less runtime to do so.

%% file: conclusion.tex
\section{Conclusion}
\vspace{-0.05in}
We have introduced \ours, a new machine learning solver for nonlinear PDEs. By combining one-dimensional Gaussian processes with tensor decomposition, our method alleviates the computational bottleneck of covariance matrices, controls the number of model parameters, and scales efficiently to massive collocation sets and higher-dimensional PDEs. The alternating least-squares (ALS) updates, coupled with partial freezing and Newton's method, yield substantial efficiency gains over standard stochastic gradient descent training. We also established theoretical guarantees on the expressivity of the model, as well as its accuracy and convergence as the number of collocation points increases. Experimental results on a variety of benchmark PDEs demonstrate not only high solution accuracy but also significant improvements in runtime efficiency.
%We have presented \ours, a new machine learning solver for nonlinear PDEs. By combining one-dimensional GPs with tensor decomposition, our method effectively addresses the computational bottleneck associated with covariance matrices, controls the number of model parameters, and scales efficiently to massive collocation and  higher-dimensional PDEs. Our alternating least square (ALS) updates coupled with partial freezing and Newton's method have shown great efficiency gains upon standard stochastic gradient descent training. We have also provided theoretical guarantees about the expressivity of our model, as well as the accuracy and convergence of our model along with the increase of collocation points. The performance on a variety of benchmark PDEs is encouraging in not only the solution accuracy but also in running efficiency. 

%In the future work, we will investigate our method in solving more PDEs.  

\section*{Acknowledgments}
HO and SZ acknowledge support from the Air Force Office of Scientific Research under MURI award number FA9550-20-1-0358 (Machine Learning and Physics-Based Modeling and Simulation). HO acknowledges support from the Air Force Office of Scientific Research under MURI award number  FOA-AFRL-AFOSR-2023-0004 (Mathematics of Digital Twins), the Department of Energy under award number DE-SC0023163 (SEA-CROGS: Scalable, Efficient, and Accelerated Causal Reasoning Operators, Graphs and Spikes for Earth and Embedded Systems), the National Science Foundations under award number 2425909 (Discovering the Law of Stress Transfer and Earthquake Dynamics in a Fault Network using a Computational Graph Discovery Approach), and
from the DoD Vannevar Bush Faculty Fellowship Program under ONR award number N00014-18-1-2363. 
SZ acknolwedges support from  NSF CAREER Award IIS-2046295,  NSF OAC-2311685 (Elements: A Convergent Physics-based and Data-driven Computing Platform for Building Modeling), NSF DMS-2529112 (Collaborative Research: MATH-DT: Computationally efficient hypercomplex variable-based sensitivity methods for rapid Digital Twin model updating), and DARPA SURGE HR0011-25-C-0036 (Rapid certification of additive manufactured
components using ML/AI and physics-based
modeling).

%% file: checklist.tex
\section*{Checklist}
% %%% BEGIN INSTRUCTIONS %%%
%The checklist follows the references. For each question, choose your answer from the three possible options: Yes, No, Not Applicable.  You are encouraged to include a justification to your answer, either by referencing the appropriate section of your paper or providing a brief inline description (1-2 sentences). 
%Please do not modify the questions.  Note that the Checklist section does not count towards the page limit. Not including the checklist in the first submission won't result in desk rejection, although in such case we will ask you to upload it during the author response period and include it in camera ready (if accepted).

%\textbf{In your paper, please delete this instructions block and only keep the Checklist section heading above along with the questions/answers below.}
% %%% END INSTRUCTIONS %%%

\begin{enumerate}
	
	\item For all models and algorithms presented, check if you include:
	\begin{enumerate}
		\item A clear description of the mathematical setting, assumptions, algorithm, and/or model. [Yes]
		\item An analysis of the properties and complexity (time, space, sample size) of any algorithm. [Yes]
		\item (Optional) Anonymized source code, with specification of all dependencies, including external libraries. [Not Applicable]
	\end{enumerate}
	
	\item For any theoretical claim, check if you include:
	\begin{enumerate}
		\item Statements of the full set of assumptions of all theoretical results. [Yes]
		\item Complete proofs of all theoretical results. [Yes]
		\item Clear explanations of any assumptions. [Yes]     
	\end{enumerate}
	
	\item For all figures and tables that present empirical results, check if you include:
	\begin{enumerate}
		\item The code, data, and instructions needed to reproduce the main experimental results (either in the supplemental material or as a URL). [Yes]
		\item All the training details (e.g., data splits, hyperparameters, how they were chosen). [Yes]
		\item A clear definition of the specific measure or statistics and error bars (e.g., with respect to the random seed after running experiments multiple times). [Yes]
		\item A description of the computing infrastructure used. (e.g., type of GPUs, internal cluster, or cloud provider). [Yes]
	\end{enumerate}
	
	\item If you are using existing assets (e.g., code, data, models) or curating/releasing new assets, check if you include:
	\begin{enumerate}
		\item Citations of the creator If your work uses existing assets. [Yes]
		\item The license information of the assets, if applicable. [Not Applicable]
		\item New assets either in the supplemental material or as a URL, if applicable. [Yes]
		\item Information about consent from data providers/curators. [Not Applicable]
		\item Discussion of sensible content if applicable, e.g., personally identifiable information or offensive content. [Not Applicable]
	\end{enumerate}
	
	\item If you used crowdsourcing or conducted research with human subjects, check if you include:
	\begin{enumerate}
		\item The full text of instructions given to participants and screenshots. [Not Applicable]
		\item Descriptions of potential participant risks, with links to Institutional Review Board (IRB) approvals if applicable. [Not Applicable]
		\item The estimated hourly wage paid to participants and the total amount spent on participant compensation. [Not Applicable]
	\end{enumerate}
	
\end{enumerate}

%% file: appendix.tex
% Supplementary material: To improve readability, you must use a single-column format for the supplementary material.
\thispagestyle{empty}
\onecolumn
\aistatstitle{Appendix}
\section{PDE Benchmarks}\label{sect:pde-benchmark}
We employed the following PDE benchmarks. 

\noindent \textbf{Burger's Equation.} The first benchmark is a viscous Burger's equation:
\begin{align}
	&u_t + uu_x - \nu \cdot u_{xx} = 0, \quad\quad \forall(t, x)\in (0, 1] \times (-1, 1) , \notag \\
	&u(t, -1) = u(t, 1) =  0, \quad u(0, x) = -\sin(\pi x), \label{eq:burgers}
\end{align}
where $\nu$ is the viscosity.  The solution is computed via the Cole-Hopf transformation with quadrature~\citep{CHEN2021110668}. We considered two cases: $\nu= 0.02$ and $\nu=0.001$.

\noindent \textbf{Nonlinear Elliptic PDE.} We then tested on a nonlinear elliptic equation~\citep{CHEN2021110668},
\begin{align}
	&-\Delta u(x_1, x_2) + u^3 = a(x_1, x_2), \quad \forall (x_1, x_2)\in \Omega, \notag \\
	& u(x_1, x_2) = 0, \quad \forall (x_1, x_2)\in \partial\Omega,  \label{eq:Nonlinear Elliptic}
\end{align}
where $\Omega \in[0, 1]^2$, the solution is defined as  
$u(x_1, x_2) = \sin{(\pi x_1)}\sin{(\pi x_2)} + 4\sin{(4\pi x_1)}\sin{(4\pi x_2)}$, and the source term $a$ is computed accordingly based on the PDE. 

\noindent\textbf{Eikonal PDE.} The third is a regularized Eikonal equation~\citep{chen2021solving,xutoward2025}, 
\begin{align}
	&|\nabla u(\x)|^2 = g(\x)^2 + \epsilon \Delta u(\x), \;\; \forall \x \in \Omega, \notag \\
	& u(\x) = 0, \;\; \forall \x \in \partial \Omega, \label{eq:ekonal}
\end{align}
where $\Omega = [0, 1]^2$, $g(\x) = 1$, and $\epsilon=0.1$. The solution was calculated from a highly-resolved finite difference solver provided by~\citep{chen2021solving}.

\noindent\textbf{Allen-Cahn Equation.} Fourth, we considered a stationary Allen-Cahn equation with Dirichlet boundary conditions, generalized from the benchmark used in~\citep{xutoward2025}, 
\begin{align}
	\sum_{i=1}^d \frac{\partial^2 u}{\partial x_i^2} + \gamma(u^m - u) = a(x_1, \ldots, x_d), \label{eq:allen-cahn}
\end{align}
where $(x_1, \ldots, x_d) \in [0, 1]^d$, $\gamma = 1$, $m=3$,  the solution is defined as 
\begin{align}
	u &= \sum\nolimits_{i=1}^d \big(\sin(2\pi \beta x_i)\cos(2\pi \beta x_{(i+1) \text{ mod } d})  + \sin(2\pi x_i)\cos(2\pi x_{(i+1) \text{ mod } d})\big)
\end{align}
%\zhec{For 2d:\begin{align}
%	u &= \sin(2\pi \beta x_1)\cos(2\pi \beta x_2)  + \sin(2\pi x_1)\cos(2\pi x_2)
%\end{align}}
%$$u = \sum\nolimits_{i=1}^d \left(\sin(2\pi \beta x_i)\cos(2\pi \beta x_{(i+1) \text{ mod } d})  + \cos(2\pi \beta x_i)\sin(2\pi \beta x_{(i+1) \text{ mod } d})\right)$$, 
%\begin{align}
%u = \sum\nolimits_{i=1}^d \left(\sin(2\pi \beta x_i)\cos(2\pi \beta x_{(i+1) \text{ mod } d})  + \cos(2\pi \beta x_i)\sin(2\pi \beta x_{(i+1) \text{ mod } d})\right), 
%\end{align}
and the corresponding source term $a$ is obtained through the equation. Here $\beta$ controls the frequency of the solution, and we varied $\beta=15, 20$, and the PDE dimension $d = 2, 4, 6$. 

\noindent\textbf{Nonlinear Darcy Flow}. Fifth, we employed a 6D nonlinear Darcy flow equation with Dirichlet boundary conditions~\citep{BATLLE2025113488}:
\begin{align}
	-\nabla \cdot \left(c\cdot \nabla u \right) + u^3 = a(x_1, \cdots, x_6),  \label{eq:darcy-flow}
\end{align}
where each $x_i \in [0, 1]$, $c(x_1, \cdots, x_6) = \exp{(\sin(\sum\nolimits_{i=1}^6 \cos\left(x_i\right)) )}$, the solution is crafted as  $u = \exp{\left(\sin\left(\beta \sum\nolimits_{i=1}^6 \cos\left(x_j\right)\right) \right)} $, and $a$ is computed based on the PDE. We set $\beta = 6$, which is more challenging than the case used in~\citep{BATLLE2025113488}. 

\section{Method Details}\label{sect:method-details}
\begin{itemize}
	\item \textbf{\sks}. We used the original JAX implementation\footnote{\url{https://github.com/BayesianAIGroup/Efficient-Kernel-PDE-Solver}}. The training is conduced via ADAM optimization with initial learning rate $10^{-3}$. The maximum number of iterations was set to 1M. In the training process, \sks does not sample mini-batches of collocation points to compute stochastic gradients. Instead, the full gradient is computed from the training objective at each step, and then fed into ADAM optimizer to update the momentum online and to adjust element-wise step-size. The optimization is stopped if the training objective does not improve for 1K updates. \sks employed Square Exponential (SE) kernel with different length-scales across the input dimensions. Alternative kernels, such as the Mat\'ern kernel, led to inferior performance. The length-scale hyperparameters were selected from a grid search, as detailed in the original paper~\citep{xutoward2025}. The nugget term was selected from  $\{5\times 10^{-5}, 10^{-5}, 5\times 10^{-6}, 10^{-6}, \ldots, 10^{-13}\}$.
	
	%For \ours, we minimize~\eqref{eq:our-loss} (with $\epsilon=0$), and used ADAM optimization with learning rate $10^{-3}$. The maximum number of epochs was set to 1M. We stopped the optimization if the loss stopped improving for 1K updates. For \yifan, we used the relaxed Gauss-Newton optimization developed in the original paper.
	%For \yifan and \ours, we used Square Exponential (SE) kernel with different length-scales across the input dimensions. We selected the nugget term from \{5E-5, 1E-5, 5E-6, 1E-6, $\ldots$, 1E-13\}. However, for solving the nonlinear elliptic PDE with DAKS, we used its default approach that assigns an adaptive nugget for the two sub-blocks in the Gram matrix. This gives the best performance for \yifan. The length-scales were selected from a grid search, from  $[0.1, 0.2]^2$  for the nonlinear elliptic and Eikonal PDEs, $[0.05, 0.01]^2$  for Allen-Cahn equation, and $[0.003, 0.05] \times [0.02, 0.3]$ for Burgers' equation. {For SKS, we selected $\alpha$ and $\beta$ in~\eqref{eq:our-loss} from the range $\{10^{-2}, 10^{-1}, \ldots, 10^{10}, 10^{12}, 10^{14}, 10^{15}, 10^{20}\}$, jointly with other hyperparameters, including the kernel length scales and nugget terms. To efficiently tune the hyperparameters, we first performed a random search to identify a promising group of hyperparameters. We then fixed all other hyperparameters and conducted a grid search over $\alpha$ and $\beta$. Finally, we fixed $\alpha$ and $\beta$ and performed a grid search over the remaining hyperparameters.}
	
	\item \textbf{\yifan}. We used the original JAX implementation provided by the authors\footnote{\url{https://github.com/yifanc96/NonLinPDEs-GPsolver}}. The training is performed using relaxed Gauss-Newton optimization. Note that applying the same method to train \sks  almost always led to divergence.  The kernel was selected from among the squared exponential (SE) kernel and the Mat\'ern kernel with degrees of freedom 3/2 or 5/2.  The nugget term was selected from the set $\{5\times 10^{-5}, 10^{-5}, 5\times 10^{-6}, 10^{-6}, \ldots, 10^{-13}\}$. Hyperparameters were chosen following the same procedure as for \sks. For solving the nonlinear elliptic PDE with DAKS, however, we employed its default strategy~\citep{chen2021solving}, which adaptively assigns nugget values to the two sub-blocks of the Gram matrix; this yielded the best performance for \yifan.

	%We used the original JAX implementation provided by the authors\footnote{\url{https://github.com/yifanc96/NonLinPDEs-GPsolver}}. A relaxed Gauss-Newton optimization is used for training. The nugget term was selected from \{5E-5, 1E-5, 5E-6, 1E-6, $\ldots$, 1E-13\}. However, for solving the nonlinear elliptic PDE with DAKS, we used its default approach that assigns an adaptive nugget for the two sub-blocks in the Gram matrix. This gives the best performance for \yifan. The hyperparameter selection follows the same process as used for \sks. The kernel was chosen from SE kernel, Mat\'ern kernel with degree of freedoms 3/2 or 5/2. 

	\item \textbf{PINN}. The network architecture was selected by varying the width and depth over $\{10,20, \ldots, 100\}$ and $\{2,3,5,8,10\}$, respectively. The \texttt{tanh} activation function was used. The weight of the boundary loss, $\lambda_b$, was selected from $\{1, 100, 500, 1000\}$. Training PINNs involved two stages: the first consisted of $10$K ADAM epochs with an initial learning rate of $10^{-3}$, followed by L-BFGS optimization until convergence, with the tolerance $10^{-9}$ and  the maximum number of iterations as 50K. 
	\item \textbf{\ours}. For our method with CP decomposition, we varied the rank $R$ in each dimension over $\{5, 10, 12, 15, 18, 20, 25\}$. For the TR decomposition, we set $R_0 = \cdots = R_d = R$ and selected $R$ from $\{3, 4, 5, 6, 7\}$. The number of inducing points for the GP components in each dimension was tuned within the range 20--720. Specifically, we first performed a random search to identify a promising configuration, followed by a grid search for refinement. The inducing points were equally spaced and kept fixed during training. For the factor function in each dimension, we chose kernel functions from the Squared Exponential (SE) and Matérn families with degrees of freedom 3/2 or 5/2. The length-scale parameters were selected from \{0.005:0.001:0.009, 0.01:0.01:0.1, 0.1:0.1:1.0, 1:1:8\}, while nugget values were drawn from $\{10^{-11}, 10^{-10}, 10^{-9}, 10^{-6}\}$ to ensure numerical stability. The regularization parameters $\alpha_1$ and $\alpha_2$ in~\eqref{eq:our-loss} were chosen from $\{10^{0}, 10^{1}, 10^{2}, \ldots, 10^{9}, 10^{10}\}$. 
\end{itemize}

\begin{table}[t]
	\caption {\small Relative $L^2$ error of solving \textit{more difficult} PDEs with \textit{a small number} of collocation points. The grids used in \sks are the same as in Table~\ref{tb:easy-small} of the main paper.} \label{tb:hard-small}
	\small
	\centering
	\begin{subtable}{\textwidth}
		\caption{\small  Burgers' equation \eqref{eq:burgers} with viscosity $\nu=0.001$.} \label{tb:burgers-nu1}
		\centering
		%	\resizebox{\linewidth}{!}{
		\begin{tabular}[c]{ccccc}
			\toprule
			\textit{Method} & 600 \cmt{($42 \times 14$)} & 1200 \cmt{($60 \times 20$)} & 2400 \cmt{($84 \times 28$)} & 4800 \cmt{($120 \times 40$)}\\
			\hline
			\yifan & 6.30E-01 & 5.08E-01 & 5.86E-01 & 3.86E-01\\
			%\yifan-grid & 3.48E-01 & 3.50E-01 & 2.23E-01 & 1.93E-01\\
			PINN & \textbf{2.07E-01} & 4.22E-01 & 5.18E-01 & 4.31E-01 \\
			SKS & 2.18E-01 & 1.81E-01 & \textbf{1.31E-01} & \textbf{3.08E-02}\\
			\ours-PF & \textbf{1.03E-01} & \textbf{7.30E-02} & \textbf{8.52E-02} & \textbf{5.18E-02}\\
			\ours-NT & 3.53E-01 & \textbf{1.78E-01} & 1.70E-01 & 2.26E-01\\
			\bottomrule
		\end{tabular}
		%	}
		\vspace{0.1in}
	\end{subtable}
	\begin{subtable}{\textwidth}
		\centering
		\caption{\small 2D Allen-Cahn equation \eqref{eq:allen-cahn}: $a=15, d=2$. } \label{tb:allen-cahn-15-few}
		%	\resizebox{\linewidth}{!}{
		\begin{tabular}[c]{ccccc}
			\toprule
			\textit{Method} & 600 \cmt{($25 \times 25$)} & 1200 \cmt{($35 \times 35$)} & 2400 \cmt{($49 \times 49$)} & 4800 \cmt{($70 \times 70$)}\\
			%\cline{2-5}
			\hline
			\yifan & 9.67E-01 & 9.36E-01 & 8.88E-01 & 8.12E-01\\
			%\yifan-grid & \textbf{6.79E-01} & 6.58E-01 & 6.27E-01 & 5.07E-01\\
			PINN & 5.69E0 & 8.77E0 & 6.03E0 & 7.62E0\\
			SKS & 9.62E-01 & 2.97E-01 & 7.28E-03 & 1.30E-04 \\
			\ours-PF & \textbf{6.43E-01} & \textbf{2.66E-04} & \textbf{8.39E-05} & \textbf{4.92E-06}\\
			\ours-NT & \textbf{6.42E-01} & \textbf{3.79E-04} & \textbf{4.36E-05} & \textbf{8.64E-06}\\
			\bottomrule
		\end{tabular}
		%	}
		\vspace{0.1in}
	\end{subtable}
	\begin{subtable}{\textwidth}
		\centering
		\caption{\small 2D Allen-Cahn equation \eqref{eq:allen-cahn}: $a=20, d=2$. } \label{tb:allen-cahn-20-few}
		%	\resizebox{\linewidth}{!}{
		\begin{tabular}[c]{ccccc}
			\toprule
			\textit{Method} & 600 \cmt{($25 \times 25$)} & 1200 \cmt{($35 \times 35$)} & 2400 \cmt{($49 \times 49$)} & 4800 \cmt{($70 \times 70$)}\\
			%\cline{2-5}
			\hline
			\yifan & {9.63E-01} & 9.29E-01 & 8.76E-01 & 7.98E-01 \\
			%\yifan-grid & \textbf{6.78E-01} & \textbf{6.57E-01} & 6.23E-01 & 5.08E-01\\
			PINN & 7.04E0 & 8.18E0 & 8.30E0 & 4.30E0 \\
			SKS & 1.00E0 & 9.77E-01 & 2.56E-01 & 1.39E-03\\
			\ours-PF & \textbf{6.74E-01} & \textbf{1.53E-01} & \textbf{2.73E-03} & \textbf{5.39E-05}\\
			\ours-NT & \textbf{6.51E-01} & \textbf{9.55E-02} & \textbf{1.32E-03} & \textbf{6.75E-05}\\
			\bottomrule
		\end{tabular}
		%	}
	\end{subtable}
	%\vspace{-0.1in}
\end{table}
\section{Proof of Lemma~\ref{thm:1} and~\ref{thm:2}}\label{sect:proof}
\begin{definition}[Sobolev Space and Weighed Sobolev Space]\label{def:sobolev} Let $\Omega \subset \mathbb{R}^d$ be an open subset, and $k \in \mathbb{N}$. The Sobolev space $H^k(\Omega)$ is defined as: 
	\[
	H^k(\Omega) := \left\{ g \in L^2(\Omega) \,\middle|\, \partial^{\balpha} g \in L^2(\Omega),\ \forall\, |\balpha| \leq k \right\},
	\]
	where $L^2(\Omega)$ is the space of square-integrable functions and $\partial^{\balpha} g$ denotes the weak derivative of $g$ of multi-index $\balpha=(\alpha_1, \ldots, \alpha_d)$, with total order $|\balpha| = \sum_i \alpha_i \leq k$. 
	For $\v = (v_1, \ldots, v_d)\in \mathbb{R}_+^d$, the weighted Sobolev space $H^k_{\v}(\Omega)\subseteq H^k(\Omega)$ is defined as
	\begin{align}
		H^k_{\v}(\Omega) = \left\{g\in H^k(\Omega): \left\|\partial^{\balpha} g\right\|_{L^2(\Omega)}\lesssim \left(v_j\right)^k\|g\|_{H^k(\Omega)}\ \text{for $|\balpha|=k$ and $j=1, \ldots, d$}\right\}.
	\end{align}
	
\end{definition}
\begin{proof}
	The proof of Lemma~\ref{thm:1} is based on the existing results for the Tucker decomposition in \citep{griebel2023analysis}. 
	In particular, \citep[Theorem 2]{griebel2023analysis} states that for  the Tucker format,  
	\begin{align}
		u(x_1, \ldots, x_d) = \sum_{r_1=1}^{R_1} \ldots \sum_{r_d=1}^{R_d} w_{r_1\ldots r_d} \cdot f_{r_1}^1(x_1) \ldots f_{r_d}^d(x_d), 
	\end{align}	
	%where $\{f_{r_d}^d\}_{1 \le r_d \le R_d, 1\le d \le D}$ are a collection of one dimensional functions, 
	if we choose $R_1 = \cdots = R_d = (\frac{\sqrt{d}}{\varepsilon})^{1/k}$, then the error is $\lesssim \varepsilon$. Here $\{f_{r_i}^i(\cdot)\}_{1 \le r_i \le R_i, 1\le i \le d}$ are a collection of one-dimensional functions. We can rewrite the Tucker format as follows,    
	\begin{align*}
		&\sum_{r_1=1}^{R_1}\ldots \sum_{r_d=1}^{R_d} w_{r_1\ldots r_d} \cdot f_{r_1}^1(x_1) \ldots f_{r_d}^d(x_d) \notag \\
		&= \sum_{r_1=1}^{R_1} f^1_{r_1}(x_1) \ldots \sum_{r_{d-1}=1}^{R_{d-1}}u^{d-1}_{r_{d-1}}(x_{d-1})\left(\sum_{r_d=1}^{R_d}w_{r_1\ldots r_d}\cdot f_{r_d}^{R_d}(x_d)\right),
	\end{align*}
	which can be viewed as  CP format that includes a summation of $\overline{R} = \prod_{i=1}^{d-1}R_i = (\frac{\sqrt{d}}{\varepsilon})^{\frac{(d-1)}{k}}$ products of rank-one functions. This establishes the first part of Lemma~\ref{thm:1}. 
	
	The proof of Lemma~\ref{thm:2} builds on \citep[Theorem~3]{griebel2023analysis}. In particular, the approximation results apply to the TT format~\citep{oseledets2011tensor} by invoking \citep[Theorems~4-5 and Remark~2]{griebel2023analysis}, where we specialize to the case in which each factor function has input dimension one. Since the TT format is a special case of the TR format with $R_0 = R_d = 1$, these approximation results carry over to the TR format as well.
\end{proof}

\section{Bound of RKHS Norm}
To prove Lemma~\ref{th:convergence}, we first prove the following RKHS norm bound. 
\begin{lemma}\label{th:rkhs-norm-bound}
		Let $\Gcal^1(\Omega_0), \ldots, \Gcal^d(\Omega_0)$ be a collection of RKHS's defined on $\Omega_0 \subset \mathbb{R}$, and let $\Ucal = \Gcal^1 \kron \ldots \kron \Gcal^d$ denote their tensor-product RKHS. For any function of the form 
		\[
		u(x_1, \ldots, x_d) = \sum_{r=1}^R \prod_{i=1}^d f^i_r(x_i), \quad f^i_r \in \Gcal^i,  
		\]
		we have
		\begin{align}
			\|u\|_\Ucal \le \left[\frac{1}{d}\sum_{i=1}^d\sum_{r=1}^R  \|f^i_r\|^2_{\Gcal^i}\right]^{d/2}.  \label{eq:rkhs-norm-bound}
		\end{align}
	\end{lemma}
	\begin{proof}
		We first obtain the inner product in $\Ucal$. Since $\Ucal$ is a tensor-product RKHS, the kernel associated with $\Ucal$ is the product of the kernels associated with each $\Gcal^i$. Therefore, given arbitrary two functions $q^1 \kron \ldots \kron q^d$ and $g^1 \kron \ldots \kron g^d$ where each $q^i, g^i \in \Gcal^i$, their inner product under $\Ucal$ is defined as 
		\begin{align}
			\langle q^1 \kron \ldots \kron q^d, g^1 \kron \ldots \kron  g^d\rangle_{\Gcal} = \langle q^1, g^1 \rangle_{\Gcal^1} \cdots \langle q^d, g^d \rangle_{\Gcal^d}.
		\end{align}
		This inner product further extends to the sum of tensor products (\ie the CP format): 
		\begin{align}
			&u = \sum_{r} \bigotimes_{i=1}^d f^i_r, \;\; q = \sum_{l} \bigotimes_{i=1}^d g^i_l \notag \\
			&\left\langle u, q \right\rangle_\Ucal = \sum_{r} \sum_{l} \prod_{i=1}^d \left\langle f^i_r, g^i_l \right\rangle_{\Gcal^i} \label{eq:inner-product}.
		\end{align}
		
		According to~\eqref{eq:inner-product}, we have
		\begin{align}
			\|u\|_\Ucal^2 = \sum_{r=1}^R \sum_{l=1}^R \prod_{i=1}^d \left \langle f^i_r, f^i_l \right \rangle_{\Gcal^i}. 
		\end{align}
		Let us define $R \times R$ matrices $\A^i$ where each element $A^i_{rl} = \langle f^i_r, f^i_l \rangle_{\Gcal^i}$. Then
		\begin{align}
			\|u\|_\Ucal^2 = \sum_{r}\sum_l \prod_{i=1}^d A^i_{rl} = \left\langle \circ_{i=1}^{d-1} \A^i, \A^d \right\rangle_F
		\end{align}
		where $\langle \cdot, \cdot \rangle_F$ is the Frobenius inner product, $\circ$ is the Hadamard (element-wise) product. Leveraging Cauchy-Schwarz inequality under Frobenius inner product, we have 
		\begin{align}
			\|u\|_\Ucal^2  =  \left\langle \circ_{i=1}^{d-1} \A^i, \A^d \right\rangle_F \le \|\circ_{i=1}^{d-1} \A^i\|_F \cdot \|\A^d\|_F.
		\end{align}
		Since in general $\|\A \circ \B\|_F \le \|\A\|_F \cdot \|\B\|_F$, we have
		\begin{align}
			\|\circ_{i=1}^{d-1} \A^i\|_F  \le \prod_{i=1}^{d-1} \|\A^i\|_F,
		\end{align}
		and therefore
		\begin{align}
			\|u\|_\Ucal^2 \le \prod_{i=1}^d \|\A^i\|_F. \label{eq:rkhs-bound-1}
		\end{align}
		For each $\A^i$, we have
		\begin{align}
			\|\A^i\|_F^2 &= \sum_{r=1}^R \sum_{l=1}^R \left\langle f^i_r, f^i_l \right\rangle_{G^{i}}^2 \notag \\
			&\le \sum_{r=1}^R \sum_{l=1}^R \|f^i_r\|_{\Gcal^i}^2 \cdot \|f^i_l\|_{\Gcal^i}^2 \quad\quad (\text{Cauchy-Schwarz Inequality}) \notag \\
			&= \left[\sum_{r=1}^R \|f^i_r\|_{\Gcal^i}^2\right]^2.
		\end{align}
		Combining with~\eqref{eq:rkhs-bound-1}, we obtain 
		\begin{align}
			\|u\|_\Ucal^2 \le  \prod_{i=1}^d \left(\sum_{r=1}^R \|f^i_r\|_{\Gcal^i}^2\right)
		\end{align}
		% $\|u\|_\Ucal^2 \le  \prod_{i=1}^d \left[\sum_{r=1}^R \|f^i_r\|_{\Gcal^i}^2\right]$, and so
		% \begin{align}
		%     \|u\|_\Ucal \le \prod_{i=1}^d \left[\sum_{r=1}^R \|f^i_r\|_{\Gcal^i}^2\right]^{1/2}.
		% \end{align}
		% For any $a_1, \ldots, a_n\ge 0$, we have  $a_1 + \ldots + a_n \le (\sqrt{a_1} + \ldots \sqrt{a_n})^2$, and hence $(a_1 + \ldots + a_n)^{1/2} \le \sqrt{a_1} + \ldots + \sqrt{a_n}$. Applying this inequality, we obtain 
		% \begin{align}
		%     \|u\|_\Ucal \le \prod_{i=1}^d \sum_{r=1}^R \|f^i_r\|_{\Gcal^i}
		% \end{align}
		
		We then leverage the AM–GM inequality (Arithmetic Mean-Geometric Mean inequality): for any $a_1, \ldots, a_n\ge 0$, $\left(a_1\cdots a_n\right)^{1/n}\le \frac{1}{n}(a_1 + \ldots + a_n)$, and so $a_1 \cdots a_n \le \left[\frac{1}{n}(a_1 + \ldots + a_n)\right]^n$. Therefore, we obtain
		\begin{align}
			\|u\|_\Ucal \le \left[\frac{1}{d} \sum_{i=1}^d\sum_{r=1}^{R} \|f^i_r\|^2_{\Gcal^i} \right]^{d/2}.
		\end{align}
	\end{proof}

\section{Proof of Lemma~\ref{th:convergence}}\label{sect:proof-convergence}
\begin{proof}
	The proof consists of the following steps. 
	
	\noindent\textbf{Step 1}. First, we show that given an arbitrarily small $\varepsilon>0$, there exists a rank $R$ and a set of one-dimensional factor functions $\{f^i_r \in \Gcal^i\}_{1\le r \le R} $ in each dimension $i$ ($1 \le i \le d$), such that their combination via CP decomposition~\eqref{eq:cp},  denoted as
	\begin{align}
		\uhat = \sum_{r=1}^R \prod_{i=1}^d \fhat^i_r(x_i), \label{eq:uhat}
	\end{align}
	satisfies
	\begin{align}
		\|\uhat - u^*\|_\Ucal \le \varepsilon. \label{eq:gap} 
	\end{align}
	
	To show this, denote the associated kernel with each $\Gcal^i$ as $\kappa_i$. Since $\Ucal = \Gcal^1 \kron \ldots \kron \Gcal^d$, the kernel inducing $\Ucal$ is therefore $\kappa(\x, \x')=\prod_{i=1}^d \kappa_i(x_i, x'_i)$. Since each $\kappa_i$ is universal, the product kernel $\kappa$ is also universal on the product domain. 
	As a result, if we denote the eigenfunctions of $k_i$ as $\{\phi^i_{j}(\cdot)\}_{j=1}^\infty$ and the eigenvalues as $\{\lambda^i_{j}\}_{j=1}^\infty$ (note that all $\lambda^i_j>0$), then we have $\Phi_{j_1\ldots j_d}(x_1, \ldots, x_d)\coloneqq\phi^1_{j_1}(x_1)\cdots \phi^d_{j_d}(x_d)$ constitute orthonormal bases in $L^2(\Omega)$. We can represent the true PDE solution as  
	\begin{align*}
		u^*(x_1, \ldots, x_d) = \sum_{j_1, \ldots, j_d = 1}^\infty \langle u^*, \Phi_{j_1\ldots j_d}\rangle \cdot\Phi_{j_1\ldots j_d}(x_1, \ldots, x_d),
	\end{align*}
	where the dot product $\langle\cdot, \cdot\rangle$ is defined in $L^2(\Omega)$. Since $\kappa = \kron_{i=1}^d \kappa_i$, $\{\Phi_{j_1\ldots j_d}\}$ and $\{\prod_{i=1}^d\lambda^i_{j_i}\}$ form the eigenfunctions and eigenvalues of $\kappa$, respectively. Because $u^* \in \Ucal$, we have
	\begin{align*}
		\|u^*\|^2_\Ucal\coloneqq\sum_{j_1, \ldots, j_d = 1}^\infty\frac{\langle u^*, \Phi_{j_1\ldots j_d}\rangle^2}{\lambda^1_{j_1}\cdots\lambda^d_{j_d}}<\infty. 
	\end{align*}
	Consequently, for any $\epsilon>0$, there exists a sufficiently large $I(\epsilon)$, such that 
	\begin{align*}
		\|u^*-\sum_{j_1, \ldots, j_d\leq I(\varepsilon)}\langle u, \Phi_{j_1\ldots j_d}\rangle\Phi_{j_1\ldots j_d}\|_\Ucal<\epsilon. 
	\end{align*}
	This truncation can be expressed as 
	\begin{align}
		\sum_{j_1=1}^{I(\varepsilon)} \phi^1_{j_1}(x_1) \sum_{j_2=1}^{I(\varepsilon)} \phi^2_{j_2}(x_2) \ldots \sum_{j_{d-1}=1}^{I(\varepsilon)} \phi^{d-1}_{j_{d-1}}(x_d-1) \left(\sum_{j_d=1}^{I(\varepsilon)}\phi^d_{j_d}(x_d)\langle u^*, \Phi_{j_1\ldots j_d}\rangle\right), \label{eq:truncation}
	\end{align}
	which can be viewed as a CP decomposition form in~\eqref{eq:uhat}, with rank $R = I(\varepsilon)^{d-1}$. 
	We  map each multi-index $(j_1, \cdots, j_{d-1})$ with  $j_i \le I(\varepsilon)$, to an index $r \in  \{1, \cdots, R\}$. For each such $r$, denote the corresponding tuple by $(j_{r_1}, \ldots, j_{r_{d-1}})$. Then we set $\fhat^i_r = \phi^i_{j_{r_i}}$ for $i < d$, and $\fhat^d_r = \sum_{j_d=1}^{I(\varepsilon)}\phi^d_{j_d}(x_d)\langle u^*, \Phi_{j_{r_1}\ldots j_{r_{d-1}} j_d}\rangle$. Clearly, each $\fhat^i_r \in \Gcal^i$, and the approximation $\uhat=\sum_{r=1}^R \prod_{i=1}^d \fhat^i_r(x_i)$ satisfies that $\|\uhat - u^*\|_\Ucal \le \varepsilon$. 
	
	\noindent\textbf{Step 2.} Next, we show that with rank $R$ and an appropriate choice of $\delta$, the optimization problem~\eqref{eq:optimization-formulation} using the CP form is feasible; that is, a solution exists.
	
	Denote by $\Mcal^i$ the projection of the collocation set $\Mcal$ onto the $i$-th coordinate axis: $\Mcal^i = \{ [\x_m]_i :\x_m \in \Mcal \}$, \ie the set of all distinct coordinates of the collocation points along dimension $i$. 
	we first construct a set of intermediate optimization problems. Each problem $\Zcal^i_r$ $(1 \le i \le d, 1 \le r \le R)$ is defined as:
	\begin{align}
		\begin{cases}
			\underset{f^i_r\in\Gcal^i}{\text{minimize}}\;\; \|f^i_r\|_{\Gcal^i}\\
			\text{s.t.  }  f^i_r(x_m) = \fhat^i_r(x_m), \;\; x_m \in \Mcal^i, \label{eq:int-u-star}
		\end{cases}
	\end{align}
	where $\fhat^i_r$ is from the approximation $\uhat$ in~\eqref{eq:uhat}. This is a standard kernel regression problem. Let us denote the minimizer of~\eqref{eq:int-u-star} by $\fhat^{i}_{r\Mcal}$. According to the optimal recovery theorem~\citep{owhadi2019operator}, we have $\fhat^{i}_{r\Mcal}$ takes the kernel interpolation form~\eqref{eq:factor-func}, and 
	\begin{align}
		\|\fhat^{i}_{r\Mcal}\|_{\Gcal^i} \le \|\fhat^{i}_r\|_{\Gcal^i}. \label{eq:component-rkhs-bound}
	\end{align}
	Let us define
	\begin{align}
		\uhat_{\Mcal} =\sum_{r=1}^R \prod_{i=1}^d \fhat^{i}_{r\Mcal}(x_i).
	\end{align}

	Obviously,  $\uhat_\Mcal(\x_m) - \uhat(\x_m) = 0$ for every $\x_m \in \Mcal_\Omega$.  According to the sampling inequality ( Proposition A.1 of~\citep{batlle2023error}), when $h$ is sufficiently small (note that the fill-in distance $h_\Omega \le h$), 
	\begin{align}
		\|\uhat_\Mcal - \uhat\|_{H^s(\Omega)} \lesssim h^\tau \|\uhat_\Mcal - \uhat\|_{H^{s+\tau}(\Omega)}. \label{eq:sampling_ineq1}
	\end{align}
	
	We now consider bounding 
	\begin{align}
		\Lcal = \|\Pcal(\uhat_\Mcal) - \Pcal(u^*)\|_{H^{k}(\Omega)} + \|\Bcal(\uhat_\Mcal)- \Bcal(u^*)\|_{H^{t}(\partial \Omega)}. 
	\end{align}
	Using the norm triangle inequality, we have 
	\begin{align}
		\Lcal \le \|\Pcal(\uhat_\Mcal) - \Pcal(\uhat)\|_{H^{k}(\Omega)} +  \|\Pcal(\uhat) - \Pcal(u^*)\|_{H^{k}(\Omega)} \notag \\
		+ \|\Bcal(\uhat_\Mcal) - \Bcal(\uhat)\|_{H^t(\partial \Omega)} + \|\Bcal(\uhat) - \Bcal(u^*)\|_{H^t(\partial \Omega)} \label{eq:norm-bound}
	\end{align}
	
	Combining~\eqref{eq:norm-bound} with the PDE stability~\eqref{eq:stability-right} in Assumption~\eqref{assump1} and the result~\eqref{eq:sampling_ineq1}, we obtain
	\begin{align}
		\Lcal \lesssim h^\tau \cdot \|\uhat_\Mcal - \uhat\|_{H^{s+\tau}(\Omega)} + \|\uhat - u^*\|_{H^{s}(\Omega)}. \label{eq:bound-int}
	\end{align}
	Since $\Ucal \hookrightarrow H^{s+\tau}(\Omega)$ (C3 of Assumption~\ref{assump1}) and $H^{s+\tau}(\Omega) \hookrightarrow H^s(\Omega)$, we have
	\begin{align}
		\|\uhat_\Mcal - \uhat\|_{H^{s+\tau}(\Omega)} \lesssim \|\uhat_\Mcal - \uhat\|_\Ucal, \quad \|\uhat - u^*\|_{H^{s}(\Omega)} \lesssim \|\uhat - u^*\|_\Ucal. 
	\end{align}
	%where $C_1$ and $C_2$ are constants independent of terms at both sides of the corresponding inequalities.  
	Combining with~\eqref{eq:bound-int}, we further derive that
	\begin{align}
		&\Lcal \lesssim h^\tau \cdot \|\uhat_\Mcal - \uhat\|_\Ucal + \|\uhat - u^*\|_\Ucal \notag \\
		&\lesssim h^\tau\cdot\|\uhat_\Mcal\|_\Ucal +  h^\tau\cdot\|\uhat\|_\Ucal +  \varepsilon. \quad\quad\left(\text{see }~\eqref{eq:gap}\right)
	\end{align}
	Leveraging the RKHS norm bound~\eqref{eq:rkhs-norm-bound} in Lemma~\ref{th:rkhs-norm-bound}, we obtain that
	\begin{align}
		&\Lcal \le C \left( h^\tau \left(\frac{1}{d}\sum_{i=1}^d\sum_{r=1}^R \|\fhat^i_{r\Mcal}\|^2_{\Gcal^i}\right)^{d/2} +  h^\tau\left(\frac{1}{d}\sum_{i=1}^d\sum_{r=1}^R\|\fhat^i_{r}\|^2_{\Gcal^i}\right)^{d/2} + \varepsilon \right)\notag \\
		&\le  C\left(h^\tau\cdot 2\left(\frac{1}{d}\sum_{i=1}^d\sum_{r=1}^R\|\fhat^i_{r}\|^2_{\Gcal^i}\right)^{d/2} + \varepsilon\right), \quad\quad (\text{according to }~\eqref{eq:component-rkhs-bound}) \label{eq:L-bound}
	\end{align}
	where $C$ is a constant independent of terms on both sides of the inequality. Note that each $\fhat^i_r \in \Gcal^i$ and can be constructed from the eigenfunctions of $\kappa_i$ that is universal ---see~\eqref{eq:truncation} , therefore $\|\fhat^i_r\|_{\Gcal^i}$ is a bounded constant  partly determined by $\varepsilon$.

	Meanwhile, because $H^k(\Omega) \hookrightarrow C^0(\Omega)$ and $H^t(\partial\Omega) \hookrightarrow C^0(\partial \Omega)$, we have 
	\begin{align}
		&\|\Pcal(\uhat_\Mcal) - \Pcal(u^*)\|_{C^0(\Omega)} \lesssim \|\Pcal(\uhat_\Mcal) - \Pcal(u^*)\|_{H^{k}(\Omega)}, \notag \\
		&\|\Bcal(\uhat_\Mcal)- \Bcal(u^*)\|_{C^0(\partial \Omega)} \lesssim
		\|\Bcal(\uhat_\Mcal)- \Bcal(u^*)\|_{H^{t}(\partial \Omega)}. \label{eq:temp4}
	\end{align}
	In addition, at any collocation point $\x_m$, 
	\begin{align}
		\left(\Pcal(\uhat_\Mcal)(\x_m) - \Pcal(u^*)(\x_m)\right)^2 \le \|\Pcal(\uhat_\Mcal) - \Pcal(u^*)\|^2_{C^0(\Omega)}, \notag \\
		\left(\Bcal(\uhat_\Mcal)(\x_m) - \Bcal(u^*)(\x_m)\right)^2 \le \|\Bcal(\uhat_\Mcal) - \Bcal(u^*)\|^2_{C^0(\partial \Omega)}. \label{eq:temp5}
	\end{align}
	Combining~\eqref{eq:temp4}, ~\eqref{eq:temp5} and~\eqref{eq:L-bound}, we therefore obtain that
	\begin{align}
		&\frac{1}{M_\Omega} \sum_{m=1}^{M_\Omega} \left(\Pcal(\uhat_\Mcal)(\x_m) - \Pcal(u^*)(\x_m)\right)^2   
		+\frac{1}{M-M_\Omega}\sum_{m=M_\Omega+1}^M \left(\Bcal(\uhat_\Mcal)(\x_m) - \Bcal(u^*)(\x_m)\right)^2 \notag \\
		&\le \Lcal^2 \le C^2\left(h^\tau\cdot 2\left(\frac{1}{d}\sum_{i=1}^d\sum_{r=1}^R\|\fhat^i_{r}\|^2_{\Gcal^i}\right)^{d/2} + \varepsilon\right)^2.\label{eq:milestone-1}
	\end{align}
	Therefore,  if we set 
	\begin{align}
		\delta = C\left(h^\tau\cdot 2\left(\frac{1}{d}\sum_{i=1}^d\sum_{r=1}^R\|\fhat^i_{r}\|^2_{\Gcal^i}\right)^{d/2} + \varepsilon\right) \label{eq:delta-setting}
	\end{align}
	in~\eqref{eq:optimization-formulation}, $\uhat_\Mcal$ is at least a feasible solution, and the optimization problem~\eqref{eq:optimization-formulation} is feasible.
	
	\noindent \textbf{Step 3.} Let us denote by $u^\dagger$ the solution of problem~\eqref{eq:optimization-formulation}.  We then analyze the error of $u^\dagger$. Using an idea similar to~\citep{xutoward2025}, we define two error functions, 
	\begin{align}
		&\xi_P(\x) = \Pcal(u^\dagger)(\x) - \Pcal(u^*)(\x),\quad \x \in \Omega\notag \\
		& \xi_B(\x) = \Bcal(u^\dagger)(\x) - \Bcal(u^*)(\x), \quad \x \in \partial \Omega.
	\end{align}
	Our goal is to bound the $L^2$ norm of the error functions: $\|\xi_P\|_{H^0(\Omega)}$ and $\|\xi_B\|_{H^0(\partial \Omega)}$. Let us first consider $\xi_P$. To bound $\|\xi_P\|_{H^0(\Omega)}$, we decompose $\Omega$ into a Voronoi diagram according to the collocation points, which results in $M_\Omega$ regular non-overlapping regions, $\Tcal_1 \cup \ldots \cup \Tcal_{M_\Omega} = \Omega$, where each region $\Tcal_i$ only contains one collocation point $\x_i$, and its filled-distance $h_i \lesssim h$ ($1 \le i \le M_\Omega$). Accordingly, we can decompose the squared $L^2$ norm as 
	\begin{align}
		\|\xi_P\|^2_{H^0(\Omega)} = \sum_{i=1}^{M_\Omega} \int_{\Tcal_i} \xi_P(\x)^2\d \x = \sum_{i=1}^{M_\Omega} \|\xi_P\|_{H^0(\Tcal_i)}^2. \label{eq:step2-temp0}
	\end{align}
	Leveraging the fact that  
	\[
	\xi_P(\x)^2 = \left(\xi_P(\x) - \xi_P(\x_i) + \xi_P(\x_i)\right)^2  \le 2  \left(\xi_P(\x) - \xi_P(\x_i)\right)^2 + 2\xi_P(\x_i)^2,
	\]
	we obtain
	\begin{align}
		\|\xi_P\|_{H^0(\Tcal_i)}^2 \lesssim \|\xi_P - \xi_P(\x_i)\|_{H^0(\Tcal_i)}^2 + \lambda(\Tcal_i) \xi_P(\x_i)^2, \label{eq:step2-temp1}
	\end{align}
	where $\lambda(\Tcal_i)$ is the volume of $\Tcal_i$.
	
	The function $\xi_P - \xi_P(\x_i)$ is zero at $\x_i$. Since 
	the aspect ratio of $\Tcal_i$ is bounded, we can apply the sampling inequality --- a.k.a Poincar\'e inequality,
	\begin{align}
		\|\xi_P - \xi_P(\x_i) \|_{H^0(\Tcal_i)} \lesssim h_i^k \|\xi_P - \xi_P(\x_i) \|_{H^k(\Tcal_i)} \lesssim h^k \|\xi_P - \xi_P(\x_i) \|_{H^k(\Tcal_i)}.
	\end{align}
	Applying the mean inequality, 
	\begin{align}
		\|\xi_P - \xi_P(\x_i) \|^2_{H^0(\Tcal_i)}
		\lesssim h^{2k} \left(\|\xi_P\|^2_{H^k(\Tcal_i)} + \|\xi_P(\x_i)\|^2_{H^k(\Tcal_i)}\right) =h^{2k} \left(\|\xi_P\|^2_{H^k(\Tcal_i)} +  \lambda(\Tcal_i)\xi_P(\x_i)^2\right).\label{eq:step2-temp2}
	\end{align}
	Since $\lambda(\Tcal_i) \lesssim h^d$, combining \eqref{eq:step2-temp0}, \eqref{eq:step2-temp1} and \eqref{eq:step2-temp2}, we can obtain
	\begin{align}
		\|\xi_P\|^2_{H^0(\Omega)} &\lesssim h^{2k} \sum_i \|\xi_P  \|^2_{H^k(\Tcal_i)} + (h^d+h^{2k+d}) \sum_i \xi_P(\x_i)^2 \notag \\
		& \lesssim h^{2k}\|\xi_P \|^2_{H^k(\Omega)} +  (h^d+h^{2k+d}) \cdot M_\Omega\cdot \delta^2,  \label{eq:step2-milestone-1}
	\end{align}
	where $\delta^2$ comes from the constraint of \eqref{eq:optimization-formulation}. Using  a similar approach, we can show that 
	\begin{align}
		\|\xi_B\|^2_{H^0(\partial \Omega)} \lesssim h^{2t}\|\xi_B\|^2_{H^t(\partial \Omega)} + (h^d + h^{2t+d}) (M-M_\Omega)\delta^2.
		\label{eq:step2-milestone-2} 
	\end{align}
	
	Combining \eqref{eq:step2-milestone-1} and \eqref{eq:step2-milestone-2},
	\begin{align}
		\left(\|\xi_P\|_{H^0(\Omega)} +\|\xi_B\|_{H^0(\partial \Omega)}\right)^2 \lesssim h^{2 \cdot \rho} \left(\|\xi_P\|_{H^k(\Omega)} + \|\xi_B\|_{H^t(\partial \Omega)}\right)^2 + (h^d + h^{2\rho+d})M\delta^2, \label{eq:step2-temp3}
	\end{align}
	where $\rho = \min(k,t)$. When $h \lesssim M^{-\frac{1}{d}}$ and is sufficiently small, we have $(h^d + h^{2\rho+d})M  \le 1 + h^{2\rho}  \le 2$, and 
	\begin{align}
		\left(\|\xi_P\|_{H^0(\Omega)} +\|\xi_B\|_{H^0(\partial \Omega)}\right)^2 \lesssim h^{2 \cdot \rho} \left(\|\xi_P\|_{H^k(\Omega)} + \|\xi_B\|_{H^t(\partial \Omega)}\right)^2 + 2\delta^2.\label{eq:step2-temp3}
	\end{align}
	Using the PDE stability~\eqref{eq:stability-left} and~\eqref{eq:stability-right}, we obtain
	\begin{align}
		\|u^\dagger - u^*\|_{H^l(\Omega)} \lesssim  h^\rho\|u^\dagger - u^*\|_{H^s(\Omega)} + \delta.
	\end{align}
	Since $\Ucal \hookrightarrow H^{s+\tau}(\Omega) \hookrightarrow H^s(\Omega)$, we further have
	\begin{align}
		&\|u^\dagger - u^*\|_{H^l(\Omega)} \lesssim  h^\rho\|u^\dagger - u^*\|_{\Ucal} + \delta  \notag \\
		& \lesssim h^\rho\|u^\dagger - \uhat\|_{\Ucal} + h^\rho\|\uhat - u^*\|_{\Ucal} + \delta \notag \\
		& \lesssim h^\rho \|u^\dagger - \uhat\|_{\Ucal} + h^\rho \varepsilon+ \delta \notag \\
		& \lesssim h^\rho \|u^\dagger\|_\Ucal + h^\rho \|\uhat\|_\Ucal + h^\rho \varepsilon + \delta \label{eq:temp-end}
	\end{align}
	Denote each factor function in $u^\dagger$ as $f^{i\dagger}_r$. According to the RKHS norm bound~\eqref{eq:rkhs-norm-bound} in Lemma~\ref{th:rkhs-norm-bound}, we have $\|u^\dagger\|_\Ucal \le \left(\frac{1}{d}\sum_{i=1}^d\sum_{r=1}^R \|f^{i\dagger}_r\|_{\Gcal^i}^2\right)^{d/2}$. Since $\uhat_\Mcal$ is a feasible solution to~\eqref{eq:optimization-formulation}, we must have
	\begin{align}
		&\sum_{i=1}^d\sum_{r=1}^R \|f^{i\dagger}_r\|_{\Gcal^i}^2 \le \sum_{i=1}^d\sum_{r=1}^R \|\fhat^{i}_{r\Mcal}\|_{\Gcal^i}^2 \notag \\
		&\le \sum_{i=1}^d\sum_{r=1}^R \|\fhat^i_r\|_{\Gcal^i}^2  \quad\quad (\text{according to }~\eqref{eq:component-rkhs-bound}).  \label{eq:temp-end2}
	\end{align}
	Combining~\eqref{eq:temp-end},~\eqref{eq:delta-setting} and~\eqref{eq:temp-end2}, we obtain
	\begin{align}
		\|u^\dagger - u^*\|_{H^l(\Omega)} &\lesssim (h^\rho + h^\tau)C_0  + ( h^\rho+1) \varepsilon \notag \\
		&\lesssim h^\nu C_0 + (h^\rho+1) \varepsilon, 
	\end{align}
	where $\nu = \min(\rho, \tau) = \min(k, t, \tau)$, and $C_0 = \left(\frac{1}{d}\sum_{i=1}^d\sum_{r=1}^R \|\fhat^i_r\|^2_{\Gcal^i}\right)^{d/2}$.
	Therefore, when $h\rightarrow 0$, 
	\[
	\|u^\dagger - u^*\|_{H^l(\Omega)} \lesssim \varepsilon.
	\]
\end{proof}
\section{Proof of Proposition \ref{th:prop}}\label{sect:prop}
We first construct the Lagrange function. The constraint optimization problem~\eqref{eq:optimization-formulation} is equivalent to the mini-max optimization problem over the Lagrange function, 
\begin{align}
	\min_{\{f^i_r\in\Gcal^i\}} \max_{\beta\ge 0}   \sum_{i=1}^d\sum_{r=1}^R \|f^i_r\|^2 + \beta &\Bigg[ \frac{1}{M_\Omega} \sum_{m=1}^{M_\Omega} \left(\Pcal(u)(\x_m) - a(\x_m)\right)^2 - \frac{\delta^2}{2}\notag \\   
	&+\frac{1}{M-M_\Omega}\sum_{m=M_\Omega+1}^M \left(\Bcal(u)(\x_m) - b(\x_m)\right)^2 - \frac{\delta^2}{2}
	\Bigg]. \label{eq:mini-max}
\end{align}
Suppose the feasible region is non-empty. Let us denote the optimum of \eqref{eq:mini-max} as $(\{{f^i_r}^\dagger\}, \beta^\dagger)$. Then $\{{f^i_r}^\dagger\}$ is a minimizer of \eqref{eq:optimization-formulation}. If we now set $\alpha_1 = \alpha_2 = \beta^\dagger$ in \eqref{eq:our-loss}, then optimizing $\eqref{eq:our-loss}$ will recover the minimizer $\{{f^i_r}^\dagger\}$.
\begin{table}[t]
	\caption {\small Relative $L^2$ error of conventional numerical solvers and \ours according to the ground-truth solution. }\label{tb:comparison-finite-diff}
	\small
	\centering
	\begin{subtable}{\textwidth}
		% \caption{\small Nonlinear Elliptic PDE~\eqref{eq:elliptic}.}\label{tb:elliptic-simple}
		\caption{\small Nonlinear elliptic PDE \eqref{eq:Nonlinear Elliptic}  } \label{tb:elliptic-simple-conventional}
		\centering
	%	\resizebox{\linewidth}{!}{
			\begin{tabular}[c]{ccccc}
				\toprule
				\textit{Method} & $18 \times 18$ & $25 \times 25$ & $35 \times 35$ & $49 \times 49$\\
				\hline
				FEM & 1.68E-02 & 8.61E-03 & 4.35E-03 & 2.20E-03 \\
				%CGL & 1.54E-01 & 8.27E-02 & 4.29E-02 & 2.22E-02 \\
				FD & 3.02E-02 & 1.60E-02 & 8.32E-03 & 4.30E-03 \\
				\ours-PF & \textbf{1.97E-06} & \textbf{2.82E-07} & \textbf{1.28E-07} & \textbf{4.04E-08}\\
				\ours-NT & \textbf{1.78E-06} & \textbf{3.52E-07} & \textbf{1.74E-07} & \textbf{4.06E-08}\\
				\bottomrule
			\end{tabular}
	%	}
	\end{subtable}
	\begin{subtable}{\textwidth}
		\caption{\small The 2D Allen-Cahn equation~\eqref{eq:allen-cahn} with $a=15$.}
		\centering
	%	\resizebox{\linewidth}{!}{
			\begin{tabular}[c]{cccccc}
				\toprule
				\textit{Method} & $80 \times 80$ & $90 \times 90$ & $150 \times 150$ & $200 \times 200$\\
				\hline
				FEM & 3.32E-02 & 2.62E-02 & 9.65E-03 & 5.42E-03 \\
				%	CGL & 2.57E-01 & 2.12E-01 & 9.09E-02 & NA \\
				FD & 1.21E-01 & 9.45E-02 & 3.30E-02 & 1.84E-02 \\
				\ours-PF & \textbf{6.00E-06} & \textbf{1.21E-06} & \textbf{4.87E-06} & \textbf{1.43E-06}\\
				\ours-NT & \textbf{3.99E-06} & \textbf{1.28E-06} & \textbf{1.70E-06} & \textbf{1.76E-06}\\
				\bottomrule
			\end{tabular}
		%}
	\end{subtable}
	\begin{subtable}{\textwidth}
		\caption{\small The 2D Allen-Cahn equation~\eqref{eq:allen-cahn} with $a=20$.}
		\centering
	%	\resizebox{\linewidth}{!}{
			\begin{tabular}[c]{cccccc}
				\toprule
				\textit{Method} & $80 \times 80$ & $90 \times 90$ & $150 \times 150$ & $200 \times 200$\\
				\hline
				FEM & 5.94E-02 & 4.66E-02 & 1.71E-02 & 9.62E-03 \\
				%CGL & 3.94E-01 & 3.35E-01 & 1.55E-01 & NA \\
				FD & 2.29E-01 & 1.75E-01 & 5.97E-02 & 3.31E-02 \\
				\ours-PF & \textbf{8.50E-06} & \textbf{8.47E-06} & \textbf{5.90E-06} & \textbf{5.14E-06}\\
				\ours-NT & \textbf{9.03E-06} & \textbf{7.42E-06} & \textbf{5.79E-06} & \textbf{4.80E-06}\\
				\bottomrule
			\end{tabular}
	%	}
	\end{subtable}
\end{table}
\section{Numerical Solvers}\label{sect:numerical-sover}
The P2 Galerkin finite element method (FEM) is implemented using the MATLAB PDE Toolbox\footnote{\url{https://www.mathworks.com/products/pde.html}} with high-order quadratic Lagrange elements and a multi-level mesh refinement strategy. Specifically, we utilized quadratic (P2) Lagrange elements on triangular meshes, which provide third-order convergence rate for smooth solutions. Each element contains nodes at vertices and edge midpoints. The multi-level mesh hierarchy is generated via progressive refinement.  For the nonlinear elliptic PDE, we have  $h_{\max} \in \{0.055, 0.04, 0.0286, 0.0204, 0.01\}$, and for the Allen-Cahn equations ($a=15$ and $a=20$), we have $h_{\max} \in \{0.04, 0.0286, 0.0204, 0.0143\}$. We used the weak formulation to construct and assemble the stiffness matrix, mass matrix, and loading vectors. The resulting nonlinear discretized system was solved using Newton iterations combined with an Armijo line search to guarantee convergence.

The finite difference (FD) scheme discretized each PDE using centered second-order finite differences. The resulting nonlinear system was solved with a Newton-Krylov method, where the inverse of the Jacobian was computed using iterative Krylov subspace techniques.
\section{Point-Wise Error}\label{sect:point-wise}
For a fine-grained evaluation, we examined the point-wise errors of \ours, \sks, and PINN when solving Burgers' equation~\eqref{eq:burgers} with $\nu=0.001$ and the 2D Allen-Cahn equation~\eqref{eq:allen-cahn} with $a=20$ and $d=2$. The number of collocation points was varied from ${2400, 28\text{K}}$ for Burgers' equation and from ${600, 2400, 4800, 8100}$ for the Allen-Cahn equation. The point-wise absolute errors are shown in Figure~\ref{fig:point-wise-error}.

As expected, when the number of collocation points is small, all methods incur larger errors across the domain --- for instance, using 2400 points for Burgers' equation or 600 points for the Allen-Cahn equation. Notably, the error of PINN with 2400 collocation points on Burgers' equation was so large that we excluded its error plot in this case. Increasing the number of collocation points substantially improves performance for all methods. Nevertheless, our approach consistently produces smaller errors across the domain. For example, in Figure~\ref{fig:point-wise-2dallen-cahn}, when solving the 2D Allen-Cahn equation with 2400 collocation points, \sks exhibits large errors throughout the domain, while both \ours-PF and \ours-NT confine relatively larger errors only near the boundary.

These results demonstrate that our method not only improves global solution accuracy but also achieves lower local error.

%For a fine-grained evaluation, we examined the point-wise error of \ours, \sks and PINN in solving the Burger's equation~\eqref{eq:burgers} with $\nu = 0.001$ and  2D Allen-Canh equations~\eqref{eq:allen-cahn} with $a=20$ and $d=2$. We varied the number of collocation points from \{2400, 28K\} for the Burgers equation, and from \{600, 2400, 4800, 8100\} for the Allen-Cahn equation. The point-wise absolute error is shown in Figure~\ref{fig:point-wise-error}. As we can see, when the number of collocation points is small, all the methods exhibit larger errors across the domain, \eg using 2400 collocation points to solve the Burgers equation, and  600 collocation points to solve the Allen-Cahn equation. Note that the error of PINN solving the Burger's equation with 2400 collocation points is much larger and we did not include its point-wise error in this case.  Along with incorporating more collocation points, every method shows substantial improvement. However, our method consistently exhibit smaller errors across the domain. For example, in Figure~\ref{fig:point-wise-2dallen-cahn}, in solving 2D Allen-Cahn with 2400 collocation points, \sks yields large errors almost across the entire domain, while both \ours-PF and \ours-NT only have relatively larger errors on the boundary. The results demonstrate that our method not only gives better global solution accuracy, but also achieves smaller error locally. 
\begin{table}[t]
	\caption {\small Relative $L^2$ error of solving PDEs on \textit{irregularly shaped}  domains.} \label{tb:irr-sample}
	\small
	\centering
	\begin{subtable}{\textwidth}
		\centering
		\caption{\small 2D Allen-Cahn equation \eqref{eq:allen-cahn}: $a=15, d=2$. } \label{tb:allen-cahn-15-few}
		%\resizebox{\linewidth}{!}{
		\begin{tabular}[c]{ccccc}
			\toprule
			\textit{Shape} & PINN & SKS  & \ours-PF & \ours-NT \\
			%\cline{2-5}
			\hline
			Triangle & 4.97E0 & 4.08E-03 & 2.85E-06 & 4.80E-06 \\
			%\yifan-grid & \textbf{6.79E-01} & 6.58E-01 & 6.27E-01 & 5.07E-01\\
			Circle & 7.48E0 & 4.50E-03 & 1.72E-06 & 1.57E-06 \\
			\bottomrule
		\end{tabular}
		%}
		\vspace{0.1in}
	\end{subtable}
	\begin{subtable}{\textwidth}
		\centering
		\caption{\small Nonlinear elliptic PDE \eqref{eq:Nonlinear Elliptic}. } \label{tb:allen-cahn-20-few}
		%\resizebox{\linewidth}{!}{
		\begin{tabular}[c]{ccccc}
			\toprule
			\textit{Shape} & PINN & SKS  & \ours-PF & \ours-NT \\
			%\cline{2-5}
			\hline
			Triangle & 4.10E-4 & 3.27E-04 & 3.98E-08 & 4.29E-08 \\
			%\yifan-grid & \textbf{6.79E-01} & 6.58E-01 & 6.27E-01 & 5.07E-01\\
			Circle & 2.66E-04 & 4.47E-04 & 3.68E-08 & 3.31E-08 \\
			\bottomrule
		\end{tabular}
		%}
	\end{subtable}
	%\vspace{-0.1in}
\end{table}
\begin{figure*}
	\centering
	\setlength{\tabcolsep}{0pt}
	\begin{subfigure}[b]{0.9\linewidth}
		\includegraphics[width=\linewidth]{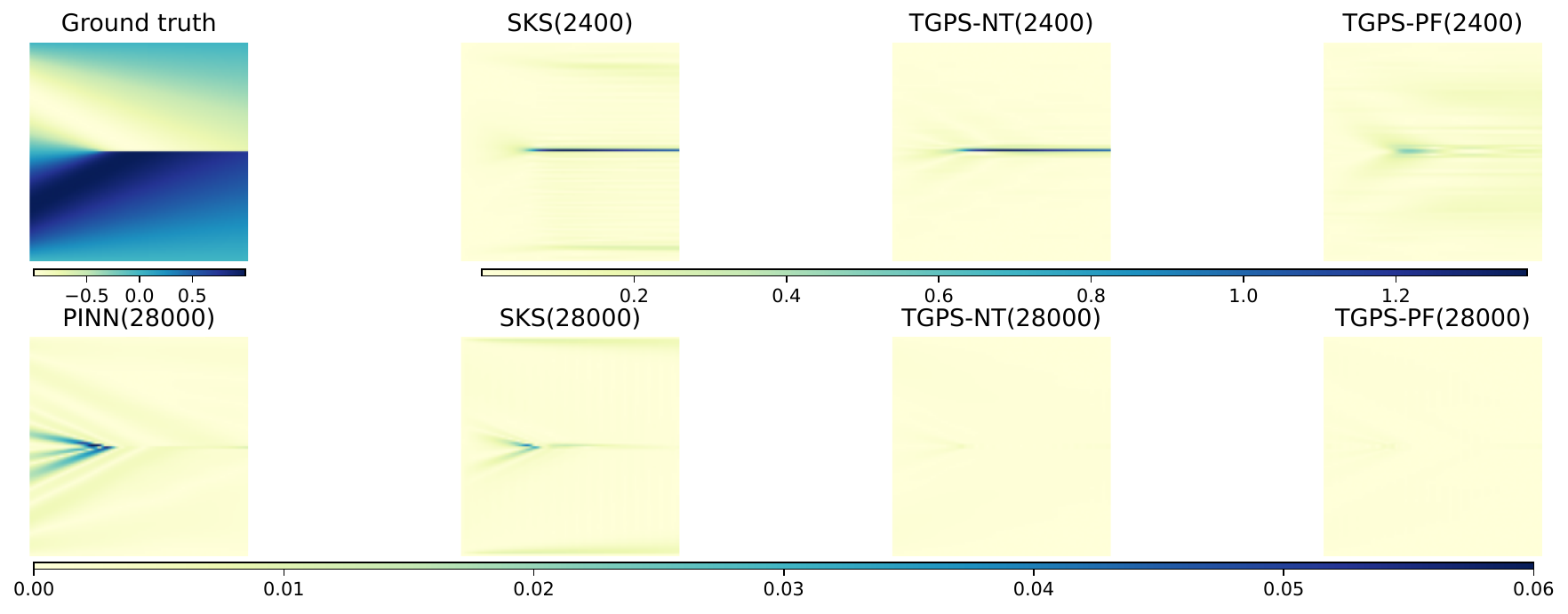}
		\caption{\small Burger's equation~\eqref{eq:burgers} with $\nu = 0.001$.}
		\label{fig:point-wise-burgers0001}
	\end{subfigure} 	
	\begin{subfigure}[b]{0.9\linewidth}
		\includegraphics[width=\linewidth]{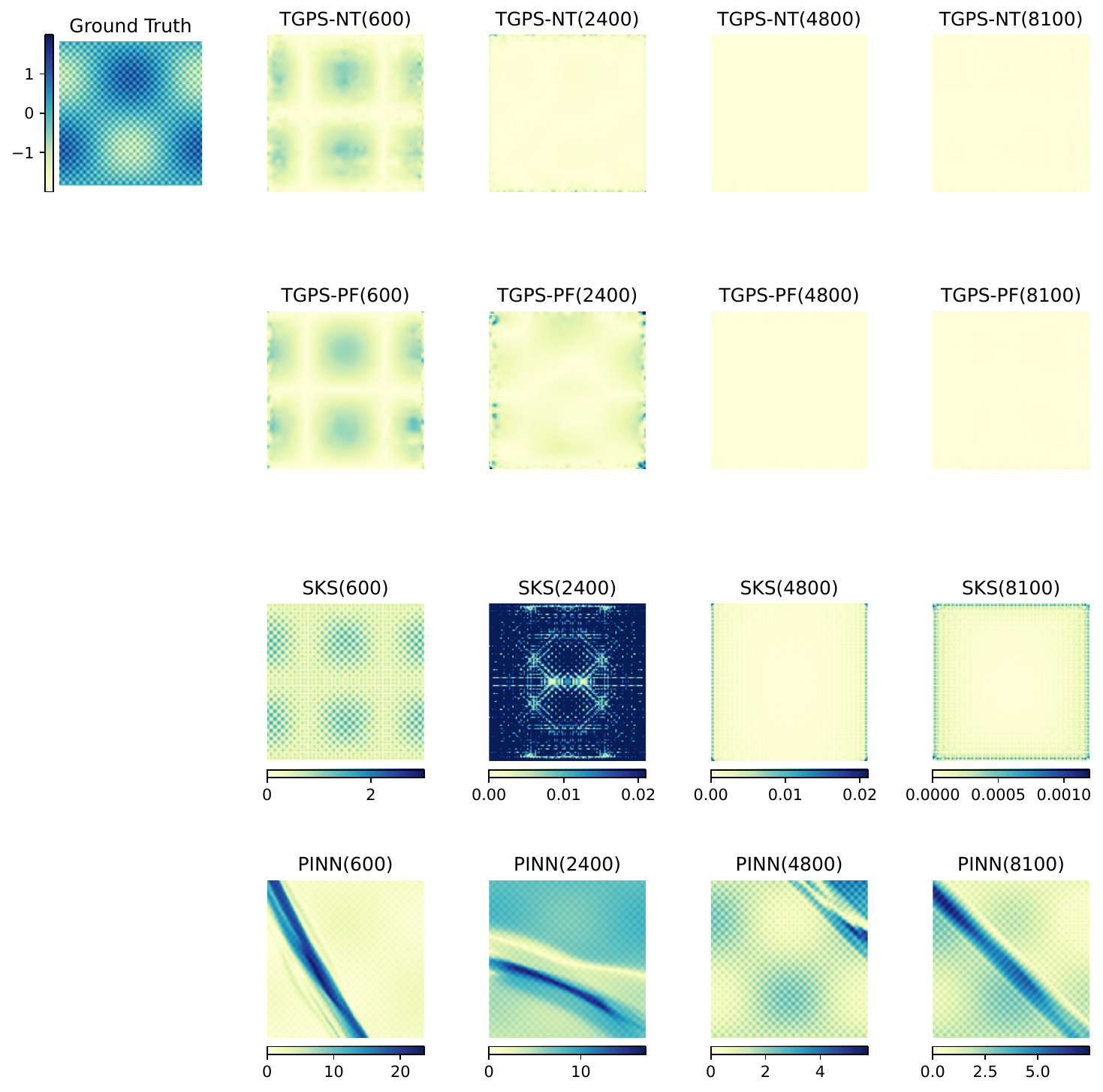}
		\caption{\small 2D Allen-Cahn~\eqref{eq:allen-cahn} ($a=20, d=2$). }
		\label{fig:point-wise-2dallen-cahn}
	\end{subfigure} 
	\caption{\small Point-wise error. Inside each parenthesis is the number of collocation points.  }\label{fig:point-wise-error}
\end{figure*}
\section{Irregular Domains}\label{sect:irregular}
We evaluated performance on the nonlinear elliptic PDE~\eqref{eq:Nonlinear Elliptic} and the 2D Allen-Cahn equation~\eqref{eq:allen-cahn} ($a=15, d=2$) over two irregular domains: (i) an inscribed circle within $[0,1]\times[0,1]$, and (ii) a triangle with vertices at $(0,0)$, $(1,0)$, and $(0.5,1)$. The reference solutions follow Section~\ref{sect:pde-benchmark}, with boundary conditions derived accordingly. Competing baselines included \sks and PINN, all tested with 10K collocation points. For fairness, each method also used the same 396 uniformly sampled boundary points.
Since \sks is restricted to regular domains, we embedded each irregular domain into a $100 \times 100$ regularly-spaced virtual grid over $[0,1]^2$. In contrast, \ours and PINN directly operated on (the same set of) 10K randomly sampled collocation points from the irregular domains (including the 396 boundary points). 

The relative $L^2$ errors are summarized in Table~\ref{tb:irr-sample}. As shown, PINN again failed on the Allen-Cahn equation, yielding a relative $L^2$ error larger than one, likely due to the spectral bias. While \sks remained reasonably accurate, its errors deteriorated by several orders of magnitude compared to regular domains. For example, when solving the nonlinear elliptic PDE, \sks achieved errors on the order of $10^{-6}$ with a $49 \times 49$ grid on $[0,1]^2$ (see Table~\ref{tb:elliptic-simple} in the main paper), but errors increased to $10^{-4}$ on the circle and triangle domains even with a denser $100 \times 100$ (virtual) grid. Similarly, for the 2D Allen-Cahn equation, \sks reached $10^{-6}$ error on the rectangular domain (see Table~\ref{tb:allen-cahn-15} in the main paper), but only $10^{-3}$ on the irregular domains.

By contrast, \ours consistently attained errors on the order of $10^{-8}$ (elliptic PDE) and $10^{-6}$ (Allen–Cahn) on both irregular domains --- matching its performance on regular domains. Notably, PINN also maintained its error level on the elliptic PDE. Overall, these results highlight the robustness of our mesh-free solver: its accuracy remains stable regardless of domain geometry. The pointwise error plots in Figure~\ref{fig:Irregular-Nonlinear elliptic} and ~\ref{fig:Irregular-2D-Allen-Cahen} further corroborate this conclusion.

\input{irregular_domain_qw}

\section{Ablation Studies}\label{sect:ablation}
Furthermore, we conducted ablation studies to evaluate the influence of two important types of hyperparameters in our model: kernel parameters and the number of factor functions (i.e., rank). For this purpose, we employed Burgers' equation~\eqref{eq:burgers} with $\nu=0.02$ and the 2D Allen-Cahn equation~\eqref{eq:allen-cahn} with $a=15$. 

\noindent\textbf{Kernel Hyperparameters.}
We first examined the effect of kernel length-scale parameters. For Burgers' equation ($\nu=0.02$), we fixed the spatial length-scale to 0.04 and varied the temporal length-scale over $\{0.001, 0.01, 0.5, 1.0, 2.0\}$, keeping the number of factor functions consistent with our main experiments (Table~\ref{tb:easy-small}). As shown in Table~\ref{tb:burgers0.02-ls1}, both \ours-PF and \ours-NT are highly sensitive to the spatial length-scale, achieving the lowest relative $L^2$ error when it is set to 0.5. Deviations in either direction caused orders-of-magnitude error growth.

Next, we fixed the temporal length-scale at 0.2 and varied the spatial length-scale over $\{0.001, 0.01, 0.1, 0.5, 1.0\}$. The results (Table~\ref{tb:burgers0.02-ls2}) reveal the same pattern: optimal performance occurs for intermediate values, while smaller or larger scales lead to substantial degradation. Finally, we tested our method on the 2D Allen-Cahn equation with identical length-scales across both spatial dimensions, varying the parameter over $\{0.001, 0.01, 0.05, 0.1, 0.2\}$. As shown in Table~\ref{tb:allen-cahn-15-ls}, the smallest error arises at 0.05, with larger or smaller values again producing error increases by orders of magnitude. Collectively, these results underscore the critical role of length-scale parameters in determining model performance.

From a theoretical standpoint, the kernel and its hyperparameters determine the GP prior and thus the associated RKHS. The closer this RKHS matches the regularity class and characteristic scales of the true PDE solution, the better the approximation quality and stability. Hyperparameter tuning can therefore be viewed as selecting an RKHS whose inductive bias is well aligned with the target solution. While there is no universal rule for kernel and hyperparameter selection, we have found several useful heuristics that guide kernel and length-scale choices:

\begin{itemize}
	\item Smooth solutions (\eg moderate-viscosity Burgers', elliptic PDE): Gaussian kernels generally perform well.
	\item Solutions with sharp gradients or low regularity (\eg near-shocks in Burgers’ with small viscosity): less smooth kernels such as Mat\'ern-2/3 are often more appropriate. 
	\item Higher-frequency structure: Smaller length-scales help capture oscillatory behavior.
\end{itemize}
For instance, in Burgers' equation, viscosity $0.02$ was well modeled using a Gaussian kernel in space, whereas viscosity $0.001$ required switching to Mat\'ern-2/3. For the nonlinear elliptic PDE and Allen-Cahn equation, Gaussian kernels remained effective but required different length-scales (\eg $0.1$ for the elliptic PDE \textit{vs.} $0.04$ for Allen–Cahn with  $a=15$ ), consistent with their different effective frequency content.

\noindent\textbf{Number of Factor Functions.}
We next evaluated the effect of the number of factor functions (rank). With length-scale parameters fixed as in Table~\ref{tb:easy-small}, we varied the rank over $\{3, 5, 10, 20\}$. Experiments were conducted with 600 and 2400 collocation points for the Burgers' equation, and with 4800 and 22.5K points for the 2D Allen-Cahn equation. As reported in Tables~\ref{tb:burgers0.02-rank} and \ref{tb:allen-cahn-15-rank}, rank 10 consistently yielded the best performance. Smaller ranks (3 or 5) reduced expressivity and resulted in errors one to two orders of magnitude larger. Increasing the rank to 20 offered no further gain and, in some cases (\eg \ours-NT with 600 collocation points), worsened performance. A similar trend was observed for the Allen-Cahn equation, although its performance was somewhat more robust to rank variations.

Overall, these studies demonstrate that both kernel length-scales and rank are crucial hyperparameters. Too small a rank limits model expressivity, degrading accuracy, while excessively large ranks increase computational and optimization burdens without clear benefits.

\begin{table}
	\caption {\small Relative $L^2$ error of \ours with different length-scales.} \label{tb:length-scale-study}
	\small
	\centering
	\begin{subtable}{\textwidth}
		\centering
		\caption{\small Solving Burgers' equation \eqref{eq:burgers} with viscosity $\nu=0.02$. The number of collocation point is 2400 and the spatial length-scale is fixed to 0.04. } \label{tb:burgers0.02-ls1}
		\begin{tabular}[c]{cccccc}
			\toprule
			Temporal length-scale & 0.001 & 0.01  & 0.5 & 1.0 &  2.0 \\
			\hline
			\ours-PF & 9.98E-01 & 8.60E-01 & 3.93E-03 & 1.17E-02 & 2.17E-02 \\
			\ours-NT & 1.01E-0 & 9.54E-01 & 3.80E-03 & 1.25E-02 & 2.42E-02 \\
			\bottomrule
		\end{tabular}
	\vspace{0.1in}
	\end{subtable}

	\begin{subtable}{\textwidth}
		\centering
		\caption{\small Solving Burgers' equation with viscosity $\nu=0.02$. The number of collocation point is 2400 and the temporal length-scale is fixed to 0.2.} \label{tb:burgers0.02-ls2}
		%\resizebox{\linewidth}{!}{
		\begin{tabular}[c]{cccccc}
			\toprule
			Spatial length-scale & 0.001 & 0.01  & 0.1 & 0.5 &  1.0 \\
			%\cline{2-5}
			\hline
			\ours-PF & 1.00E0 & 9.66E-01 & 1.32E-02 & 1.86E-02 & 2.73E-01 \\
			%\yifan-grid & \textbf{6.79E-01} & 6.58E-01 & 6.27E-01 & 5.07E-01\\
			\ours-NT & 1.44E0 & 9.23E-01 & 4.07E-02 & 2.56E-01 & 3.23E-01 \\
			\bottomrule
		\end{tabular}
		%}
		\vspace{0.1in}
	\end{subtable}
	\begin{subtable}{\textwidth}
		\centering
		\caption{\small Solving 2D Allen-Cahn equation~\eqref{eq:allen-cahn} with $a=15$. The number of collocation points is 4800. The length-scales are the same for both spatial dimensions. } \label{tb:allen-cahn-15-ls}
		%\resizebox{\linewidth}{!}{
		\begin{tabular}[c]{cccccc}
			\toprule
			Length-scale & 0.001 & 0.01 & 0.05 & 0.1 & 0.2 \\
			%\cline{2-5}
			\hline
			\ours-PF & 1.00E0 & 1.04E0 & 8.69E-04 & 9.99E-01 & 1.25E0 \\
			%\yifan-grid & \textbf{6.79E-01} & 6.58E-01 & 6.27E-01 & 5.07E-01\\
			\ours-NT & 1.00E0 & 1.04E0 & 5.79E-04 & 4.01E0 & 7.70E0 \\
			\bottomrule
		\end{tabular}
		%}
	\end{subtable}
	%\vspace{-0.1in}
\end{table}
\begin{table}
	\caption {\small Relative $L^2$ error of \ours with different numbers of factor functions (rank).} \label{tb:cp-rank-study}
	\small
	\centering
	\begin{subtable}{\textwidth}
		\centering
		\caption{\small Solving Burgers' equation \eqref{eq:burgers} with viscosity $\nu=0.02$.  Inside parentheses are the number of collocation points.   }\label{tb:burgers0.02-rank}
		%\resizebox{\linewidth}{!}{
		\begin{tabular}[c]{ccccc}
			\toprule
			Rank & 3 & 5  & 10 & 20 \\
			%\cline{2-5}
			\hline
			\ours-PF (600) & 2.59E-01 & 1.47E-02 & 7.03E-03 & 5.54E-03  \\
			%\yifan-grid & \textbf{6.79E-01} & 6.58E-01 & 6.27E-01 & 5.07E-01\\
			\ours-NT (600) & 3.03E-01 & 8.40E-01 & 8.42E-03 & 1.14E-02  \\
			\ours-PF (2400) & 1.64E-02 & 1.04E-03 & 1.85E-04 &  2.43E-04 \\
			\ours-NT (2400) & 1.97E-02  & 9.82E-04 & 2.37E-04 & 4.06E-04  \\
			\bottomrule
		\end{tabular}
		%}
		\vspace{0.1in}
	\end{subtable}
	\begin{subtable}{\textwidth}
		\centering
		\caption{\small Solving 2D Allen-Cahn equation \eqref{eq:allen-cahn}: $a=15, d=2$. } \label{tb:allen-cahn-15-rank}
		\begin{tabular}[c]{ccccc}
			\toprule
			Rank & 3 & 5 & 10 & 20 \\
			\hline
			\ours-PF (4800) & 1.61E-05 & 1.23E-05 & 6.52E-06 & 1.01E-05 \\
			\ours-NT(4800) & 8.92E-06 & 6.50E-06 & 5.87E-06 & 9.66E-06 \\
			\ours-PF (22500) & 1.63E-05 & 7.10E-06 & 5.51E-06 &  9.22E-06\\
			\ours-NT (22500) & 1.57E-05 & 5.27E-06 & 2.21E-06 & 7.34E-06 \\
			\bottomrule
		\end{tabular}
	\end{subtable}
	%\vspace{-0.1in}
\end{table}

\section{Limitation}
While the partial freezing strategy and Newton’s method allow us to derive closed-form ALS updates, they also make the update process effectively behave like a sequence of fixed-point iterations. A well-known limitation of such iterations is their sensitivity to initialization: if the starting point is poorly chosen, the iterations may diverge instead of converging. To mitigate this risk, in future work we plan to design additional regularization techniques that explicitly incorporate parameter estimates from earlier iterations into the update rules. This would provide a stabilizing effect, improving both the robustness and the reliability of our method.
%While using the partial freezing strategy and Newton's method enables closed-form ALS updates, these methods make our model updates essentially performing fixed point iterations. Consequently, an improper initialization might cause divergence. To address this problem, in the future work, we plan to develop additional regularization approach that more explicitly integrates model parameter estimate in previous iterations into the update to enhance the robustness and stability of our method.  

\begin{table}[t]
	\caption {\small Relative $L^2$ error of solving \textit{higher dimensional} PDEs.}\label{tb:high-dim}
	%\vspace{-0.25in}
	\small
	\centering
	\begin{subtable}{\textwidth}
		\caption{\small 4D Allen-Cahn equation \eqref{eq:allen-cahn}: $a=15, d= 4$. } \label{tb:4dallen-cahn}
		\centering
		%	\resizebox{\linewidth}{!}{
		\begin{tabular}[c]{ccccc}
			\toprule
			\textit{Method} & 8000 \cmt{1600} & 16000 \cmt{3200} & 32000 \cmt{6400} & 48000 \cmt{9600}\\
			%\cline{2-5}
			% & (1600) & (90$\times$90) & (150$\times$150) & (200$\times$200)\\
			\hline
			PINN & 8.01E-01 & 7.87E-01 & 7.68E-01 & 7.62E-01 \\
			SKS & 9.85E-01 & 9.91E-01 & 9.93E-01 & 7.50E-01 \\
			\ours-CP-PF & \textbf{4.00E-04} & 1.10E-04 & \textbf{2.76E-05} & 1.65E-05 \\
			\ours-CP-NT & \textbf{3.93E-04} & \textbf{9.80E-05} & \textbf{1.87E-05} & 1.66E-05 \\
			\ours-TR-PF & 5.50E-01 & \textbf{5.44E-05} & \textbf{2.59E-05} & 1.09E-05 \\
			\ours-TR-NT & 5.65E-01 & \textbf{7.48E-05} & \textbf{1.85E-05} & \textbf{7.47E-06} \\
			\bottomrule
		\end{tabular}
		%	}
	\end{subtable}
	\begin{subtable}{\textwidth}
		\caption{\small 6D Allen-Cahn equation \eqref{eq:allen-cahn}: $a=15, d= 6$.} \label{tb:6dallen-cahn}
		\centering
		%	\resizebox{\linewidth}{!}{
		\begin{tabular}[c]{ccccc}
			\toprule
			\textit{Method} & 16000 \cmt{(3200)} & 32000 \cmt{(6400)} & 48000 \cmt{(9600)} & 96000 \cmt{(19200)}\\
			%\cline{2-5}
			\hline
			PINN & 6.79E-01 & 6.10E-01 & 8.66E-01 & 6.12E-01\\
			\ours-CP-PF & 6.62E-01 & 4.79E-04 & 3.50E-04 & 8.34E-05 \\
			\ours-CP-NT & 6.62E-01 & 4.98E-04 & 3.00E-04 & 7.75E-05 \\
			\ours-TR-PF & \textbf{1.50E-02} & \textbf{4.23E-05} & \textbf{3.39E-05} & \textbf{1.11E-05} \\
			\ours-TR-NT & 7.52E-01 & \textbf{4.62E-05} & \textbf{4.48E-05} & \textbf{1.12E-05} \\
			\bottomrule
		\end{tabular}
		%	}
	\end{subtable}
	\begin{subtable}{\textwidth}
		\caption{\small 6D Nonlinear Darcy flow equation \eqref{eq:darcy-flow}. } \label{tb:6ddarcy}
		\centering
		%	\resizebox{\linewidth}{!}{
		\begin{tabular}[c]{cccccc}
			\toprule
			\textit{Method} & 1000 \cmt{(200)} & 2000 \cmt{(400)} & 4000 \cmt{(800)} & 8000 \cmt{(1600)} & 16000 \cmt{(3200)} \\
			% & (200$\times$40) & (400$\times$40) & (700$\times$40) & (800$\times$40)\\
			\hline
			PINN & \textbf{1.04E-02} & \textbf{3.65E-02} & 7.34E-02 & 8.87E-02 & 1.22E-02 \\
			DAKS & 3.87E0 & 3.81E0 & NA & NA & NA\\
			\ours-CP-PF & 5.24E-01 & 3.02E-01 & 1.67E-01 & 2.96E-02 & \textbf{6.59E-03} \\
			\ours-CP-NT & 5.19E-01 & 3.70E-01 & 1.71E-01 & 2.42E-02 & 1.31E-02\\
			\ours-TR-PF & 5.97E-01 & \textbf{5.15E-02} & \textbf{2.67E-02} & \textbf{7.53E-03} & \textbf{5.48E-03} \\
			\ours-TR-NT & 5.75E-01 & \textbf{5.39E-02} & \textbf{1.48E-02} & \textbf{5.31E-03} & \textbf{4.02E-03} \\
			\bottomrule
		\end{tabular}
		%	}
	\end{subtable}
\end{table}

%% file: irregular_domain_qw.tex
\begin{figure*}
	\centering
	\setlength{\tabcolsep}{0pt}
	\begin{subfigure}[b]{1.0\linewidth}
		\includegraphics[width=\linewidth]{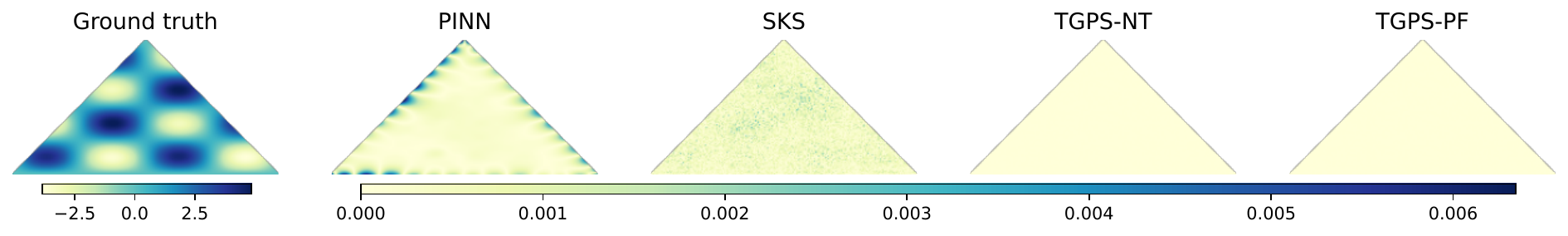}
		\caption{\small \textit{Triangle domain}.}
		\label{fig:triangle-non-elliptic}
	\end{subfigure} 
	\begin{subfigure}[b]{1.0\linewidth}
		\includegraphics[width=\linewidth]{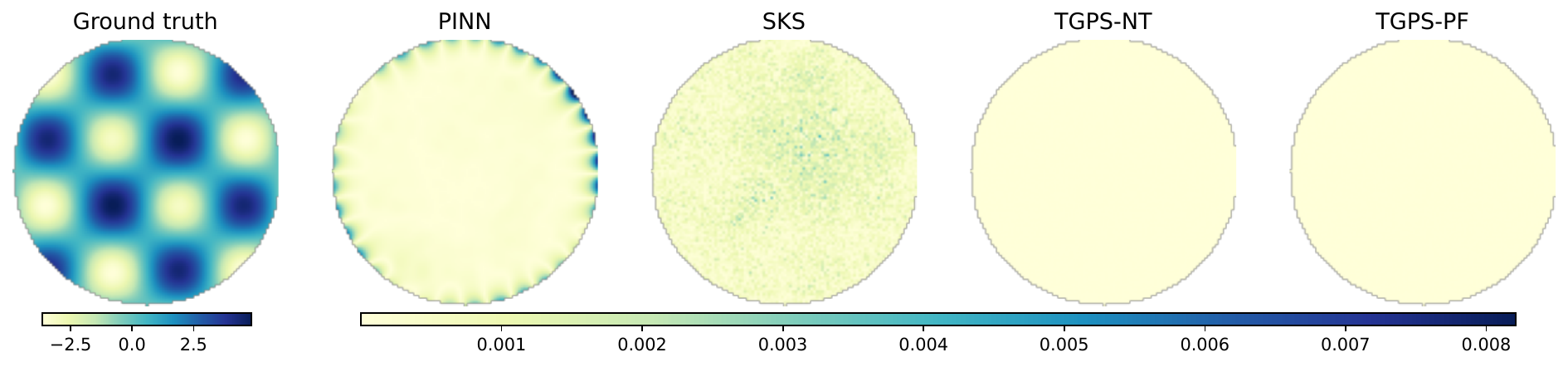}
		\caption{\small \textit{Circle domain}.}
		\label{fig:circle-non-elliptic}
	\end{subfigure} 	
	\caption{\small Solving the nonlinear elliptic PDE on irregular domains. The first column shows the ground-truth while the remaining columns the point-wise error of each method. }\label{fig:Irregular-Nonlinear elliptic}
\end{figure*}

\begin{figure*}
	\centering
	\setlength{\tabcolsep}{0pt}
		\begin{subfigure}[b]{1.0\linewidth}
			\includegraphics[width=\linewidth]{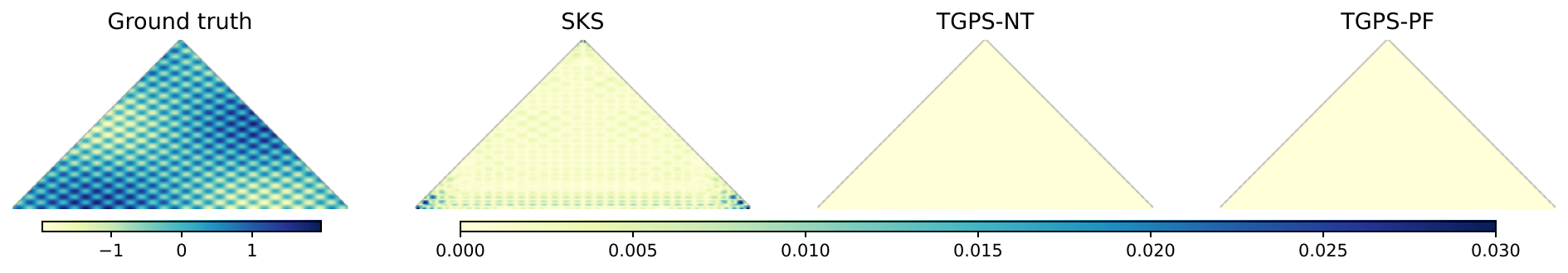}
			\caption{\small \textit{Triangle domain}.}
			\label{fig:triangle-2D-allen-cahen}
		\end{subfigure} 
		\begin{subfigure}[b]{1.0\linewidth}
			\includegraphics[width=\linewidth]{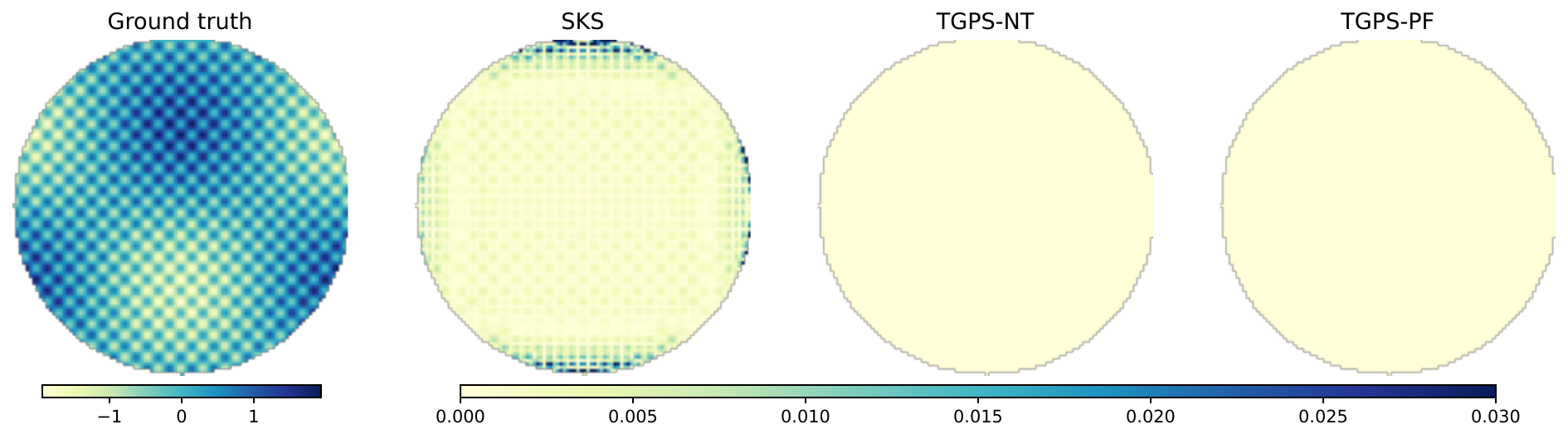}
			\caption{\small \textit{Circle domain}.}
			\label{fig:circle-2D-allen-cahen}
		\end{subfigure} 	
	\caption{\small Solving 2D Allen-Cahn equation ($a=15$) on irregular domains. The first column shows the ground-truth while the remaining columns the point-wise error of each method. }\label{fig:Irregular-2D-Allen-Cahen}
\end{figure*}

%% file: reference.bib
@inproceedings{long2022autoip,
	title={AutoIP: A united framework to integrate physics into Gaussian processes},
	author={Long, Da and Wang, Zheng and Krishnapriyan, Aditi and Kirby, Robert and Zhe, Shandian and Mahoney, Michael},
	booktitle={International Conference on Machine Learning},
	pages={14210--14222},
	year={2022},
	organization={PMLR}
}

@article{shin2020convergence,
  title={On the Convergence of Physics Informed Neural Networks for Linear Second-Order Elliptic and Parabolic Type PDEs},
  author={Shin, Yeonjong},
  journal={Communications in Computational Physics},
  volume={28},
  number={5},
  pages={2042--2074},
  year={2020}
}

@book{brenner2008mathematical,
	title={The mathematical theory of finite element methods},
	author={Brenner, Susanne C and Scott, L Ridgway},
	year={2008},
	publisher={Springer}
}

@article{barrett1993finite,
	title={Finite element approximation of the $p$-{L}aplacian},
	author={Barrett, John W and Liu, Wen Bin},
	journal={Mathematics of computation},
	volume={61},
	number={204},
	pages={523--537},
	year={1993}
}

@article{penwarden2023unified,
  title={A unified scalable framework for causal sweeping strategies for physics-informed neural networks ({PINN}s) and their temporal decompositions},
  author={Penwarden, Michael and Jagtap, Ameya D and Zhe, Shandian and Karniadakis, George Em and Kirby, Robert M},
  journal={Journal of Computational Physics},
  volume={493},
  pages={112464},
  year={2023},
  publisher={Elsevier}
}

@inproceedings{li2023meta,
  title={Meta learning of interface conditions for multi-domain physics-informed neural networks},
  author={Li, Shibo and Penwarden, Michael and Xu, Yiming and Tillinghast, Conor and Narayan, Akil and Kirby, Robert and Zhe, Shandian},
  booktitle={Proceedings of the 40th International Conference on Machine Learning},
  pages={19855--19881},
  year={2023}
}

@inproceedings{xutoward2025,
  title={Toward Efficient Kernel-Based Solvers for Nonlinear PDEs},
  author={Xu, Zhitong and Long, Da and Xu, Yiming and Yang, Guang and Zhe, Shandian and Owhadi, Houman},
  booktitle={Forty-second International Conference on Machine Learning},
year={2025},
organization={PMLR}
}

@article{kingma2014adam,
  title={Adam: A method for stochastic optimization},
  author={Kingma, Diederik P},
  journal={arXiv preprint arXiv:1412.6980},
  year={2014}
}

@article{zhao2016tensor,
  title={Tensor ring decomposition},
  author={Zhao, Qibin and Zhou, Guoxu and Xie, Shengli and Zhang, Liqing and Cichocki, Andrzej},
  journal={arXiv preprint arXiv:1606.05535},
  year={2016}
}

@inproceedings{rahaman2019spectral,
	title={On the spectral bias of neural networks},
	author={Rahaman, Nasim and Baratin, Aristide and Arpit, Devansh and Draxler, Felix and Lin, Min and Hamprecht, Fred and Bengio, Yoshua and Courville, Aaron},
	booktitle={International conference on machine learning},
	pages={5301--5310},
	year={2019},
	organization={PMLR}
}

@article{paszke2019pytorch,
	title={Pytorch: An imperative style, high-performance deep learning library},
	author={Paszke, Adam and Gross, Sam and Massa, Francisco and Lerer, Adam and Bradbury, James and Chanan, Gregory and Killeen, Trevor and Lin, Zeming and Gimelshein, Natalia and Antiga, Luca and others},
	journal={Advances in neural information processing systems},
	volume={32},
	year={2019}
}

@article{frostig2018compiling,
	title={Compiling machine learning programs via high-level tracing},
	author={Frostig, Roy and Johnson, Matthew James and Leary, Chris},
	journal={Systems for Machine Learning},
	volume={4},
	number={9},
	year={2018},
	publisher={SysML}
}

@article{CHEN2021110668,
title = {Solving and learning nonlinear PDEs with Gaussian processes},
journal = {Journal of Computational Physics},
volume = {447},
pages = {110668},
year = {2021},
issn = {0021-9991},
doi = {https://doi.org/10.1016/j.jcp.2021.110668},
url = {https://www.sciencedirect.com/science/article/pii/S0021999121005635},
author = {Yifan Chen and Bamdad Hosseini and Houman Owhadi and Andrew M. Stuart}
}

@article{BATLLE2025113488,
title = {Error analysis of kernel/GP methods for nonlinear and parametric PDEs},
journal = {Journal of Computational Physics},
volume = {520},
pages = {113488},
year = {2025},
issn = {0021-9991},
doi = {https://doi.org/10.1016/j.jcp.2024.113488},
url = {https://www.sciencedirect.com/science/article/pii/S0021999124007368},
author = {Pau Batlle and Yifan Chen and Bamdad Hosseini and Houman Owhadi and Andrew M. Stuart}
}

@misc{xu2024efficientkernelbasedsolversnonlinear,
      title={Toward Efficient Kernel-Based Solvers for Nonlinear PDEs}, 
      author={Zhitong Xu and Da Long and Yiming Xu and Guang Yang and Shandian Zhe and Houman Owhadi},
      year={2024},
      eprint={2410.11165},
      archivePrefix={arXiv},
      primaryClass={cs.LG},
      url={https://arxiv.org/abs/2410.11165}, 
}

@article{harshman1970foundations,
	title={Foundations of the {PARAFAC} procedure: Models and conditions for an ``explanatory" multi-modal factor analysis},
	author={Harshman, Richard A and others},
	journal={UCLA working papers in phonetics},
	volume={16},
	number={1},
	pages={84},
	year={1970},
	publisher={Los Angeles, CA}
}

@book{kolda2006multilinear,
	title={Multilinear operators for higher-order decompositions},
	author={Kolda, Tamara Gibson},
	volume={2},
	year={2006},
	publisher={United States. Department of Energy}
}

@article{oseledets2011tensor,
	title={Tensor-train decomposition},
	author={Oseledets, Ivan V},
	journal={SIAM Journal on Scientific Computing},
	volume={33},
	number={5},
	pages={2295--2317},
	year={2011},
	publisher={SIAM}
}

@article{raissi2019physics,
	title={Physics-informed neural networks: A deep learning framework for solving forward and inverse problems involving nonlinear partial differential equations},
	author={Raissi, Maziar and Perdikaris, Paris and Karniadakis, George E},
	journal={Journal of Computational physics},
	volume={378},
	pages={686--707},
	year={2019},
	publisher={Elsevier}
}

@book{owhadi2019operator,
	title={Operator-Adapted Wavelets, Fast Solvers, and Numerical Homogenization: From a Game Theoretic Approach to Numerical Approximation and Algorithm Design},
	author={Owhadi, Houman and Scovel, Clint},
	volume={35},
	year={2019},
	publisher={Cambridge University Press}
}

@article{batlle2023error,
	title={Error Analysis of Kernel/GP Methods for Nonlinear and Parametric PDEs},
	author={Batlle, Pau and Chen, Yifan and Hosseini, Bamdad and Owhadi, Houman and Stuart, Andrew M},
	journal={arXiv preprint arXiv:2305.04962},
	year={2023}
}

@article{krishnapriyan2021characterizing,
	title={Characterizing possible failure modes in physics-informed neural networks},
	author={Krishnapriyan, Aditi and Gholami, Amir and Zhe, Shandian and Kirby, Robert and Mahoney, Michael W},
	journal={Advances in Neural Information Processing Systems},
	volume={34},
	year={2021}
}

@article{chen2021solving,
	title={Solving and learning nonlinear {PDE}s with {G}aussian processes},
	author={Chen, Yifan and Hosseini, Bamdad and Owhadi, Houman and Stuart, Andrew M},
	journal={arXiv preprint arXiv:2103.12959},
	year={2021}
}

@inproceedings{graepel2003solving,
	Author = {Graepel, Thore},
	Booktitle = {ICML},
	Pages = {234--241},
	Title = {Solving noisy linear operator equations by {G}aussian processes: Application to ordinary and partial differential equations},
	Year = {2003}}

@article{owhadi2015bayesian,
	title={Bayesian numerical homogenization},
	author={Owhadi, Houman},
	journal={Multiscale Modeling \& Simulation},
	volume={13},
	number={3},
	pages={812--828},
	year={2015},
	publisher={SIAM}
}

@article{wang2021bayesian,
	title={Bayesian numerical methods for nonlinear partial differential equations},
	author={Wang, Junyang and Cockayne, Jon and Chkrebtii, Oksana and Sullivan, Timothy John and Oates, Chris J},
	journal={Statistics and Computing},
	volume={31},
	pages={1--20},
	year={2021},
	publisher={Springer}
}

@article{chen2023sparse,
	title={Sparse Cholesky factorization for solving nonlinear PDEs via Gaussian processes},
	author={Chen, Yifan and Owhadi, Houman and Sch{\"a}fer, Florian},
	journal={arXiv preprint arXiv:2304.01294},
	year={2023}
}

@article{wilson2015thoughts,
	title={Thoughts on massively scalable Gaussian processes},
	author={Wilson, Andrew Gordon and Dann, Christoph and Nickisch, Hannes},
	journal={arXiv preprint arXiv:1511.01870},
	year={2015}
}

@article{chen2021physics,
	title={Physics-informed learning of governing equations from scarce data},
	author={Chen, Zhao and Liu, Yang and Sun, Hao},
	journal={Nature communications},
	volume={12},
	number={1},
	pages={1--13},
	year={2021},
	publisher={Nature Publishing Group}
}

@article{raissi2019deep,
	title={Deep learning of vortex-induced vibrations},
	author={Raissi, Maziar and Wang, Zhicheng and Triantafyllou, Michael S and Karniadakis, George Em},
	journal={Journal of Fluid Mechanics},
	volume={861},
	pages={119--137},
	year={2019},
	publisher={Cambridge University Press}
}

@article{sahli2020physics,
	title={Physics-informed neural networks for cardiac activation mapping},
	author={Sahli Costabal, Francisco and Yang, Yibo and Perdikaris, Paris and Hurtado, Daniel E and Kuhl, Ellen},
	journal={Frontiers in Physics},
	volume={8},
	pages={42},
	year={2020},
	publisher={Frontiers}
}

@article{jin2021nsfnets,
	title={NSFnets (Navier-Stokes flow nets): Physics-informed neural networks for the incompressible Navier-Stokes equations},
	author={Jin, Xiaowei and Cai, Shengze and Li, Hui and Karniadakis, George Em},
	journal={Journal of Computational Physics},
	volume={426},
	pages={109951},
	year={2021},
	publisher={Elsevier}
}

@article{raissi2020hidden,
	title={Hidden fluid mechanics: Learning velocity and pressure fields from flow visualizations},
	author={Raissi, Maziar and Yazdani, Alireza and Karniadakis, George Em},
	journal={Science},
	volume={367},
	number={6481},
	pages={1026--1030},
	year={2020},
	publisher={American Association for the Advancement of Science}
}

@phdthesis{saatcci2012scalable,
	title={Scalable inference for structured Gaussian process models},
	author={Saat{\c{c}}i, Yunus},
	year={2012},
	school={Citeseer}
}

@inproceedings{wilson2015kernel,
	title={Kernel interpolation for scalable structured Gaussian processes (KISS-GP)},
	author={Wilson, Andrew and Nickisch, Hannes},
	booktitle={International Conference on Machine Learning},
	pages={1775--1784},
	year={2015}
}

@inproceedings{izmailov2018scalable,
	title={Scalable Gaussian Processes with Billions of Inducing Inputs via Tensor Train Decomposition},
	author={Izmailov, Pavel and Novikov, Alexander and Kropotov, Dmitry},
	booktitle={International Conference on Artificial Intelligence and Statistics},
	pages={726--735},
	year={2018}
}

@article{raissi2017physics,
	title={Physics informed deep learning (part i): Data-driven solutions of nonlinear partial differential equations},
	author={Raissi, Maziar and Perdikaris, Paris and Karniadakis, George Em},
	journal={arXiv preprint arXiv:1711.10561},
	year={2017}
}

@article{schafer2021sparse,
	title={Sparse Cholesky Factorization by Kullback--Leibler Minimization},
	author={Schafer, Florian and Katzfuss, Matthias and Owhadi, Houman},
	journal={SIAM Journal on scientific computing},
	volume={43},
	number={3},
	pages={A2019--A2046},
	year={2021},
	publisher={SIAM}
}

@article{fang2023solving,
	title={Solving High Frequency and Multi-Scale PDEs with Gaussian Processes},
	author={Fang, Shikai and Cooley, Madison and Long, Da and Li, Shibo and Kirby, Robert and Zhe, Shandian},
	journal={arXiv preprint arXiv:2311.04465},
	year={2023}
}

@article{bachmayr2023low,
	title={Low-rank tensor methods for partial differential equations},
	author={Bachmayr, Markus},
	journal={Acta Numerica},
	volume={32},
	pages={1--121},
	year={2023},
	publisher={Cambridge University Press}
}

@article{beylkin2002numerical,
	title={Numerical operator calculus in higher dimensions},
	author={Beylkin, Gregory and Mohlenkamp, Martin J},
	journal={Proceedings of the National Academy of Sciences},
	volume={99},
	number={16},
	pages={10246--10251},
	year={2002},
	publisher={National Acad Sciences}
}

@article{griebel2023analysis,
	title={Analysis of tensor approximation schemes for continuous functions},
	author={Griebel, Michael and Harbrecht, Helmut},
	journal={Foundations of Computational Mathematics},
	pages={1--22},
	year={2023},
	publisher={Springer}
}

@inproceedings{richter2021solving,
	title={Solving high-dimensional parabolic PDEs using the tensor train format},
	author={Richter, Lorenz and Sallandt, Leon and N{\"u}sken, Nikolas},
	booktitle={International Conference on Machine Learning},
	pages={8998--9009},
	year={2021},
	organization={PMLR}
}

@article{oster2022approximating,
	title={Approximating optimal feedback controllers of finite horizon control problems using hierarchical tensor formats},
	author={Oster, Mathias and Sallandt, Leon and Schneider, Reinhold},
	journal={SIAM Journal on Scientific Computing},
	volume={44},
	number={3},
	pages={B746--B770},
	year={2022},
	publisher={SIAM}
}

@article{fackeldey2022approximative,
	title={Approximative policy iteration for exit time feedback control problems driven by stochastic differential equations using tensor train format},
	author={Fackeldey, Konstantin and Oster, Mathias and Sallandt, Leon and Schneider, Reinhold},
	journal={Multiscale Modeling \& Simulation},
	volume={20},
	number={1},
	pages={379--403},
	year={2022},
	publisher={SIAM}
}
